\pdfoutput=1
\documentclass{article}

\usepackage[final, nonatbib]{neurips_2023}
\usepackage[numbers, sort, compress]{natbib}
\usepackage[dvipsnames]{xcolor}
\usepackage[utf8]{inputenc} 
\usepackage[T1]{fontenc}    
\usepackage{hyperref}       
\usepackage{url}            
\usepackage{booktabs}       
\usepackage{amsfonts}       
\usepackage{nicefrac}       
\usepackage{listings}
\usepackage{microtype}      
\usepackage{mathtools}
\usepackage{refcount}
\usepackage{algorithm}
\newcommand{\rarrow}[2]{\xrightarrow{\makebox[#1]{$#2$}}}
\newcommand{\larrow}[2]{\xleftarrow{\makebox[#1]{$#2$}}}
\usepackage{algpseudocode}
\usepackage{amsmath,amssymb, amsthm}
\usepackage{physics}
\usepackage[subtle, mathdisplays=tight, charwidths=normal, leading=normal]{savetrees}
\usepackage{caption}
\usepackage[bottom]{footmisc}
\usepackage{subcaption}

\usepackage{amsmath,amsfonts,bm, amsthm}

\newtheorem{theorem}{Theorem}
\newtheorem{assumption}{Assumption}
\newtheorem{lemma}{Lemma}

\def\eqref#1{equation~\ref{#1}}

\def\1{\bm{1}}

\def\vc{{\bm{c}}}

\def\vg{{\bm{g}}}
\def\vh{{\bm{h}}}

\def\vr{{\bm{r}}}

\def\vv{{\bm{v}}}
\def\vw{{\bm{w}}}
\def\vx{{\bm{x}}}
\def\vy{{\bm{y}}}
\def\vz{{\bm{z}}}

\DeclareMathAlphabet{\mathsfit}{\encodingdefault}{\sfdefault}{m}{sl}
\SetMathAlphabet{\mathsfit}{bold}{\encodingdefault}{\sfdefault}{bx}{n}

\def\sZ{{\mathbb{Z}}}

\newcommand{\R}{\mathbb{R}}

\newcommand{\custompar}[1]{\noindent{\bf #1.}\:}

\newcommand{\grab}{GraB}
\newcommand{\cgrab}{\grab}
\newcommand{\dgrab}{CD-\grab}

\newcommand{\shuffle}{RR}
\newcommand{\dshuffle}{D-\shuffle}
\newcommand{\so}{SO}

\newcommand{\epochs}{T}
\newcommand{\workers}{m}
\newcommand{\workerexamples}{n}
\newcommand{\examples}{N}

\newcommand{\windex}{i}
\newcommand{\exindex}{j}
\newcommand{\eindex}{t}
\newcommand{\loss}{f}

\newcommand{\dataex}{\vx}
\newcommand{\ex}{\vz}
\newcommand{\barex}{\bar{\vz}}
\newcommand{\exij}{\ex_{\windex,\exindex}}

\newcommand{\exj}{\ex_\exindex}
\newcommand{\weights}{\vw}
\newcommand{\weightst}{\weights_\eindex}
\newcommand{\perm}{\pi}

\newcommand{\sgn}{s}
\newcommand{\signj}{\sgn_\exindex}
\newcommand{\wexsign}{\sgn_\exindex^\windex}
\newcommand{\wexsignprev}{\sgn_{\exindex-1}^\windex}

\newcommand{\g}{\vg}
\newcommand{\wexgrad}{\g_\exindex^\windex}
\newcommand{\wexgradprev}{\g_{\exindex-1}^\windex}

\newcommand{\dstep}[1]{\textcolor{blue}{\textbf{\texttt{#1}}}}
\usepackage{float}
\usepackage{enumitem}

\newcommand*\samethanks[1][\value{footnote}]{\footnotemark[#1]}

\title{Coordinating Distributed~Example Orders for Provably Accelerated Training}

\author{%
 A. Feder Cooper\thanks{Equal contribution} \\
  \And Wentao Guo\samethanks\\
  \And Khiem Pham\samethanks\\
  \AND Tiancheng Yuan\\
  \And Charlie F. Ruan\\
  \And Yucheng Lu\\
  \And Christopher De Sa\\
  \vspace{-.2cm}
  \and Cornell University\\
  \texttt{\{afc78, wg247, dkp45, ty373, cfr54, yl2967, cmd353\}@cornell.edu}
}
  
\begin{document}

\maketitle

\begin{abstract}
Recent research on online Gradient Balancing (GraB) has revealed that there exist permutation-based example orderings for SGD that are guaranteed to outperform random reshuffling (RR). Whereas RR arbitrarily permutes training examples, GraB leverages stale gradients from prior epochs to order examples --- achieving a provably faster convergence rate than RR. However, GraB is limited by design: while it demonstrates an impressive ability to scale-up training on \emph{centralized} data, it does not naturally extend to modern \emph{distributed} ML workloads. We therefore propose \emph{Coordinated Distributed GraB} (CD-GraB), which uses insights from prior work on kernel thinning to translate the benefits of provably faster permutation-based example ordering to distributed settings. With negligible overhead, CD-GraB exhibits a linear speedup in convergence rate over centralized GraB and outperforms distributed RR on a variety of benchmark tasks.\looseness=-1

\end{abstract}

\vspace{-.1cm}
\section{Introduction}\label{sec:intro}
\vspace{-.1cm}

Random reshuffling, which samples training-data examples without replacement, has become the \emph{de facto} example-ordering method in modern deep-learning libraries~\citep{pytorchshuffle}, given that it tends to accelerate optimizer convergence in practice. However, some recent theoretical work has identified cases in which random reshuffling can lead to data orderings that have a poor effect on convergence~\citep{desa2020shuffle, yun2021can, rajput2021permutationbased}. This has encouraged a line of research to investigate if there exist provably better permutation-based orderings that afford greater scalability in training~\citep{lu2021general,mohtashami2022characterizing,lu2022grab}. Notably, \citet{lu2022grab} connects permuted-order SGD to the \emph{herding problem}~\citep{harvey2014near}, and proposes the herding-based online Gradient Balancing algorithm (\grab), which converges provably faster than random reshuffling, and does so  with little memory or computational overhead. In fact, in follow-on work, \citet{cha2023tighter} proves that \grab{} is optimal: in theory, \grab{} is the fastest possible permutation-based example ordering algorithm.

These results are very exciting, suggesting that \grab{} should unseat random reshuffling as the example ordering method-of-choice for SGD; however, they only hold with respect to a \emph{single} machine. \grab{} is optimal in settings with \emph{centralized} data, but does not naturally translate to problems of modern-ML scale, which demand that training workloads be distributed across \emph{multiple parallel} workers that each only have access to a subset of the training data. This drawback raises an important question:

\vspace{-.2cm}
\begin{center}
    \textit{Can we simultaneously achieve the scalability benefits of distributed training and \\ provably faster permutation-based example ordering for SGD --- both in theory and in practice?}
\end{center}
\vspace{-.2cm}

In this work, we show that it is indeed possible to attain these twin objectives. 
To do so, we suggest the online \textbf{C}oordinated \textbf{D}istributed \textbf{Gra}diant \textbf{B}alance algorithm (\dgrab), which leverages insights from kernel thinning to elevate the herding framework of centralized \grab{} (\cgrab) to the parallel setting. 
Felicitously, as a side effect, this choice of formulation brings about positive practical performance benefits (that can also improve the empirical behavior of centralized \cgrab). 
Using the exact same assumptions as the original \cgrab{} paper, \textbf{we show analytically that coordinating example orders across parallel workers leads a linear speedup in convergence rate}. 
For $\epochs$ epochs and $\workers$ parallel workers, each with access to $\workerexamples$ examples, \dgrab's convergence rate is $\tilde{O}((\workers\workerexamples\epochs)^{-2/3})$ on smooth, non-convex objectives and $\tilde{O}((\workers\workerexamples\epochs)^{-2})$ under the Polyak-\L ojasiewicz (P.L.) condition.\footnote{In this paper, we use $\tilde O$ by convention to hide logarithmic factors in the problem parameters.} 

We  run a series of experiments to verify these improvements in practice, implementing \dgrab{} on a single node that distributes computation across multiple GPUs. 
We also run an ablation study in order to disentangle the benefits of parallelism from the positive side effects of using kernel thinning to formulate the \dgrab{} algorithm. 
Similar to how centralized \cgrab{} demonstrates improved generalization over centralized random reshuffling (\shuffle), we observe that \dgrab{} exhibits improved generalization over distributed random reshuffling (\dshuffle). 
Altogether, the success of our work suggests a new distributed training paradigm to explore in future work, which we call the \emph{Order Server} (Section~\ref{sec:conclusion}). In summary, we:
\begin{itemize}[topsep=0pt, leftmargin=.5cm]
    \item Propose the online \textbf{C}oordinated \textbf{D}istributed \textbf{Gra}dient \textbf{B}alancing (\dgrab) algorithm, which enables provably accelerated training using SGD in the parallel setting (Section~\ref{sec:dgrab});
    \item Prove that the convergence rate for \dgrab{} exhibits a linear speedup over \cgrab, using the exact same assumptions as the original \cgrab{} paper (Section~\ref{sec:theory}); 
    \item Produce extensive empirical validation of \dgrab's improved scalability on a variety of tasks in deep learning and on large-scale logistic regression  (Section~\ref{sec:experiments}). 
\end{itemize}

\vspace{-.1cm}
\section{Preliminaries and Related Work}\label{sec:prelimrw}
\vspace{-.1cm}

In this section, we discuss the preliminaries and prior scholarship on permutation-based example ordering, with particular attention paid to the centralized online Gradient Balancing Algorithm (\cgrab)~\cite{lu2022grab}. 
This lays the groundwork for how our coordinated, distributed \grab{} algorithm (Section~\ref{sec:dgrab}) imparts the efficiency guarantees of \cgrab{} to the parallelized regime (Section~\ref{sec:theory}).

\custompar{Ordering data examples during training} Training a model can be formulated as minimizing a differentiable loss function $\loss:\R^d\rightarrow\R$ over $\examples$ data examples. The goal of this minimization is to obtain the target model weights $\weights^* = \arg\min_{\weights}\loss(\weights)$, where $\loss(\weights) = \frac{1}{\examples} \sum_{\exindex=1}^{\examples} \loss(\weights; \exindex)$,
for which $\loss(\weights; \exindex)$ denotes the loss incurred on the $\exindex$-th example. A typical training process iteratively updates the model parameters $\weights$ by scanning over the $\examples$ data examples repeatedly, with $t$-th scan (or epoch) following
\vspace*{.2cm}
\begin{align}
\label{equ:grab:main_update}
    \weights_{\eindex}^{\exindex+1} = \weights_\eindex^{\exindex} - \alpha \nabla \loss(\weights_\eindex^\exindex; \pi_{\eindex}(\exindex)), \hspace{.5em} \forall \exindex \in [\examples], 
\end{align}

where $\alpha$ denotes the learning rate, and $\perm_t:[\examples] \rightarrow [\examples]$ denotes a permutation ordering\footnote{While without-replacement orderings are most common in large-scale learning~\citep{bottou2012stochastic}, 
ordering strategies need not be permutations, e.g.,  
with-replacement sampling~\citep{schmidt2017minimizing,needell2014stochastic,lu2021variance}  
or curriculum learning~\citep{graves2017automated,matiisen2019teacher,soviany2022curriculum}.} adopted in the $t$-th epoch from which the examples are chosen to compute gradients, $\weights_t^1$ denotes the  initial model weights for the $\eindex$-th epoch, and $\weights_\eindex^{\exindex}$ denotes the model weights after $\exindex-1$ gradient updates in the $\eindex$-th epoch.\footnote{Note that we write (\ref{equ:grab:main_update}) in terms of per-example-$j$ gradients.}

The choice of ordering $\perm$ can have a significant effect on optimizer performance. Two popular methods, which can demonstrate convergence speedups in practice, are 1) random reshuffling (\shuffle)~\citep{ying2017performance}, for which the permutations are random and differ over epochs, and 2) Shuffle Once (SO)~\citep{bertsekas2011incremental,gurbuzbalaban2019convergence}, for which a random permutation is computed once and remains fixed for all epochs. \citet{recht2012toward} conducted the first theoretical investigation of \shuffle, while subsequent works like \citet{yun2021can} and \citet{desa2020shuffle} have given counterexamples in which \shuffle{} leads to orderings that have a poor effect on convergence. Altogether, many studies indicate that \shuffle{} and \so{} only provide efficiency benefits under certain conditions~\citep{haochen2019random,gurbuzbalaban2021random,mishchenko2020random}.%

These limitations of \shuffle{} and \so{} have motivated research to identify permutations that outperform random ones. \citet{rajput2021permutationbased} introduces an \shuffle{} variant that achieves improved convergence for quadratics by reversing the ordering every other epoch. Other non-\shuffle-based methods pick efficient orderings based on correlations between adjacently selected examples.  
In a recent line of work, \citet{lu2021general} proves that faster convergence is possible for SGD when the averages of consecutive stochastic gradients converge faster to the full gradient. 
Based on this result, in follow-on work
\citet{lu2022grab} proposes the centralized online Gradient Balancing algorithm (\grab), which outperforms \shuffle, and upon which we base this work.

\vspace{-.1cm}
\subsection{\cgrab: Optimal, online, permutation-based example ordering for centralized ML}\label{sec:cgrab}
\vspace{-.1cm}

\grab{} is a permutation-based example-ordering algorithm that identifies provably better-than-random orderings \emph{in centralized, single-node settings} for SGD. 
\grab{} finds such orderings by leveraging information in stale stochastic gradients from previous epochs to guide  ordering in the next epoch. 
More formally, for smooth, non-convex objectives, \citet{lu2022grab} proves that any permutation $\perm^*$ that guarantees
\vspace*{.2cm}
\begin{align}
\label{equ:grab:grad_error}
    \textstyle
    \max_{k\in[\examples]} \left\| \sum_{\exindex=1}^k\nabla \loss(\weights;\perm^*(\exindex)) - \nabla \loss(\weights)
 \right\|_\infty = \tilde{O}(1) \hspace{.75em} \text{(} \nabla \loss(\weights) \text{ is the average gradient)},
\end{align}

will yield a convergence rate of $\tilde{O}((\examples\epochs)^{-2/3})$ (for epochs $\epochs$) for SGD, which is superior to the $O(\examples^{-1/3}\epochs^{-2/3})$ convergence rate of random reshuffling~\cite{mishchenko2020random}. 

\custompar{\cgrab's connection to herding and balancing} To find such a permutation $\perm^*$, \citet{lu2022grab} connect (\ref{equ:grab:grad_error}) to the \emph{herding problem} and vector \emph{balancing}~\citep{harvey2014near, welling2009herding}. Understanding why \grab{} does not naturally extend to the distributed setting --- and our main contributions (Sections~\ref{sec:dgrab} and~\ref{sec:theory}) --- requires some additional details on the fundamentals of herding: 

Given $\examples$ vectors\footnote{Herding does not have an optimization context. Here, $\examples$ does \emph{not} refer to the number of data examples used in training (\ref{equ:grab:main_update}); rather, $\examples \in \sZ^+$ describes the size of a set of arbitrary vectors. We slightly abuse notation because we execute the herding subroutine on exactly $\examples$ gradients (Section~\ref{sec:dgrab}), which happen to equal the number of $\examples$ examples.} 
$\{\exj\}_{\exindex=1}^\examples$ ($\exj \in \R^d$), $\norm{\exj}_2 \le 1$ ($\forall \exindex$), herding identifies 
a permutation $\perm^*$ such that
\vspace{.2cm}
\begin{align}
\hspace{-.31cm}
\label{equ:herding:objective}
    \textstyle
    \max_{k \in [\examples]} \norm{\sum_{\exindex=1}^k \left( \ex_{\perm^*(\exindex)} - \barex \right)}_\infty = \tilde{O}(1), \hspace{.5cm} \text{ where } \barex = \frac{1}{\examples}\sum_{\exindex=1}^\examples \exj.
\end{align} 

It is clear that (\ref{equ:herding:objective}) generalizes  (\ref{equ:grab:grad_error}), which 
is a specific case of herding in an optimization setting. 

\citeauthor{harvey2014near} solve (\ref{equ:herding:objective}) with a method called \emph{balancing}~\citep{harvey2014near}. Balancing uses a \emph{signed} version of the herding problem to optimize any given permutation $\perm$ to reduce the bound in (\ref{equ:herding:objective}). That is, balancing formulates the signed herding problem
\begin{align}
\label{equ:herding:signed_objective}
    \textstyle
    \max_{k \in [\examples]} \norm{\sum_{\exindex=1}^k \sgn_{\perm(\exindex)} \left( \ex_{\perm(\exindex)} - \barex \right) }_\infty, \hspace{1em} \text{where} \hspace{.5em} \{\sgn_\exindex\}_{\exindex=1}^\examples \in\{+1, -1\}.
\end{align}
Given a group of such signs $\{\sgn_\exindex\}_{\exindex=1}^\examples$ and an arbitrary permutation $\perm$, \citeauthor{harvey2014near} prove that Algorithm~\ref{alg:reorder} produces a new permutation $\perm'$ such that
{\small\begin{align*}
    \textstyle
    \max \limits_{k \in [\examples]} \norm{\sum_{\exindex=1}^k \left( \ex_{\perm'(\exindex)} - \barex \right)}_\infty  \, 
    \leq  \; \frac{1}{2} \max \limits_{k \in [\examples]} \norm{\sum_{\exindex=1}^k \sgn_{\perm(\exindex)}\left( \ex_{\perm(\exindex)} - \barex \right)}_\infty + \frac{1}{2}\max \limits_{k \in [\examples]} \norm{\sum_{\exindex=1}^k \left( \ex_{\perm(\exindex)} - \barex \right)}_\infty.
\end{align*}}%

\vspace{-.2cm}
\begin{minipage}{.425\linewidth}
This says that, with new permutation $\perm'$, the objective of (\ref{equ:herding:objective}) now approaches the bound of (\ref{equ:herding:signed_objective}). Importantly, recent advances show that it is quite cheap to find a group of signs, such that (\ref{equ:herding:signed_objective}) is on the order of $\tilde{O}(1)$ (e.g., \citet{alweiss2021discrepancy}, in  Algorithm~\ref{alg:pairbalance}). We are therefore able to call Algorithm~\ref{alg:reorder} repeatedly, which will eventually obtain the $\perm^*$ that solves the $\tilde{O}(1)$ herding objective in (\ref{equ:herding:objective}). 
\end{minipage}
\hfill
\begin{minipage}{0.542\textwidth}
\vspace{-.4cm}
    \begin{algorithm}[H]
    	\caption{Reordering Vectors based on Balanced Signs [\citet{harvey2014near}]}\label{alg:reorder}
        \footnotesize
    	\begin{algorithmic}[0]
    	\State \textbf{input:} a group of signs $\{\sgn_\exindex\}_{\exindex=1}^\examples$, initial order $\perm$
    	\State \textbf{initialize:} two order-sensitive lists $L_{\text{pos}}\leftarrow [\hspace{0.1em}]$, $L_{\text{neg}}\leftarrow [\hspace{0.1em}]$.
    	\For{$\exindex = 1\dots\examples$}
    	    \State $L_{\text{pos}}.\textsf{\scriptsize{append}}(\perm(\exindex))$ \textbf{if} $\signj$ is $+1$ \textbf{else} $L_{\text{neg}}.\textsf{\scriptsize{append}}(\perm(\exindex))$.
    	\EndFor
    	\State \textbf{return:} new order $\perm'\coloneqq\textsf{\scriptsize{concat}}(L_{\text{pos}}, \textsf{\scriptsize{reverse}}(L_{\text{neg}}))$.
    	\end{algorithmic}
    \end{algorithm}
\end{minipage}

\custompar{\cgrab's application of herding to gradient balancing}\citet{lu2022grab} applies this framework of herding and balancing to develop  \grab{}, i.e., to minimize (\ref{equ:grab:grad_error}). The main challenge for the success of this approach is to find the right gradients $\exj$ in the optimization context of (\ref{equ:grab:grad_error}). Notably, the herding and balancing framework requires the vector mean $\barex$ in advance. To satisfy this requirement, \grab{} ``centers'' the gradient vectors using a \emph{stale mean}. That is, 
\grab{} runs the herding algorithm on vectors that are defined as
\vspace{.2cm}
\begin{align}
\label{equ:grab:stale_mean}
    \textstyle
    \exj = \nabla \loss(\weights_\eindex^\exindex;\perm_\eindex(\exindex)) - \frac{1}{\examples}\sum_{p=1}^{\examples}\nabla \loss(\weights_{\eindex-1}^p;\perm_{\eindex-1}(p)),
\end{align}

where $\weightst^p$ denotes the model weights after $p-1$ updates in the $t$-th epoch, and $\perm_\eindex$ denotes the permutation adopted in the $t$-th epoch. \citet{lu2022grab} proves that this definition of $\exj$ preserves the benefits of balancing with negligible noise or overhead. The only overhead comes from storing the running average of the gradients in epoch $\eindex -1$ to ``center'' the gradients in the subsequent epoch $\eindex$.

With this approach, \citet{lu2022grab} proves that \grab{} demonstrates more efficient convergence than \shuffle{} for SGD. Better still, \citet{cha2023tighter} demonstrates that \grab{} is in fact the \emph{optimal} permutation-based ordering method for SGD: 
it is not possible to produce a permutation-based ordering in the centralized setting that achieves a faster convergence rate for SGD. 

Despite \grab's clear benefits over \shuffle, it assumes local access to all examples. This assumption does not hold for popular, modern, parallel settings (e.g., parameter server~\citep{li2014ps}), in which workers only have access to subsets of examples. 
No present work has attempted to investigate \grab's applicability to this setting.
While some work has studied distributed \shuffle{} (\dshuffle)~\citep{yun2021minibatch,huang2021distributed,malinovsky2022server,sadiev2022federated}, it remains an open question if \grab's efficiency benefits for SGD can be conferred to the modern-scale, distributed-ML setup. 

\vspace{-.1cm}
\section{\dgrab: A Provably Efficient Ordering Algorithm for Distributed Training}\label{sec:dgrab}
\vspace{-.1cm}

Our main contribution is to elevate \grab{} to the parallel regime, so that distributed training can enjoy the efficiency benefits of provably better example ordering. Based on the preliminaries, we can now explain why this is not a straightforward task: \textbf{While \grab{} achieves the optimal convergence rate for SGD on centralized data, it does not naturally translate to a distributed setting} (Section~\ref{sec:dgrab:issues}). Our key insights for resolving these problems are to reformulate the herding framework in \citet{lu2022grab} to work in parallel, and to leverage insights from 
kernel thinning~\citep{dwivedi2021kernel, dwivedi2022generalized, barp2022targeted} to derive the \emph{online} $\mathsf{PairBalance}$ algorithm, which solves this parallelized herding objective (Section~\ref{sec:dgrab:solution}). Lastly, we present the full-stack \dgrab{} algorithm that makes our solution work in practice (Section~\ref{sec:dgrab:algo}). The server implements online $\mathsf{PairBalance}$, which coordinates gradient information from the distributed workers in training epoch $\eindex$ in order to determine a provably efficient example order for the next epoch $\eindex + 1$ (Section~\ref{sec:theory}). 

\vspace{-.1cm}
\subsection{Issues with \grab{} in the distributed setting}\label{sec:dgrab:issues}
\vspace{-.1cm}

To clarify the issues with distributing \grab, we first need to define the distributed training setup more precisely. 
We consider the standard data-parallel 
regime with $\workers$ parallel workers, where each worker keeps a copy of the model weights $\weights\in\R^d$ and maintains $\workerexamples= \examples / \workers$ local 
examples.\footnote{Without loss of generality, 
we assume the $\examples$ examples are divided evenly among the $\workers$ workers 
and $\workerexamples$ is even.} As in many data-parallel training applications,\footnote{One such popular 
paradigm is federated learning
~\citep[e.g.]{mcmahan2017communication}. Federated learning 
typically involves highly imbalanced loads, heterogeneous data, partial user participation, and additional privacy-preserving mechanisms. 
These characteristics are orthogonal to what we consider here for example order. 
If we were to allow for such data organization, we would need to assume non-global communication per iteration or additional constraints on how global communication occurs.
For \dgrab, we focus on the regime of using parallelism to accelerate training.} 
 such as geo-distributed model training \citep{yuan2022decentralized}, we assume \emph{the data examples cannot be shared or moved across workers}.
More formally, this setup can be expressed as
\vspace{.2cm}\begin{equation}
\label{equ:d-grab:objective}
\textstyle
\min_{\weights \in \R^d} \left[ \loss(\weights) = \frac{1}{\workers}\sum_{\windex=1}^{\workers} \loss^\windex(\weights) \right] \quad \text{with} \quad \loss^\windex(\weights) = \frac{1}{\workerexamples}\sum_{\exindex=1}^{\workerexamples} \loss^\windex(\weights; \exindex),
\end{equation}

where $\loss^\windex(\weights; \exindex): \R^d \rightarrow \R$, $\exindex \in [\workerexamples]$, denotes the loss incurred on the $\exindex$-th example on the $\windex$-th worker for model weights $\weights$. We can now consider running (\ref{equ:grab:main_update}) using this setup, for which each worker scans over their $\workerexamples$ local-data examples using (potentially) different permutations. We denote $\perm_{\eindex,\windex}: [\workerexamples] \rightarrow [\workerexamples]$ as the permutation-based ordering adopted on the $\windex$-th worker in the $\eindex$-th training epoch. Adjusting (\ref{equ:grab:main_update}) to accommodate the setup in (\ref{equ:d-grab:objective}), the update to the model can be summarized as
\vspace{.2cm}\begin{align}
\label{equ:d-grab:main_update}
    \textstyle
    \weights_{\eindex}^{\exindex+1} = \weights_\eindex^{\exindex} - \frac{\alpha}{\workers}\sum_{\windex=1}^\workers \nabla \loss^\windex(\weights_\eindex^\exindex; \pi_{\eindex,\windex}(\exindex)), \hspace{.5em} \forall \exindex \in [\workerexamples]. 
\end{align}

That is, in epoch $\eindex$, each worker $\windex$ selects their respective, local $\exindex$-th example according to $\{\perm_{\eindex,\windex}\}_{\windex=1}^\workerexamples$ in order to compute stochastic gradients (Appendix). 

\textbf{Following this setup, Algorithm~\ref{alg:reorder} no longer guarantees the $\tilde{O}(1)$ bound to the herding problem~(\ref{equ:herding:objective})}, a bound that is valid only when \emph{all} data examples can be permuted \emph{freely}~\citep{harvey2014near}. This constraint is fine for centralized \grab, but, in distributed training, 
workers only have access to a \emph{subset} of examples. Distributed training requires that \emph{worker-specific permutations only involve the examples in their respective local subsets}. 
Further, recall that \grab{} uses stale means to center gradients (\ref{equ:grab:stale_mean}) in order to solve the herding objective. This, too, causes problems in distributed training. In practice, it is typical to employ larger learning rates $\alpha$ for greater scalability~\citep{smith2018don}; larger $\alpha$ increases the discrepancy between averaged gradients in adjacent epochs, which, in turn, would make \grab's use of stale means unreliable. 

\vspace{-.1cm}
\subsection{Our efficient solution: parallel herding and pair balancing}\label{sec:dgrab:solution}
\vspace{-.1cm}

To address the limitations presented in the prior section, which preclude the direct application of \grab{} to distributed training, we will need to \textbf{1) reformulate the herding problem to fit the parallel setting, and 2)  redesign how to do gradient balancing}, such that it both solves our new herding formulation and allows for  reliability with higher learning rates. We now present our solution to both these problems; we introduce the \emph{parallel herding} problem and the online $\mathsf{PairBalance}$ subroutine that solves it.

\custompar{Parallel Herding} To extend herding to the parallel setting, consider the following setup: There are $\workers$ workers, which each have local access to $\workerexamples$ vectors. 
Let $\ex_{\windex,\exindex} \in \R^d$ denote the vector indexed by $\exindex$ on the $\windex$-th worker. Assuming $\| \ex_{\windex,\exindex} \|_2 \le 1 \;\; (\forall \windex \in [\workers], \forall \exindex \in [\workerexamples])$, the goal of parallel herding is to find $\workers$ permutations, $\perm_1, \perm_2, \ldots, \perm_\workers$ where $\perm_\windex:[\workerexamples]\rightarrow[\workerexamples] \;\; (\forall \windex \in[\workers])$, so as to minimize:
\vspace{.1cm}\begin{align}
  \textstyle
  \max_{k \in [\workerexamples]} \; \left\| \sum_{\exindex=1}^k \sum_{\windex=1}^\workers \left( \ex_{\windex, \perm_\windex(\exindex)} - \barex \right) \right\|_{\infty}, \hspace{1em}\text{with}\hspace{1em}\barex=\frac{1}{\workers\workerexamples}\sum_{\windex=1}^{\workers}\sum_{\exindex=1}^{\workerexamples}\ex_{\windex,\exindex}.
\label{equ:paraherding:objective}
\end{align}

When directly comparing (\ref{equ:paraherding:objective}) with (\ref{equ:herding:objective}), it is clear that parallel herding differs in two notable ways from the original herding problem. First, each permutation $\perm_\windex:[\workerexamples]\rightarrow[\workerexamples] \;\; (\forall \windex \in[\workers])$ only decides the ordering of the $\workerexamples$ vectors that are associated with worker $\windex$. Second, the prefix sum taken in the objective norm is accumulated over all the workers (the inner sum from $\windex=1\ldots\workers$). This formulation naturally captures the setting in a distributed environment: \textbf{workers need to decide permutations collaboratively, and the worker-specific vectors are processed simultaneously rather than sequentially}.

Given that this formulation fits the distributed setting, we next need to show that parallel herding does in fact address the limitations posed by centralized \grab: that it is possible recover the original  $\tilde{O}(1)$  herding bound, and that we can solve the issue of unreliable stale gradients (Section~\ref{sec:dgrab:issues}). The solution that we present in the remainder of this section is a new vector balancing subroutine: online $\mathsf{PairBalance}$.  To give an intuition, as its name suggests, online $\mathsf{PairBalance}$ leverages insights from kernel thinning to \emph{balance} vector differences over vector \emph{pairs}. This also eliminates the need to perform vector centering, and thus solves the stale mean problem. 

\custompar{Using kernel thinning to solve parallel herding} We call our solution to the parallel herding objective (\ref{equ:paraherding:objective}) \emph{pair balancing}, which we derive from key insights 
in \emph{kernel thinning}~\citep{dwivedi2021kernel,dwivedi2022generalized,barp2022targeted}. 
In particular, \citeauthor{dwivedi2021kernel} show that it is possible to solve the herding objective in $\tilde{O}(1)$ \textbf{by only examining differences on  \emph{pairs of examples}}~\citep{dwivedi2021kernel}. They derive an algorithm that generalizes \citet[subroutine in Algorithm~\ref{alg:pairbalance}]{alweiss2021discrepancy}, which solves herding in $\tilde{O}(1)$ (Section~\ref{sec:prelimrw}), and does so by operating only on vector-pair differences.\footnote{\citeauthor{dwivedi2021kernel} minimize the maximum mean discrepancy (MMD) between a selected coreset and an empirical distribution. They 
develop a new self-balancing Hilbert walk on differences of \emph{pairs of examples} 
to select exactly half of the dataset points, and solve coreset selection by iteratively halving the input vector sequence into balanced coresets then selecting and refining a candidate coreset to minimize MMD with the input sequence.\looseness=-1}  This comes with a very useful property: eliminating the requirement of knowing the maximum vector norm ahead of time and centering the vectors (i.e., making all the vectors sum to zero) in order to solve the herding problem. This is the key to solving the parallel herding objective (\ref{equ:paraherding:objective}) in $\tilde{O}(1)$, and elevating the benefits of \grab{} to a distributed setting. 

\begin{figure}[t!]
\hspace{-.25cm}
\vspace{-.1cm}
\begin{minipage}{.485\textwidth}
\vspace{-.7cm}
\begin{algorithm}[H]
\caption{$\mathsf{PairBalance}$}\label{alg:pairbalance}
\footnotesize
\begin{algorithmic}[0]
  \Statex $\rhd$ 
  The inputs, outputs and subroutine for this algorithm 
  are order-sensitive\vspace{.2cm}
  \Statex \textbf{input:} current running sum $\vr$, 
  paired vectors $\vz_1$, $\vz_2$\vspace{.2cm}
  \State \textbf{compute:} $\sgn, \vr \leftarrow \mathsf{\small{RandomizedBalance}}(\vr, \vz_1 - \vz_2)$\looseness=-1
  \State \textbf{return:}  $\sgn$ (sign for $\vz_1$), 
  \Statex \hspace{.75cm} $-\sgn$ (sign for $\vz_2$), 
  \Statex \hspace{1cm} $\vr$ (updated running sum) \vspace{.4cm}
  \Statex $\rhd$ Adapted from~\citet{alweiss2021discrepancy}
  \State \textbf{define subroutine:} $\mathsf{\small{RandomizedBalance}}(\vr, \vc)$
  \State \hspace{.2cm} \textbf{compute:} $p\leftarrow \frac{1 - \langle \vr,\vc\rangle}{2}$
  \State \hspace{.2cm} \textbf{compute:} $\sgn\leftarrow +1$ \hspace{.2em} with probability \hspace{.2em} $p$; 
  \Statex \hspace{1.55cm} $\sgn\leftarrow -1$ \hspace{.2em} with probability \hspace{.2em} $1-p$
  \State \hspace{.2cm} \textbf{update:}  $\vr\leftarrow \vr+\sgn\vc$
  \State \hspace{.2cm} \textbf{return:}  $\sgn$, $\vr$
\end{algorithmic}
\end{algorithm}
\end{minipage}
\hspace{.25cm}
\begin{minipage}{.58\textwidth}
  \includegraphics[width=\linewidth]{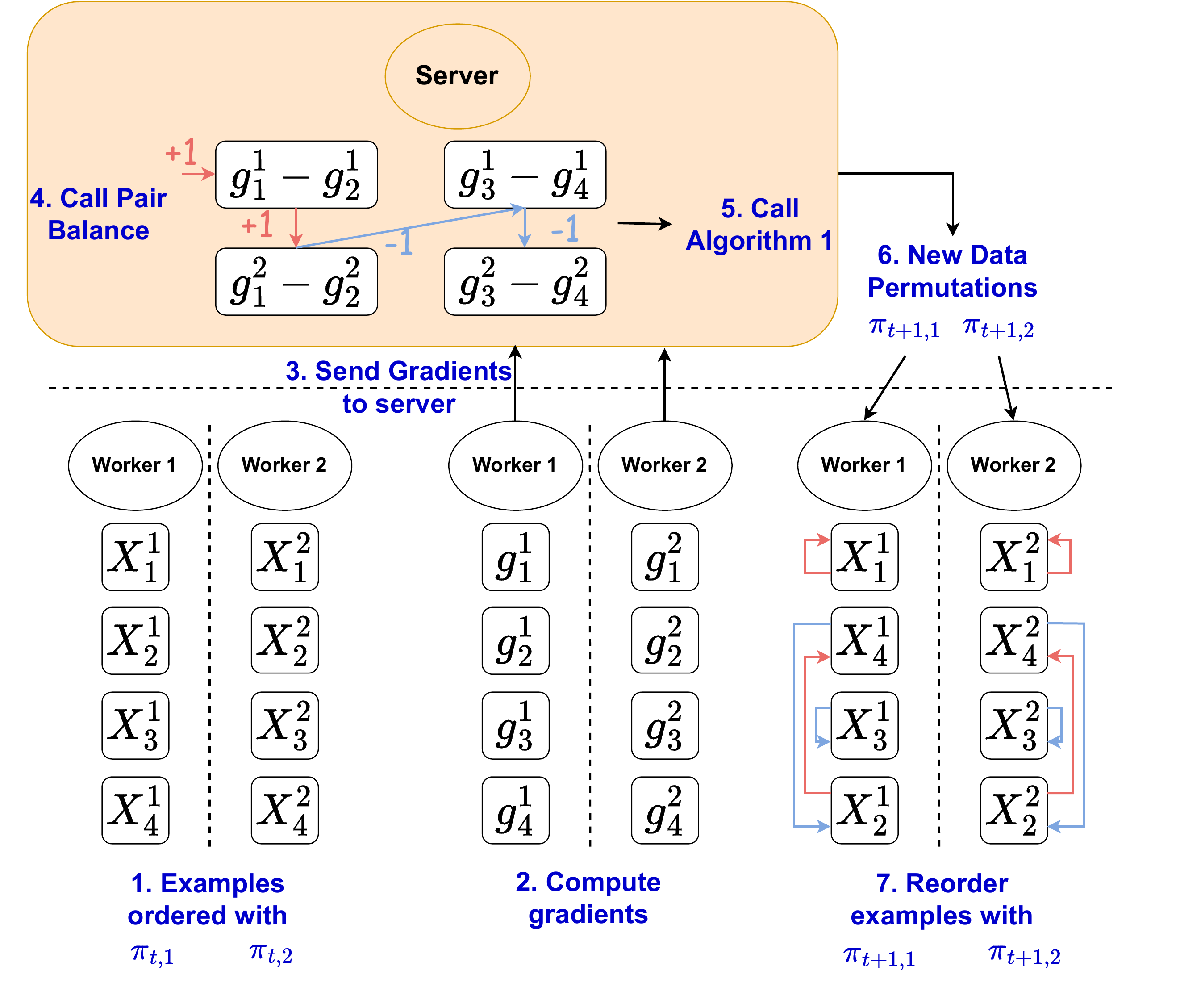}
\end{minipage}
\vspace{-.1cm}
\caption{\textbf{Left:} The $\mathsf{PairBalance}$ algorithm, which the server runs online.  \textbf{Right:} 
\dgrab{} running on one server (top) and two workers (bottom). 
The workers do not share data examples.}
\vspace{-.5cm}
\label{fig:diagram}
\end{figure}

Following \citet{dwivedi2021kernel}, we will balance over paired vectors, and will do so in an \emph{online} fashion (Section~\ref{sec:dgrab:algo}). This eliminates \grab's requirement of using a stale mean to center gradient vectors (Section~\ref{sec:cgrab}), but still minimizes the parallel herding objective to $\tilde{O}(1)$. We defer proving this result to Section~\ref{sec:theory}, and first describe our concrete algorithm.  Online $\mathsf{PairBalance}$ applies Algorithm~\ref{alg:reorder} on the``flattened'' and ``paired'' sequence of all of the workers' paired-difference gradients, i.e., \looseness=-1
\vspace{.2cm}
\begin{align*}
\textstyle
    \vy_{\workerexamples(k-1)+\windex} = \ex_{\windex,2k-1} - \ex_{\windex,2k}, \hspace{1em} \forall k\in[\frac{\workerexamples}{2}], \hspace{1em} \windex=1\ldots\workers.
\end{align*}

That is, we fit these ordered-paired differences $\{\vy_\windex\}_{\windex=1}^{\workers\workerexamples/2}$ into the herding and balancing framework  (Algorithm~\ref{alg:reorder}): 
if sign $\sgn$ is associated with $\vy_{\workerexamples(k-1)+\windex}$, then $\ex_{\windex,2k-1}$ and $\ex_{\windex,2k}$ receive 
$\sgn$ and $-\sgn$, respectively.\looseness=-1 

\vspace{-.1cm}
\subsection{The full-stack \dgrab{} algorithm}\label{sec:dgrab:algo}
\vspace{-.1cm}

Having solved the parallel herding problem with pair balancing, we now demonstrate how to bring everything together in an optimization context to \emph{coordinate distributed gradient balancing} for distributed training. That is, we can now introduce our full-stack \dgrab{} algorithm, which trains models in a distributed setting (Section~\ref{sec:dgrab:issues}) while efficiently ordering the examples by using $\mathsf{PairBalance}$ (Section~\ref{sec:dgrab:solution}, Algorithm~\ref{alg:pairbalance}) 
in an online manner. 

We describe \dgrab{} at two levels of abstraction: a high-level illustration (Figure~\ref{fig:diagram}, steps \dstep{1-7}) and a detailed pair of worker-server algorithm statements (Figure~\ref{alg:dgrab}). 
Since the workers only have access to a subset of the training data, in parallel they compute local, per-example stochastic gradients  and send them to the server. The server simultaneously calls $\mathsf{PairBalance}$ online, which coordinates information from all the workers' gradients (i.e., using adjacent example-specific gradients) to determine the next epoch's worker-specific permutations. In more detail:

In epoch $\eindex$, (Figure~\ref{fig:diagram},  step \dstep{1}) the two workers have permutations $\perm_{\eindex, 1}$ and $\perm_{\eindex,2}$, respectively. Each worker computes per-example gradients $\g_\exindex^\windex$ (\dstep{2}; Algorithm~\ref{alg:dgrab:workers}:4), and sends them to the server (\dstep{3}; Algorithm~\ref{alg:dgrab:workers}:5). The server we implement functions as a parameter server~\cite{li2014ps}: It computes the average of the workers' per-example gradients (Algorithm~\ref{alg:dgrab:server}:6), and sends it back to all workers (Algorithm~\ref{alg:dgrab:server}:7) so that they can update their local models (Algorithm~\ref{alg:dgrab:workers}:6-7). Simultaneously, as the server receives gradients (Algorithm~\ref{alg:dgrab:server}:5), it calls  $\mathsf{PairBalance}$ (Algorithm~\ref{alg:pairbalance}) on adjacent vectors (\dstep{4}; Algorithm~\ref{alg:dgrab:server}:4-13). $\mathsf{PairBalance}$ produces signs to supply to the reordering algorithm  (Algorithm~\ref{alg:reorder}), which, using the current worker permutations $\perm_{\eindex,\windex}$, produces the new per-worker permutations for the next epoch (\dstep{5};  Algorithm~\ref{alg:dgrab:server}:14). In Figure~\ref{fig:diagram}, these correspond to $\perm_{\eindex+1, 1}$ and $\perm_{\eindex+1, 2}$, which the server then sends back to the respective workers (\dstep{6};  Algorithm~\ref{alg:dgrab:server}:15). Lastly, before the start of the next epoch, the workers reorder their examples according to the new permutations (\dstep{7};  Algorithm~\ref{alg:dgrab:workers}:9).

\begin{figure}[t]
\hspace{-.1cm}
\begin{minipage}[t]{.47\linewidth}
\begin{algorithm}[H]
\caption{\dgrab{} Workers}\label{alg:dgrab:workers}
\footnotesize
\begin{algorithmic}[1]
    \Statex \textbf{require:} $\workers$ workers, $\workerexamples \coloneqq \frac{\examples}{\workers}$ ex. per worker
    \Statex \textbf{input:} initial $\weights_1^1$, epochs $\epochs$, learning rate $\alpha$
    \Statex
    \State \textbf{receive:} initial permutations 
    \For{epoch $\eindex \coloneqq 1 \ldots \epochs$}
        \Statex \hspace{.5cm}$\rhd$ Run in parallel for workers $\windex=1 \ldots \workers$
            \For{example $j \coloneqq 1 \ldots \workerexamples$}
            \State \textbf{compute:} $\wexgrad \leftarrow \nabla\loss^\windex(\weights_\eindex^\exindex, \perm_{\eindex, \windex}(\exindex))$ \vspace{.05cm}
            \Statex \hspace{.36\linewidth}$\rarrow{0.61\textwidth}{\text{$j$-th stochastic grad. }\wexgrad}$ \vspace{-.7cm}
            \State \textbf{send:} $\wexgrad$
            \vspace{.15cm}
            \Statex \hspace{.42\linewidth}$\larrow{0.56\textwidth}{\text{avg. $\exindex$-th stochastic grad. } {\bar{\vg}_\exindex}}$ \vspace{-.7cm}
            \State \textbf{receive:} $\bar{\vg}_\exindex$ \vspace{.1cm}
            \State \textbf{update:} $\weights_\eindex^{\exindex + 1} \leftarrow \weights_\eindex^\exindex - \alpha \bar{\vg}_\exindex$
            \Statex
        \EndFor
        \Statex \vspace{1.25cm}
        \State \textbf{receive:} next permutation 
        \State \textbf{update:} $\weights_{\eindex+1}^1 \coloneqq \weights_{\eindex}^{\workerexamples + 1}$
  \EndFor
  \State \textbf{return:} $\weights_{\epochs+1} \coloneqq \weights_{\epochs+1}^1$
\end{algorithmic}
\end{algorithm}
\end{minipage}
\hfill
\begin{minipage}[t]{.5\linewidth}
\begin{algorithm}[H]
\caption{\dgrab{} Parameter Server}\label{alg:dgrab:server}
\footnotesize
\begin{algorithmic}[1]
    \Statex \textbf{require:} $\workers$ workers, $\workerexamples \coloneqq \frac{\examples}{\workers}$ ex. per worker
    \Statex \textbf{input:} epochs $\epochs$ \vspace{-.12cm}
    \Statex
    \Statex \hspace{-.52\linewidth}$\larrow{0.38\textwidth}{\{\perm_{1,\windex}\}_{\windex=1}^\workers}$ 
    \vspace{-.35cm}
    \State \textbf{send:} initial permutations $\{\perm_{1,\windex}\}_{\windex=1}^\workers$
  \For{epoch $\eindex \coloneqq 1 \ldots \epochs$} 
        \State \textbf{initialize:} running sum $\vh=\bm{0}$; empty list $\mathcal{S}$ 
        \For{ example $j \coloneqq 1 \ldots \workerexamples$}\vspace{.7cm}
            \State \textbf{receive:} $\{\wexgrad\}_{\windex=1}^\workers$ from all workers $\windex$
            \State \textbf{compute:} avg. gradient: $\bar{\vg}_\exindex\leftarrow\frac{1}{\workers}\sum_{\windex=1}^\workers \wexgrad$
            \State \textbf{send:} $\bar{\vg}_\exindex$ to all the workers
            \For{worker $\windex \coloneqq 1 \ldots \workers$}
                \State \hspace{-.2cm} \textbf{if } $\exindex\bmod 2 = 0$: 
                \State $\vh,\wexsignprev, \wexsign \leftarrow \mathsf{PairBalance}(\vh, \wexgradprev, \wexgrad)$
                \State $\mathcal{S}.\textsf{\scriptsize{append}}(\wexsignprev)$; $\mathcal{S}.\textsf{\scriptsize{append}}(\wexsign)$
            \EndFor
        \EndFor
        \Statex \hspace{.5cm}$\rhd$ Call Alg.~\ref{alg:reorder} for $\windex=1\ldots \workers$ on  $\perm_{\eindex,\windex} \text{ and } \mathcal{S}$
        \looseness=-1
        \State \textbf{compute:} next permutations $\{\perm_{\eindex + 1,\windex}\}_{\windex=1}^\workers$\looseness=-1 \vspace{-.15cm}
        \Statex  \hspace{-.48\linewidth}$\larrow{0.31\textwidth}{\perm_{\eindex + 1,\windex}}$\vspace{-.35cm}
        \State \textbf{send:} $\{\perm_{\eindex + 1,\windex}\}_{\windex=1}^\workers$ to each worker $\windex$
  \EndFor
\end{algorithmic}
\end{algorithm}
\end{minipage}
\caption{\dgrab{} worker and server (here, a parameter server~\citep{li2014ps}) algorithms.} 
\label{alg:dgrab}
\vspace{-.55cm}
\end{figure}

\vspace{-.1cm}
\section{Convergence Analysis}\label{sec:theory}
\vspace{-.1cm}

We next demonstrate formally that our \dgrab{} algorithm (Section~\ref{sec:dgrab:algo}) confers the efficiency benefits of centralized \grab{}  (Section~\ref{sec:cgrab}) to the distributed setting. In brief, our main theoretical results show that \textbf{\dgrab{} enjoys a linear speedup in convergence rate} under two sets of conditions: smoothness (Theorem~\ref{thm:dgrab:smooth}) and the Polyak-\L ojasiewicz (P.L.) condition (Theorem~\ref{thm:dgrab:PL}). \textbf{Both results guarantee that \dgrab{} is faster than distributed random reshuffling (\dshuffle)}. 
Our proofs rely on Corollary 7 from \citet{dwivedi2021kernel}, which shows that, with high probability, $\mathsf{RandomizedBalance}$ (subroutine in Algorithm~\ref{alg:pairbalance}, from \citet{alweiss2021discrepancy}) guarantees a 
$\tilde{O}(1)$ bound to the signed herding objective (\ref{equ:herding:signed_objective}).\footnote{Corollary 7 from \citet{dwivedi2021kernel} improves the result of Theorem 1.1 from \citet{alweiss2021discrepancy}.} 

To begin, we restate this result to cohere with our framework, for which the vectors $\exj$ are gradients in an optimization context: 

\begin{theorem}[\textbf{Corollary 7, \citet{dwivedi2021kernel}}]
\label{statement:alweiss}
    Consider any vectors $\{\vz_\exindex\}_{\exindex=1}^\examples$ ($\exj \in \R^d$) with $\norm{\vz_\exindex}_2 \le 1$ supplied as input to the $\mathsf{RandomizedBalance}$ subroutine in Algorithm~\ref{alg:pairbalance}. 
    Then for any $\delta > 0$, with probability at least $1 - \delta$, $\mathsf{RandomizedBalance}$ outputs a sequence of signs $\{s_\exindex\}_{\exindex=1}^\examples\in \{-1,1\}$ that satisfy $\textstyle \max_{k\in[\examples]}\norm{\sum\nolimits_{\exindex=1}^k s_\exindex\vz_\exindex}_{\infty} \le \tilde{A}$, where 
    $\tilde{A}=\sqrt{2\log(\frac{4d}{\delta})\log(\frac{4N}{\delta})}=\tilde{O}(1)$.


\end{theorem}


To integrate this result with our parallel setting, we need some additional assumptions that are standard in the literature on distributed optimization --- that the variance of the per-example gradients on each worker is uniformly bounded (Assumption~\ref{ass:inner-deviation}), and that the  variance between worker-specific gradients is similarly bounded (Assumption~\ref{ass:outer-deviation}). 
More precisely, following the distributed setup in (\ref{equ:d-grab:main_update}), we denote the  global loss gradient to be $\nabla \loss(\weights)$, each $\windex$-th worker's local loss gradient to be $\nabla \loss^\windex(\weights)$ ($\forall \windex \in [\workers]$), and each $\windex$-th worker's per-example loss gradients to be $\nabla \loss^\windex(\weights; \exindex)$ ($\forall \exindex \in [\workerexamples]$). We assume: 



\begin{assumption}[\textbf{Bounded Gradient Variance}]
\label{ass:inner-deviation}
For all $\windex \in [\workers]$ there exists a constant $\sigma > 0$ such that for all $\exindex \in [\workerexamples]$ and for all $\weights \in \R^d$, it holds that $\norm{\nabla \loss^\windex(\weights; \exindex) - \nabla \loss^\windex(\weights)}_2^2 \le \sigma^2$.
\end{assumption}

\begin{assumption}[\textbf{Bounded Data Heterogeneity}]
\label{ass:outer-deviation}
There exists a constant $\varsigma > 0$ such that $\forall \windex \in[\workers]$,
$\norm{\nabla \loss^\windex(\weights) - \nabla \loss(\weights)}_2^2 \le \varsigma^2.$
\end{assumption}

Lastly, we include one additional assumption from the 
original \grab{} paper~\citep{lu2022grab}: we assume a cross norm $L_{2,\infty}$ (which can be easily adapted to 
$L_2$-smoothness by setting $L_{2,\infty}$ to be $\sqrt{d}L_2$).

\begin{assumption}[\textbf{Smoothness}]
\label{ass:smoothness}
There exists constant $L_{2,\infty}>0 \text{ such that for any }\vw,\vv\in\mathbb{R}^d$, any $\windex \in [\workers]$, and any $\exindex \in [\workerexamples]$, it holds that $\norm{\nabla f^i(\vw; j) - \nabla f^i(\vv; j)}_2 \le L_{2,\infty}\|\vw - \vv\|_\infty$.
\end{assumption}

Given these assumptions, we can prove a convergence guarantee for \dgrab:  

\begin{theorem}
\label{thm:dgrab:smooth}
Suppose that Assumptions~\ref{ass:inner-deviation},\ref{ass:outer-deviation} and \ref{ass:smoothness} hold. 
For any $\delta > 0$, if we set learning rate $\alpha$ to be
\begin{align*}
    \alpha = \min\left\{\frac{1}{16 L_{2,\infty} (2\workerexamples + \tilde{A}/\workers)}, \left(\frac{4 F_1 \workers^2}{42 L_{2,\infty}^2 (\varsigma + \sigma)^2\tilde{A}^2 \workerexamples \epochs + 18L_{2,\infty}^2  \workers^2\workerexamples^3 \sigma^2}\right)^{1/3}\right\}, 
\end{align*}
where $F_1=f(\weights_1) - \inf_{\weights\in\mathbb{R}^d}f(\weights)$ and $\tilde A$ comes from Theorem~\ref{statement:alweiss}. Then, with probability at least $1 - T\delta$, 
\begin{align*}
\frac{1}{\epochs}\sum_{t=1}^{\epochs} \norm{\nabla f(\vw_t)}_2^2 &\le \frac{9 (F_1 L_{2,\infty}(\varsigma + \sigma)\tilde{A})^{2/3}}{(\workers \workerexamples \epochs)^{2/3}} + \frac{(72 F_1 L_{2,\infty}\sigma)^{2/3} + 64F_1 L_{2,\infty} (2 + \tilde{A}/(\workers \workerexamples))}{\epochs}\\
&= \tilde{O}\left(\frac{1}{(mnT)^{2/3}} + \frac{1}{\epochs}\right). 
\end{align*} 
\vspace{-.3cm}
\end{theorem}

We can also prove an accelerated rate for \dgrab{} if we additionally assume the P.L. condition:
\begin{assumption}[\textbf{P.L. Condition}]
\label{ass:PL}
We say the loss function $f$ fulfills the P.L. condition if there exists $\mu>0$ such that for any $\weights\in\R^d$, $\frac{1}{2}\|\nabla f(\vw)\|_2^2 \geq \mu(f(\vw) - \inf_{\vv\in\mathbb{R}^d}f(\vv)).$
\end{assumption}

\begin{theorem}
\label{thm:dgrab:PL}
Suppose that Assumptions~\ref{ass:inner-deviation}, ~\ref{ass:outer-deviation},~\ref{ass:smoothness}, and \ref{ass:PL} hold. 
For any $\delta > 0$, we set constants $\tilde W$ and $C_3$ to be 
\[
    C_3 = \frac{(F_1+\sigma^2/L_{2,\infty})\mu^2}{224L_{2,\infty}^2(\varsigma + \sigma)^2\tilde{A}^2}
    \hspace{1em} \text{ and } \hspace{1em}
    \tilde{W} = W_0(T^2\workers^2\workerexamples^2C_3),
\]
where $\tilde A$ comes from Theorem~\ref{statement:alweiss}, $F_1$ is from Theorem~\ref{thm:dgrab:smooth}, and $W_0$ is the Lambert-W function. 
If we set learning rate $\alpha = \frac{2\tilde{W}}{T\workerexamples\mu}$ and if the number of epochs $T$ satisfies
\[
\epochs \ge 10 + \frac{1}{\mu}32 L_{2,\infty}(2+\tilde{A}/(\workers \workerexamples))W_0((\workers \workerexamples \epochs)^2C_3) = \tilde O(1),
\]
then, with probability at least $1 - T\delta$, 
\text{ it holds that}
\begin{align*}
F_{\epochs+1} &\le \frac{1}{(\workers \workerexamples \epochs)^2}\left(\frac{(F_1 + L_{2,\infty}^2\sigma^2)\tilde{W}}{C_3} + \frac{112L_{2,\infty}^2(\varsigma + \sigma)^2\tilde{A}^2{\tilde{W}}^2}{\mu^3}\right)
= \tilde{O}\left(\frac{1}{(\workers\workerexamples \epochs)^{2}}\right),
\end{align*}
where $F_{\epochs+1}=f(\weights_{\epochs+1}) - \inf_{\weights\in\mathbb{R}^d}f(\weights)$.
\end{theorem}

We prove Theorems~\ref{thm:dgrab:smooth} and~\ref{thm:dgrab:PL} in the Appendix. 
Together, they show that \dgrab{} exhibits a linear speedup in the number of workers $\workers$ over \grab~\citep{lu2022grab}'s convergence rates ($\tilde{O}((\workerexamples\epochs)^{-2/3})$ and $\tilde{O}((\workerexamples\epochs)^{-2})$, respectively).\footnote{For centralized \grab, the total number of examples $\examples=\workerexamples$ and $\workers=1$.} 
under both smoothness and the P.L. condition. 
Further, \dgrab's convergence rate of $\tilde{O}((\workers \workerexamples \epochs)^{-2})$ is faster than many previous rates,\footnote{These exclusively focus on the P.L. case, so we compare \dgrab{} to them under the same condition.\looseness=-1} such as the high probability bound of $\tilde{O}((\workers \workerexamples)^{-1}T^{-2})$ for \dshuffle{} in \citet{yun2021minibatch}. 

\section{\dgrab{} in Practice: Distributed and Simulation Experiments}\label{sec:experiments}

We next verify \dgrab's accelerated convergence on a variety of empirical tasks.\footnote{Our GitHub repository is \href{https://github.com/GarlGuo/CD-GraB}{https://github.com/GarlGuo/CD-GraB}.} 
For ease of comparison, we follow the experimental plan from the original \grab{} paper,\footnote{Following \citet{lu2022grab}, for our LSTM experiment on  WikiText-2, we set the embedding dimension to 32. We note that we can improve perplexity if we set the dimension higher.} and add some additional large-scale logistic regression experiments. 
We also run an ablation study to isolate the effects of different improvements in \dgrab. We do this because online $\mathsf{PairBalance}$ exhibits performance benefits that are separate from parallelism --- namely, removing the need for gradient centering with a stale mean and allowing for higher learning rates (Section~\ref{sec:dgrab:solution}).\footnote{\grab{} can also implement online $\mathsf{PairBalance}$, in place of $\mathsf{Balance}$~\citep{lu2021general} (Appendix).\looseness=-1} 





\vspace{-.2cm}
\begin{figure*}[t]
  \centering
  \begin{minipage}{1.05\linewidth}
      \hspace{-.5cm}
    \includegraphics[width=\columnwidth]{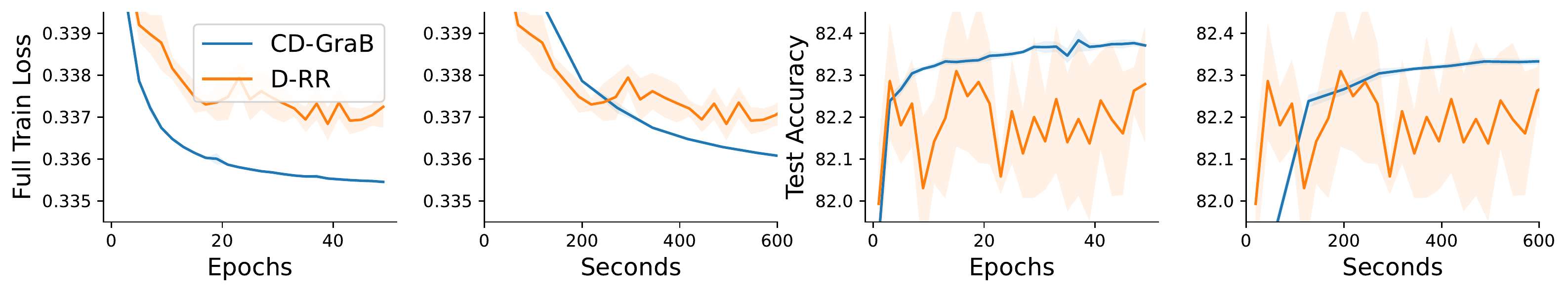}
      \vspace{-.15cm}
      \subcaption{Logistic regression on mortgage application (NY 2017 subset)~\citep{cooper2023variance}}
      \label{fig:exp:ny}
  \end{minipage}
  \begin{minipage}{1.05\linewidth}
    \hspace{-.5cm}
    \includegraphics[width=\columnwidth]{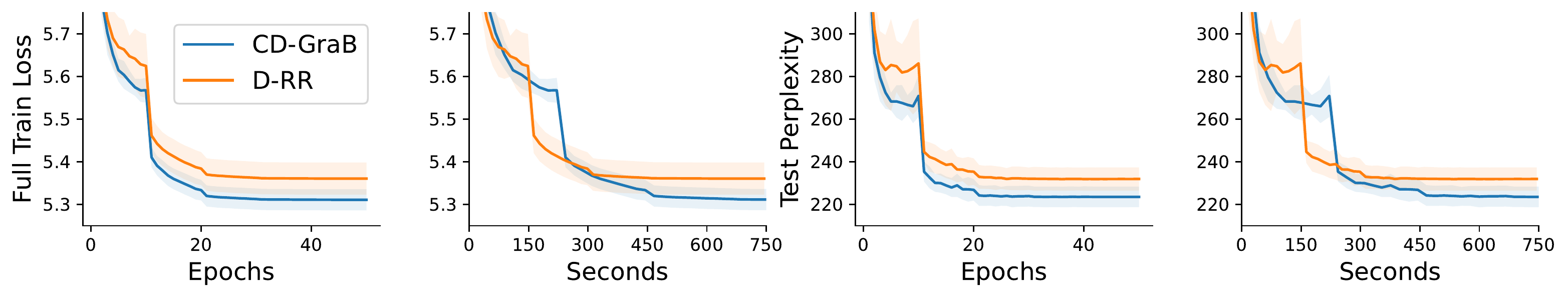}
      \vspace{-.15cm}
      \subcaption{LSTM~\citep{hochreiter1997long} on WikiText-2~\citep{merity2017regularizing}}
      \label{fig:exp:wiki}
    \end{minipage}
    \begin{minipage}{1.05\linewidth}
        \hspace{-.5cm}
    \includegraphics[width=\columnwidth]{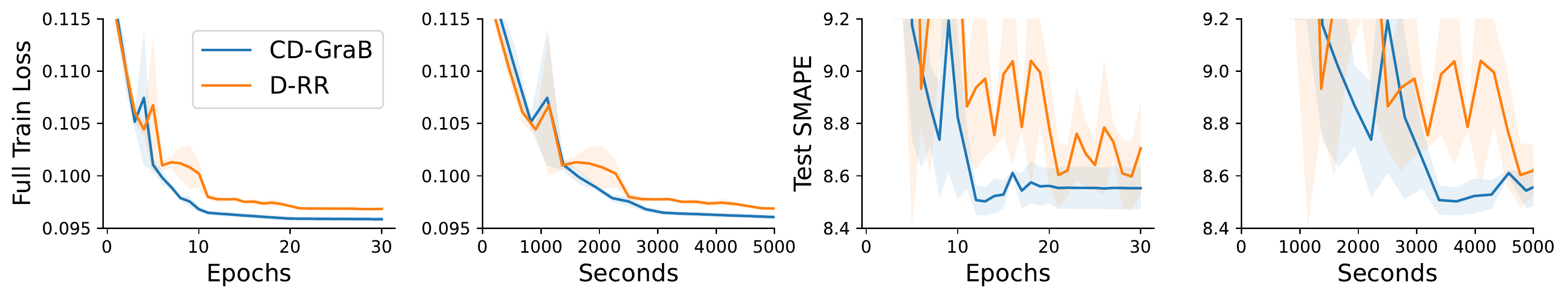}
      \vspace{-.15cm}
      \subcaption{Autoregressive MLP on M4 Weekly~\citep{MAKRIDAKIS202054}}
      \label{fig:exp:m4}
  \end{minipage}
  \caption{Convergence of \dgrab{} in comparison to \dshuffle. For each experiment, we show train loss over epochs and time (\textbf{left} of each subfigure) and test performance over epochs and time (\textbf{right} of each subfigure). We run at least 3 random seeds, and plot the mean $\pm$ STD.} 
  \label{fig:exp}
  \vspace{-.6cm}
\end{figure*}

\vspace{.1cm}
\custompar{Evaluating \dgrab's convergence speedup} We use the following three tasks for evaluating distributed training efficiency: logistic regression on a large-scale mortgage application (New York 2017 subset, 244,107 examples with 18 features)~\citep{cooper2023variance} (Figure~\ref{fig:exp:ny}), Long Short-Term Memory (LSTM)~\citep{hochreiter1997long} on the WikiText-2 dataset~\citep{merity2017regularizing} (Figure~\ref{fig:exp:wiki}), and autoregressive Multi-Layer Perceptron (MLP) on the M4 Weekly dataset~\citep{MAKRIDAKIS202054} (Figure~\ref{fig:exp:m4}). 
We measure the loss incurred on the entire training set (Full Train Loss) and task-appropriate test metrics during evaluation, with respect to both the number of epochs and wall-clock time. Regarding test metrics, we measure test accuracy for the mortgage application, perplexity for WikiText-2, and SMAPE for M4. Additional details regarding the datasets, models, and test metrics can be found in the Appendix. 

For all three tasks, we use a single 128 GiB memory machine with 4 NVIDIA GeForce RTX 2080 Ti GPUs. For the mortgage application and WikiText-2 (Figures~\ref{fig:exp:ny} and~\ref{fig:exp:wiki}), we launch $\workers=4$ workers (processes), where each worker runs on one GPU. For the M4 task, we launch $\workers=32$ workers, where each of the 4 GPUs hosts 8 process workers. We use NCCL as the distributed communication backend~\cite{nccl} for the mortgage application and WikiText-2 tasks, and GLOO~\cite{gloo} as the distributed communication backend for the M4 task. 

As shown in Figure~\ref{fig:exp}, we compare \dgrab{}'s convergence to the standard distributed-training example-ordering method: random reshuffling (D-RR). From all  subfigures in Figure~\ref{fig:exp}, we observe that \dgrab{} outperforms the \dshuffle{} baseline significantly and consistently: \dgrab{} exhibits better training loss and test metrics, measured against both the number of epochs and wall-clock time.
We also note that the results for \dgrab{} are much smoother than for \dshuffle. 
This is likely due to the variance of stochastic gradients during training, which \dgrab{} reduces as a side-effect (so, too, does \grab, in comparison to RR). 
For smoother \dshuffle{} results, we can reduce the learning rate (Appendix). 
\dgrab{} allows for the use of a larger learning rate, which accelerates training while preserving the final model's performance. 

\custompar{Ablation simulation study: the importance of coordination at large scale} \dgrab{} has several design benefits over the original centralized \grab{} algorithm~\citep{lu2022grab}: coordinating parallel workers' specific permutations using $\mathsf{PairBalance}$ on the server (Algorithm~\ref{alg:dgrab}) and removing the dependency on a stale mean (Section~\ref{sec:cgrab}), which enables the ability to using larger learning rates reliably (Section~\ref{sec:dgrab:solution}). Clearly, not all of these benefits come directly from distributing training. For example, being able to use larger learning rates, is a side effect of our solution to develop \dgrab, not our main contribution. Therefore, we run a simulation ablation study to disentangle the relative importance of each of \dgrab's efficiency benefits over \grab. To do so, we compare the convergence of \dgrab{} to two additional baselines in the distributed setting, beyond \dshuffle: 
(1) \textbf{ID-\grab{} (Bal)}, where each independent worker runs \grab{} locally using $\mathsf{RandomizedBalance}$ (subroutine in Algorithm~\ref{alg:pairbalance}) to perform gradient vector balancing; (2) \textbf{ID-\grab{} (PairBal)}, where each independent worker runs \grab{} locally using $\mathsf{PairBalance}$.

Figure~\ref{fig:nodes} summarizes the results, with  convergence curves  for $\workers\in\{4,8,16,32,64\}$ workers training LeNet 
on CIFAR-10. 
We choose this task and architecture to cohere with the experiments done in the original \grab{} paper. For these experiments, we denote $B$ to be the \emph{aggregated} minibatch across all the workers, which refers to the number of stochastic examples used for an overall optimization step; each worker thus has a subset of this minibatch --- an equivalently-sized subset of $B$ examples.\footnote{For example, if we have 4 workers with an aggregated minibatch size of 32, each worker would compute their respective local gradients with 8 examples, and then all-reduce these gradients to obtain the aggregated minibatch gradient for all 32 examples for the optimization step. We discard $\examples \bmod B$ 
examples at random to ensure $\workerexamples$ examples per worker.} 
We make two main observations. First, when scaling up training with more workers, \dgrab{} converges increasingly faster than the no-coordination-ordering methods \textbf{ID-\grab{} (Bal)} and \textbf{ID-\grab{} (PairBal)}. This result aligns with our theory and intuition that, when the number of workers $\workers$ increases, the parallel herding bound (\ref{equ:paraherding:objective}) will increase linearly if there is no coordination. Second, as we scale up to larger $\workers$, the convergence curves of \textbf{ID-\grab{} (Bal)} and \textbf{ID-\grab{} (PairBal)} gradually approach the curve for \dshuffle: 
at larger scales, herding-based example ordering will be no better than  randomly permuting the dataset. Both observations give strong evidence that coordination   
(i.e., running online $\mathsf{PairBalance}$ on the server 
to coordinate per-worker permutations) is critical for accelerating training.



\begin{figure}[!t]
\vspace{-.2cm}
  \centering
    \includegraphics[width=\columnwidth]{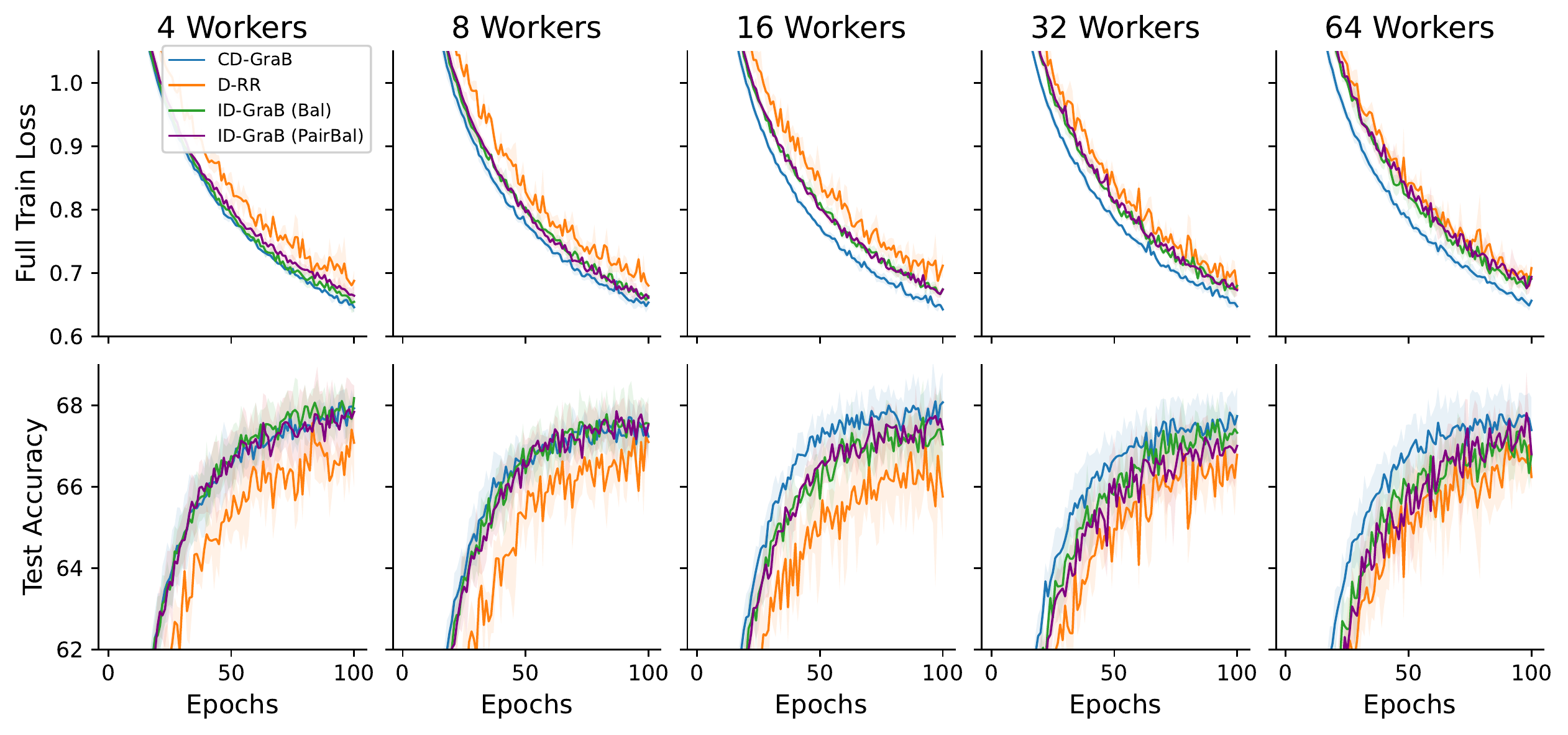}
    \vspace{-.5cm}
    \caption{Convergence for \dgrab, \dshuffle, ID-\grab{} (Bal), and ID-\grab{} (PairBal) training LeNet on CIFAR-10, with $\workers \in \{4, 8, 16, 32, 64\}$ workers. For each, the aggregated minibatch size per update is 64.\looseness=-1} 
  \label{fig:nodes}
  \vspace{-.55cm}
\end{figure}





We note that all of these experiments use SGD, since both the theoretical results of the original \grab{} paper and our results for \dgrab{} here are for SGD. 
In the Appendix, we additionally include results for training GPT-2 on WikiText-103, for which we use AdamW as the optimizer. We find that \cgrab{} with AdamW works in practice; however, our theory results do not directly apply to these experiments. 
We additionally include results on memory usage in the Appendix, which show that \dgrab{} results in negligible overhead in practice.

\vspace{-.3cm}
\section{Conclusion and Future Work: Toward an Order Server Architecture}\label{sec:conclusion}
\vspace{-.2cm}




We elevate the benefits of provably faster, permutation-based example ordering to the contemporary ML distributed-training setting. We focus 
on reformulating the online \textbf{Gra}dient \textbf{B}alancing algorithm (\grab)~\citep{lu2022grab} because, even though it is the provably optimal permutation-based example-ordering method~\citep{cha2023tighter}, it is limited by design to \emph{centralized} settings (Section~\ref{sec:dgrab:issues}). To overcome these limitations, we redesign \grab's herding and balancing framework to account for parallel workers: A \emph{parallel herding} objective, which we solve with an online $\mathsf{PairBalance}$ subroutine, based on key insights from kernel thinning~\citep{dwivedi2021kernel, dwivedi2022generalized, barp2022targeted}. $\mathsf{PairBalance}$ operates on ordered \emph{pairs} of vectors to do \emph{balancing}, which enables our full-stack, low-overhead, \emph{\textbf{C}oordinated} and \emph{\textbf{D}istributed} online \dgrab{} algorithm. 
We give a full specification of our online \dgrab{} algorithm (Section~\ref{sec:dgrab:algo}), provide convergence rate guarantees regarding its speedups on both 1) smooth non-convex and 2) P.L. objectives (Section~\ref{sec:theory}), and verify these speedups in practice on 
single-node distributed tasks and a simulated ablation study (Section~\ref{sec:experiments}). 

Both our theory and experiments demonstrate that \dgrab{} really shines when there are multiple training epochs (Appendix). This is another reason that we do not emphasize experiments involving fine-tuning pre-trained models like GPT-2, as fine-tuning can be achieved in just a couple of epochs. As noted above, it is also more common to train such models using optimizers from the Adam family. In future work, we intend to extend the theory on \grab{} and \dgrab{} to such optimizers, which would make the results on optimal, permutation-based example ordering more useful for base-model pre-training. 

Pre-training from scratch would demonstrate the tremendous power of \dgrab{} to scale to very large models; however, we did not have the training budget to perform such  experiments for the present work. Further, to truly exercise the benefits of \dgrab{} in such large-scale settings, future work should investigate moving beyond the single-node setup that we present. 
Notably, to train larger models, 
our results suggest 
a novel distributed training architecture. The ordering operation performed by the server (Algorithm~\ref{alg:dgrab:server}) is \emph{not} very latency sensitive; the server has the duration of the entire epoch $\eindex$ to compute the new 
permutations for the next, $\eindex + 1$ epoch. Given this relaxed latency requirement, and the success of our algorithmic results, 
it would be an exciting direction for future ML-systems research to invest in building an \emph{Order Server} architecture. Such an architecture, which could be composed with traditional parameter servers, would afford the scalability benefits of \dgrab{} to a host of massive-scale ML applications. 




\section*{Acknowledgments}

A. Feder Cooper is supported by Christopher De Sa's NSF CAREER grant. Yucheng Lu is supported by Meta Ph.D. Fellowship. We also acknowledge a gift from SambaNova Systems.

\bibliography{references}

\begin{thebibliography}{49}
\providecommand{\natexlab}[1]{#1}
\providecommand{\url}[1]{\texttt{#1}}
\expandafter\ifx\csname urlstyle\endcsname\relax
  \providecommand{\doi}[1]{doi: #1}\else
  \providecommand{\doi}{doi: \begingroup \urlstyle{rm}\Url}\fi

\bibitem[glo(2023)]{gloo}
{Collective communications library with various primitives for multi-machine training}, 2023.
\newblock URL \url{https://github.com/facebookincubator/gloo}.

\bibitem[Alweiss et~al.(2021)Alweiss, Liu, and Sawhney]{alweiss2021discrepancy}
Ryan Alweiss, Yang~P Liu, and Mehtaab Sawhney.
\newblock Discrepancy minimization via a self-balancing walk.
\newblock In \emph{Proceedings of the 53rd Annual ACM SIGACT Symposium on Theory of Computing}, pages 14--20, 2021.

\bibitem[Barp et~al.(2022)Barp, Simon-Gabriel, Girolami, and Mackey]{barp2022targeted}
Alessandro Barp, Carl-Johann Simon-Gabriel, Mark Girolami, and Lester Mackey.
\newblock Targeted separation and convergence with kernel discrepancies.
\newblock \emph{arXiv preprint arXiv:2209.12835}, 2022.

\bibitem[Bertsekas(2011)]{bertsekas2011incremental}
Dimitri~P. Bertsekas.
\newblock {Incremental Gradient, Subgradient, and Proximal Methods for Convex Optimization: A Survey}.
\newblock In \emph{{Optimization for Machine Learning}}. The MIT Press, 2011.

\bibitem[Bottou(2012)]{bottou2012stochastic}
L{\'e}on Bottou.
\newblock Stochastic gradient descent tricks.
\newblock In \emph{Neural networks: Tricks of the trade}, pages 421--436. Springer, 2012.

\bibitem[Cha et~al.(2023)Cha, Lee, and Yun]{cha2023tighter}
Jaeyoung Cha, Jaewook Lee, and Chulhee Yun.
\newblock {Tighter Lower Bounds for Shuffling SGD: Random Permutations and Beyond}.
\newblock In \emph{Proceedings of the 40th International Conference on Machine Learning}, ICML'23. JMLR.org, 2023.

\bibitem[Cooper et~al.(2023)Cooper, Lee, Barocas, Sa, Sen, and Zhang]{cooper2023variance}
A.~Feder Cooper, Katherine Lee, Solon Barocas, Christopher~De Sa, Siddhartha Sen, and Baobao Zhang.
\newblock {Is My Prediction Arbitrary? Measuring Self-Consistency in Fair Classification}.
\newblock \emph{arXiv preprint arXiv:2301.11562}, 2023.

\bibitem[De~Sa(2020)]{desa2020shuffle}
Christopher De~Sa.
\newblock {Random Reshuffling is Not Always Better}.
\newblock In \emph{Advances in Neural Information Processing Systems}, 2020.

\bibitem[Devlin et~al.(2019)Devlin, Chang, Lee, and Toutanova]{devlin2019bert}
Jacob Devlin, Ming-Wei Chang, Kenton Lee, and Kristina Toutanova.
\newblock {BERT: Pre-training of Deep Bidirectional Transformers for Language Understanding}.
\newblock In \emph{Proceedings of the 2019 Conference of the North American Chapter of the Association for Computational Linguistics: Human Language Technologies, Volume 1 (Long and Short Papers)}, pages 4171--4186, 2019.

\bibitem[Dwivedi and Mackey(2021)]{dwivedi2021kernel}
Raaz Dwivedi and Lester Mackey.
\newblock Kernel thinning.
\newblock \emph{arXiv preprint arXiv:2105.05842}, 2021.

\bibitem[Dwivedi and Mackey(2022)]{dwivedi2022generalized}
Raaz Dwivedi and Lester Mackey.
\newblock {Generalized Kernel Thinning}.
\newblock In \emph{Tenth International Conference on Learning Representations}, 2022.

\bibitem[Graves et~al.(2017)Graves, Bellemare, Menick, Munos, and Kavukcuoglu]{graves2017automated}
Alex Graves, Marc~G Bellemare, Jacob Menick, Remi Munos, and Koray Kavukcuoglu.
\newblock Automated curriculum learning for neural networks.
\newblock In \emph{international conference on machine learning}, pages 1311--1320. PMLR, 2017.

\bibitem[G{\"{u}}rb{\"{u}}zbalaban et~al.(2019)G{\"{u}}rb{\"{u}}zbalaban, Ozdaglar, and Parrilo]{gurbuzbalaban2019convergence}
Mert G{\"{u}}rb{\"{u}}zbalaban, Asuman~E. Ozdaglar, and Pablo~A. Parrilo.
\newblock {Convergence Rate of Incremental Gradient and Incremental {Newton} Methods}.
\newblock \emph{{SIAM} Journal on Optimization}, 29\penalty0 (4):\penalty0 2542--2565, 2019.

\bibitem[G{\"u}rb{\"u}zbalaban et~al.(2021)G{\"u}rb{\"u}zbalaban, Ozdaglar, and Parrilo]{gurbuzbalaban2021random}
Mert G{\"u}rb{\"u}zbalaban, Asu Ozdaglar, and Pablo~A Parrilo.
\newblock Why random reshuffling beats stochastic gradient descent.
\newblock \emph{Mathematical Programming}, 186\penalty0 (1):\penalty0 49--84, 2021.

\bibitem[HaoChen and Sra(2019)]{haochen2019random}
Jeff~Z. HaoChen and Suvrit Sra.
\newblock {Random Shuffling Beats {SGD} after Finite Epochs}.
\newblock In \emph{Proceedings of the International Conference on Machine Learning}, volume~97, pages 2624--2633, 2019.

\bibitem[Harvey and Samadi(2014)]{harvey2014near}
Nick Harvey and Samira Samadi.
\newblock {Near-Optimal Herding}.
\newblock In \emph{Proceedings of The 27th Conference on Learning Theory}, volume~35, pages 1165--1182, 2014.

\bibitem[Hochreiter and Schmidhuber(1997)]{hochreiter1997long}
Sepp Hochreiter and J{\"u}rgen Schmidhuber.
\newblock Long short-term memory.
\newblock \emph{Neural computation}, 9\penalty0 (8):\penalty0 1735--1780, 1997.

\bibitem[Huang et~al.(2021)Huang, Li, Milzarek, Pu, and Qiu]{huang2021distributed}
Kun Huang, Xiao Li, Andre Milzarek, Shi Pu, and Junwen Qiu.
\newblock {Distributed Random Reshuffling over Networks}.
\newblock \emph{arXiv preprint arXiv:2112.15287}, 2021.

\bibitem[Inan et~al.(2016)Inan, Khosravi, and Socher]{inan2016tying}
Hakan Inan, Khashayar Khosravi, and Richard Socher.
\newblock Tying word vectors and word classifiers: A loss framework for language modeling.
\newblock \emph{arXiv preprint arXiv:1611.01462}, 2016.

\bibitem[Li et~al.(2014)Li, Andersen, Park, Smola, Ahmed, Josifovski, Long, Shekita, and Su]{li2014ps}
Mu~Li, David~G. Andersen, Jun~Woo Park, Alexander~J. Smola, Amr Ahmed, Vanja Josifovski, James Long, Eugene~J. Shekita, and Bor-Yiing Su.
\newblock {Scaling Distributed Machine Learning with the Parameter Server}.
\newblock In \emph{Proceedings of the 11th USENIX Conference on Operating Systems Design and Implementation}, OSDI'14, page 583–598, USA, 2014. USENIX Association.
\newblock ISBN 9781931971164.

\bibitem[Loshchilov and Hutter(2017)]{loshchilov2017decoupled}
Ilya Loshchilov and Frank Hutter.
\newblock Decoupled weight decay regularization.
\newblock \emph{arXiv preprint arXiv:1711.05101}, 2017.

\bibitem[Lu et~al.(2021{\natexlab{a}})Lu, Meng, and De~Sa]{lu2021general}
Yucheng Lu, Si~Yi Meng, and Christopher De~Sa.
\newblock {A General Analysis of Example-Selection for Stochastic Gradient Descent}.
\newblock In \emph{International Conference on Learning Representations}, 2021{\natexlab{a}}.

\bibitem[Lu et~al.(2021{\natexlab{b}})Lu, Park, Chen, Wang, De~Sa, and Foster]{lu2021variance}
Yucheng Lu, Youngsuk Park, Lifan Chen, Yuyang Wang, Christopher De~Sa, and Dean Foster.
\newblock {Variance Reduced Training with Stratified Sampling for Forecasting Models}.
\newblock In \emph{Proceedings of the International Conference on Machine Learning}, pages 7145--7155. PMLR, 2021{\natexlab{b}}.

\bibitem[Lu et~al.(2022)Lu, Guo, and Sa]{lu2022grab}
Yucheng Lu, Wentao Guo, and Christopher~De Sa.
\newblock {GraB: Finding Provably Better Data Permutations than Random Reshuffling}.
\newblock In Alice~H. Oh, Alekh Agarwal, Danielle Belgrave, and Kyunghyun Cho, editors, \emph{Advances in Neural Information Processing Systems}, 2022.
\newblock URL \url{https://openreview.net/forum?id=nDemfqKHTpK}.

\bibitem[Makridakis et~al.(2020)Makridakis, Spiliotis, and Assimakopoulos]{MAKRIDAKIS202054}
Spyros Makridakis, Evangelos Spiliotis, and Vassilios Assimakopoulos.
\newblock The m4 competition: 100,000 time series and 61 forecasting methods.
\newblock \emph{International Journal of Forecasting}, 36\penalty0 (1):\penalty0 54--74, 2020.
\newblock ISSN 0169-2070.
\newblock \doi{https://doi.org/10.1016/j.ijforecast.2019.04.014}.
\newblock URL \url{https://www.sciencedirect.com/science/article/pii/S0169207019301128}.
\newblock M4 Competition.

\bibitem[Malinovsky et~al.(2022)Malinovsky, Mishchenko, and Richt{\'a}rik]{malinovsky2022server}
Grigory Malinovsky, Konstantin Mishchenko, and Peter Richt{\'a}rik.
\newblock {Server-Side Stepsizes and Sampling Without Replacement Provably Help in Federated Optimization}.
\newblock \emph{arXiv preprint arXiv:2201.11066}, 2022.

\bibitem[Matiisen et~al.(2019)Matiisen, Oliver, Cohen, and Schulman]{matiisen2019teacher}
Tambet Matiisen, Avital Oliver, Taco Cohen, and John Schulman.
\newblock Teacher--student curriculum learning.
\newblock \emph{IEEE transactions on neural networks and learning systems}, 31\penalty0 (9):\penalty0 3732--3740, 2019.

\bibitem[McMahan et~al.(2017)McMahan, Moore, Ramage, Hampson, and y~Arcas]{mcmahan2017communication}
Brendan McMahan, Eider Moore, Daniel Ramage, Seth Hampson, and Blaise~Aguera y~Arcas.
\newblock Communication-efficient learning of deep networks from decentralized data.
\newblock In \emph{Artificial intelligence and statistics}, pages 1273--1282. PMLR, 2017.

\bibitem[Merity et~al.(2018)Merity, Keskar, and Socher]{merity2017regularizing}
Stephen Merity, Nitish~Shirish Keskar, and Richard Socher.
\newblock Regularizing and optimizing lstm language models.
\newblock In \emph{International Conference on Learning Representations}, 2018.

\bibitem[Mishchenko et~al.(2020)Mishchenko, Khaled, and Richt{\'{a}}rik]{mishchenko2020random}
Konstantin Mishchenko, Ahmed Khaled, and Peter Richt{\'{a}}rik.
\newblock {Random Reshuffling: Simple Analysis with Vast Improvements}.
\newblock In \emph{Advances in Neural Information Processing Systems}, 2020.

\bibitem[Mohtashami et~al.(2022)Mohtashami, Stich, and Jaggi]{mohtashami2022characterizing}
Amirkeivan Mohtashami, Sebastian Stich, and Martin Jaggi.
\newblock {Characterizing \& Finding Good Data Orderings for Fast Convergence of Sequential Gradient Methods}.
\newblock \emph{arXiv preprint arXiv:2202.01838}, 2022.

\bibitem[Needell et~al.(2014)Needell, Ward, and Srebro]{needell2014stochastic}
Deanna Needell, Rachel Ward, and Nathan Srebro.
\newblock {Stochastic Gradient Descent, Weighted Sampling, and the Randomized Kaczmarz algorithm}.
\newblock In \emph{Advances in Neural Information Processing Systems}, pages 1017--1025, 2014.

\bibitem[NVIDIA(2023)]{nccl}
NVIDIA.
\newblock {NVIDIA Collective Communication Library}, 2023.
\newblock URL \url{https://https://developer.nvidia.com/nccl}.

\bibitem[{PyTorch Contributors}(2023)]{pytorchshuffle}
{PyTorch Contributors}.
\newblock {DataLoader API}, 2023.
\newblock URL \url{https://pytorch.org/docs/stable/data.html}.

\bibitem[Radford et~al.(2019)Radford, Wu, Child, Luan, Amodei, and Sutskever]{radford2019language}
Alec Radford, Jeff Wu, Rewon Child, David Luan, Dario Amodei, and Ilya Sutskever.
\newblock Language models are unsupervised multitask learners.
\newblock 2019.

\bibitem[Rajput et~al.(2022)Rajput, Lee, and Papailiopoulos]{rajput2021permutationbased}
Shashank Rajput, Kangwook Lee, and Dimitris Papailiopoulos.
\newblock {Permutation-Based SGD: Is Random Optimal?}
\newblock In \emph{International Conference on Learning Representations}, 2022.

\bibitem[Recht and R{\'{e}}(2012)]{recht2012toward}
Benjamin Recht and Christopher R{\'{e}}.
\newblock {Toward a Noncommutative Arithmetic-geometric Mean Inequality: Conjectures, Case-studies, and Consequences}.
\newblock In \emph{Conference on Learning Theory}, volume~23, pages 11.1--11.24, 2012.

\bibitem[Sadiev et~al.(2022)Sadiev, Malinovsky, Gorbunov, Sokolov, Khaled, Burlachenko, and Richt{\'a}rik]{sadiev2022federated}
Abdurakhmon Sadiev, Grigory Malinovsky, Eduard Gorbunov, Igor Sokolov, Ahmed Khaled, Konstantin Burlachenko, and Peter Richt{\'a}rik.
\newblock {Federated Optimization Algorithms with Random Reshuffling and Gradient Compression}.
\newblock \emph{arXiv preprint arXiv:2206.07021}, 2022.

\bibitem[Schmidt et~al.(2017)Schmidt, Roux, and Bach]{schmidt2017minimizing}
Mark Schmidt, Nicolas~Le Roux, and Francis~R. Bach.
\newblock Minimizing finite sums with the stochastic average gradient.
\newblock \emph{Mathematical Programming}, 162\penalty0 (1-2):\penalty0 83--112, 2017.

\bibitem[Smith et~al.(2018)Smith, Kindermans, Ying, and Le]{smith2018don}
Samuel~L Smith, Pieter-Jan Kindermans, Chris Ying, and Quoc~V Le.
\newblock {Don't Decay the Learning Rate, Increase the Batch Size}.
\newblock In \emph{International Conference on Learning Representations}, 2018.

\bibitem[Soviany et~al.(2022)Soviany, Ionescu, Rota, and Sebe]{soviany2022curriculum}
Petru Soviany, Radu~Tudor Ionescu, Paolo Rota, and Nicu Sebe.
\newblock {Curriculum learning: A survey}.
\newblock \emph{International Journal of Computer Vision}, pages 1--40, 2022.

\bibitem[Spiliotis et~al.(2020)Spiliotis, Kouloumos, Assimakopoulos, and Makridakis]{spiliotis2020forecasting}
Evangelos Spiliotis, Andreas Kouloumos, Vassilios Assimakopoulos, and Spyros Makridakis.
\newblock Are forecasting competitions data representative of the reality?
\newblock \emph{International Journal of Forecasting}, 36\penalty0 (1):\penalty0 37--53, 2020.

\bibitem[Stephen et~al.(2017)Stephen, Caiming, James, and Socher]{stephen2017pointer}
Merity Stephen, Xiong Caiming, Bradbury James, and Richard Socher.
\newblock Pointer sentinel mixture models.
\newblock \emph{Proceedings of ICLR}, 2017.

\bibitem[Wang et~al.(2018)Wang, Singh, Michael, Hill, Levy, and Bowman]{wang-etal-2018-glue}
Alex Wang, Amanpreet Singh, Julian Michael, Felix Hill, Omer Levy, and Samuel Bowman.
\newblock {GLUE}: A multi-task benchmark and analysis platform for natural language understanding.
\newblock In \emph{Proceedings of the 2018 {EMNLP} Workshop {B}lackbox{NLP}: Analyzing and Interpreting Neural Networks for {NLP}}, pages 353--355, Brussels, Belgium, November 2018. Association for Computational Linguistics.
\newblock \doi{10.18653/v1/W18-5446}.
\newblock URL \url{https://aclanthology.org/W18-5446}.

\bibitem[Welling(2009)]{welling2009herding}
Max Welling.
\newblock Herding dynamical weights to learn.
\newblock In \emph{Proceedings of the 26th Annual International Conference on Machine Learning}, pages 1121--1128, 2009.

\bibitem[Ying et~al.(2017)Ying, Yuan, Vlaski, and Sayed]{ying2017performance}
Bicheng Ying, Kun Yuan, Stefan Vlaski, and Ali~H. Sayed.
\newblock On the performance of random reshuffling in stochastic learning.
\newblock In \emph{2017 Information Theory and Applications Workshop (ITA)}, pages 1--5. IEEE, 2017.

\bibitem[Yuan et~al.(2022)Yuan, He, Davis, Zhang, Dao, Chen, Liang, Re, and Zhang]{yuan2022decentralized}
Binhang Yuan, Yongjun He, Jared Davis, Tianyi Zhang, Tri Dao, Beidi Chen, Percy~S Liang, Christopher Re, and Ce~Zhang.
\newblock Decentralized training of foundation models in heterogeneous environments.
\newblock \emph{Advances in Neural Information Processing Systems}, 35:\penalty0 25464--25477, 2022.

\bibitem[Yun et~al.(2021{\natexlab{a}})Yun, Rajput, and Sra]{yun2021minibatch}
Chulhee Yun, Shashank Rajput, and Suvrit Sra.
\newblock {Minibatch vs Local SGD with Shuffling: Tight Convergence Bounds and Beyond}.
\newblock In \emph{International Conference on Learning Representations}, 2021{\natexlab{a}}.

\bibitem[Yun et~al.(2021{\natexlab{b}})Yun, Sra, and Jadbabaie]{yun2021can}
Chulhee Yun, Suvrit Sra, and Ali Jadbabaie.
\newblock {Open Problem: Can Single-Shuffle {SGD} be Better than Reshuffling {SGD} and {GD}?}
\newblock In \emph{Conference on Learning Theory}, 2021{\natexlab{b}}.

\end{thebibliography}
\bibliographystyle{plainnat}

\newpage
\appendix
\onecolumn

\section{Glossary}\label{app:sec:gloassary}

\begingroup
\setlength{\tabcolsep}{8pt} 
\renewcommand{\arraystretch}{1.1} 
\begin{table}[H]
\small
\begin{center}
      \centering
        \begin{tabular}{p{0.12\linewidth}p{0.05\linewidth}p{0.7\linewidth}}
\toprule
\textbf{Term/Symbol} && \textbf{Explanation} \\
\midrule
\cgrab & & Centralized online Gradient Balancing algorithm. We use this term to refer to the centralized algorithm developed in the original \grab{} paper (\citet{lu2022grab}). \\\midrule
\dgrab & & Coordinated and distributed online Gradient Balancing algorithm. We use this term to refer to the algorithm that constitutes our main contribution. \\\midrule
ID-\grab & & Independent and distributed online gradient balancing. We implement this for our ablation study, to compare with coordinated and distributed \grab. There are two variants: One which uses the original \grab{} paper's online $\mathsf{Balance}$ algorithm (ID-\grab{} (Bal)), and one which implements our online $\mathsf{PairBalance}$ algorithm (ID-\grab{} (PairBal)).\\\midrule
\shuffle & & Random reshuffling algorithm. We use this to refer to its centralized variant.\\\midrule
\dshuffle & & Distributed random reshuffling algorithm.\\\midrule
\so & & Shuffle Once algorithm.\\\midrule
$\dataex$ && Data-example vector; we do not use this in the math in the main paper, but do refer to examples in our schematic description for $\mathsf{PairBalance}$ ordering in Figure~\ref{fig:diagram}.\\\midrule
$\ex$ & &  Vector (for illustration under the herding context). For \cgrab{} and \dgrab, these are gradients. \\\midrule
$\barex$ & & The average vector (for illustration under the herding context).\\\midrule
$\exj$ & & The $\exindex$-th component of a vector (for illustration under the herding context).\\\midrule
$\exij$ & & The $\exindex$-th component of the gradient on worker $\windex$ (for illustration under our parallel herding framework).\\\midrule
$\weights$ & & Parameters / model-weights vector.\\\midrule
$\loss$ & & Loss function. \\\midrule
$\nabla\loss(\weights)$ & & Global loss gradient. \\\midrule
$\nabla\loss^\windex(\weights)$ & & Local $\windex$-th worker's loss gradient. \\\midrule
$\nabla\loss^\windex(\weights; \exindex)$ & & Local $\windex$-th worker's, $\exindex$-th example's loss gradient. \\\midrule
$\perm$ & & A permutation; we study permutation-based example orderings. \\\midrule
$\epochs$ & & Number of epochs.\\\midrule
$\eindex$ & & Index for iterating over $\epochs$ epochs.\\\midrule
$\workers$ & & Number of workers (in this paper, workers are processes, potentially on different GPUs but on the same node). $\workers=1$ in the centralized setting.\\\midrule
$\windex$ & & Index for iterating over $\workers$ workers .\\\midrule
$\workerexamples$ & & Number of training-data examples per worker; equivalent to $\frac{\examples}{\workers}$.\\\midrule
$\examples$ & & Number of total training-data examples. $\examples=\workerexamples$ in the centralized setting.\\\midrule
$\exindex$ & & Index for iterating over examples.\\\midrule
$\g$ & & Gradient, taken with respect to the model weights $\weights$ and data examples $\dataex$.\\\midrule
$\wexgrad$ & & Gradient associated with the $\exindex$-th data example $\dataex$ on worker $\windex$.\\\midrule
$\sgn$ & & A sign, either $+1$ or $-1$; related to the signed herding problem.\\\midrule
$\sgn_\exindex^\windex$ & & A sign, either $+1$ or $-1$, computed according to the $\exindex$-th example gradient $\wexgrad$ for worker $\windex$; to be associated with the example $\dataex_\exindex$ when determining a permutation ordering using Algorithm~\ref{alg:reorder}.\\
\bottomrule
\end{tabular}
\end{center}
\end{table}

\newpage

\section{Additional Details on the \dgrab{} Algorithm and online $\mathsf{PairBalance}$}\label{app:sec:details}

In this Appendix, we provide more details on related work and our contributions. To start, we give a unified description of our online \dgrab{} algorithm with prior work on herding, vector balancing, and kernel thinning (Appendix~\ref{app:sec:details:contribution}), some more details on \citet{alweiss2021discrepancy} that we elide in the main paper due to space constraints (Appendix~\ref{app:sec:alweiss}), conceptual details on implementing \dgrab{} with a parameter server (Appendix~\ref{app:sec:ps}), and implementing our improved balancing algorithm (online $\mathsf{PairBalance}$) in a centralized fashion to get additional improvements for \grab{} (Appendix~\ref{app:sec:central}). 

\subsection{Distinguishing our contributions}\label{app:sec:details:contribution}

We summarize our contributions in relation to prior work in a concise format. This kind of presentation would not be easily understandable without the appropriate background and context that we provide in the paper. This is why present it here, in the Appendix, so that (ideally) this is seen by the reader after finishing the main paper.

We emphasize that it is prior work that:

\begin{itemize}[topsep=0pt, leftmargin=.5cm]
    \item Formulates the herding objective and solves it with vector balancing~\citep{harvey2014near, welling2009herding} (Algorithm~\ref{alg:reorder}).
    \item Leverages ideas from herding and vector balancing (above) in an optimization setting to do permutation-based example ordering~\citep{lu2022grab}.
    \item Observes and proves that it is possible to solve the herding objective in $\tilde{O}(1)$ by only examining differences on pairs of examples (the overarching idea of $\mathsf{PairBalance}$~\citep{dwivedi2021kernel}, which relies on the online $\mathsf{RandomizedBalance}$ subroutine~\citep{alweiss2021discrepancy}; see Algorithm~\ref{alg:pairbalance}). 
\end{itemize}

Our contributions are to bring together all of this prior work in a novel way. We 

\begin{itemize}[topsep=0pt, leftmargin=.5cm]
    \item Translate the herding and balancing framework to the parallel setting via defining a parallel herding objective (\ref{equ:paraherding:objective}).
    \item Leverage prior work on herding in an optimization setting~\citep{lu2022grab} so that we can do parallel herding in an optimization setting (Section~\ref{sec:dgrab}).
    \item Execute \emph{online} pair balancing on a server (Algorithm~\ref{alg:pairbalance} on a running sum, Figure~\ref{fig:diagram}), i.e., do pair balancing in a streaming and asynchronous (rather than blocking) fashion from gradient vectors produced on distributed workers (Algorithm~\ref{alg:dgrab}), on the flattened sequenced of paired-difference gradients (Section~\ref{sec:dgrab:solution}); this leads to an improvement over \grab, which relies on a stale mean.
\end{itemize}

\subsection{More details on $\mathsf{RandomizedBalance}$ from \citet{alweiss2021discrepancy}} \label{app:sec:alweiss}

In the subroutine for $\mathsf{RandomizedBalance}$ in Algorithm~\ref{alg:pairbalance}, we elide details about how the probability $p$ is computed exactly as in \citet{alweiss2021discrepancy}. We provide a more complete specification in Algorithm~\ref{app:alg:alweiss} written in terms of a single input vector (which, for us, is the vector containing the difference between adjacent gradients). Note that the difference here is in the use of a required parameter, constant upper bound $w$, which is used to compute the probability $p$. For clarity of presentation in the subroutine in Algorithm~\ref{alg:pairbalance}, we have set $w=1$.
\citet{alweiss2021discrepancy} sets this threshold differently, which we still elide for simplicity. 

\begin{algorithm}[h]
\caption{Probabilistic Balancing with Logarithm Bound [\citet{alweiss2021discrepancy}]}\label{app:alg:alweiss}
    \begin{algorithmic}[1]
    \Statex \textbf{require:} parameter $w$, used to compute probability
    \Statex \textbf{input:} current running sum $\vr$ vector, vector $\vz_{\text{diff}}$
    \If{$|\langle \vr, \vz_{\text{diff}} \rangle |>w$ or $\norm{\vr}_\infty>w$}
           \State \textbf{Fail}
    \EndIf
    \State \textbf{compute:} $p\leftarrow \frac{1}{2} - \frac{\langle \vr,\vz_{\text{diff}}\rangle}{2w}$
    \State \textbf{compute:} $\sgn \leftarrow +1$ \hspace{.2em} with probability \hspace{.2em} $p$; 
    \Statex  \hspace{4.25em}$\sgn\leftarrow -1$ \hspace{.2em} with probability \hspace{.2em} $1-p$
    \State \textbf{update:}  $\vr\leftarrow \vr+\sgn\vz_{\text{diff}}$
    \State \textbf{return:}  $\sgn$, $\vr$
    \end{algorithmic}
\end{algorithm}

 In practice, we actually do not use $\mathsf{RandomizedBalance}$ in our online $\mathsf{PairBalance}$. We use the deterministic, greedy-ordering algorithm from the original \citet[Algorithm 5]{lu2022grab} paper: 

 \begin{algorithm}[h]
\caption{Balancing without normalization [\citet{lu2022grab}]}\label{app:alg:greedy}
    \begin{algorithmic}[1]
    \Statex \textbf{input:} current running sum $\vr$ vector, vector $\vz_{\text{diff}}$
    \State \textbf{if } $\norm{\vr + \vz_{\text{diff}}} < \norm{\vr - \vz_{\text{diff}}}$ \textbf{ then } $\sgn \leftarrow +1$ \textbf{ else } $\sgn \leftarrow -1$
    \State \textbf{update:}  $\vr\leftarrow \vr+\sgn\vz_{\text{diff}}$
    \State \textbf{return:}  $\sgn$, $\vr$
    \end{algorithmic}
\end{algorithm}

Note that, unlike \citet{alweiss2021discrepancy} (Algorithm~\ref{app:alg:alweiss}), Algorithm~\ref{app:alg:greedy} from \citet{lu2022grab} cannot end up in a failure state. 

In \citet{alweiss2021discrepancy}, Theorem 1.1 proves the $\tilde{O}(1)$ probabilistic bound for Algorithm~\ref{app:alg:alweiss} (See Theorem~\ref{statement:alweiss}) for a restatement of this result in terms of our work). Corollary 7 of \citet{dwivedi2021kernel} re-proves this result (which they mislabel as \citet{alweiss2021discrepancy}, Theorem 1.2, see \citet[Appendix R, p. 69]{dwivedi2021kernel}). They improve the constants and have a less conservative setting of the thresholds $w$. The proof is also very short and elegant, by relying on their Theorem 3. 

\subsection{Implementing \dgrab{} with a parameter server}\label{app:sec:ps}

For our implementation of \dgrab{}, we use a parameter server architecture~\citep{li2014ps}. For our purposes, this just entails computing the average gradient (used to update the model on all workers) on the server side. That is, the server (other than determining the ordering for the next epoch) also has the function of aggregating gradient information (in this case, a simple mean) to send back to the workers. 

We have the server compute the average $\exindex$-th gradient for illustrative purposes. We could, instead, implement the computation the average gradient as an all-reduce operation, in which each worker broadcasts their gradients to all other workers, so that they can each locally compute the average gradient to update their local models. We implement \dgrab{} using a parameter server pattern to show that this is a plausible architecture to use with our coordinated and distributed example ordering algorithm. We could also implement a full parameter server system, for which the server also coordinates global model updates.

If we kept everything in our implementation the same and switched to all-reduce, then we would no longer be following a parameter server paradigm. In this case, the server would just function to determine example orders. It is this kind of paradigm that suggests the abstraction of an \emph{order server}, which we mention briefly in Section~\ref{sec:conclusion}: A server whose sole responsibility is coordinating worker information to determine example ordering.

In future work, we intend to explore a host of architectural possibilities --- of building a full system that incorporates both traditional parameter server aspects with our new abstraction of an order server. For example, we could have parameter servers and order servers work in tandem in a distributed system to perform model training. To move beyond the single-node implementation we present in this paper, we intend to investigate the benefits and trade-offs associated with such design decisions in an actual implemented system.

\subsection{Centralized online $\mathsf{PairBalance}$} \label{app:sec:central}

In Section~\ref{sec:dgrab:algo}, we provide a schematic diagram of how online  $\mathsf{PairBalance}$ works for a distributed implementation using a parameter server (Figure~\ref{fig:diagram}). We also claim in Section~\ref{sec:dgrab}~that online $\mathsf{PairBalance}$ can be applied to the original centralized \grab{} algorithm for improved empirical performance. We provide a schematic here, in Figure~\ref{app:fig:centralizedpairbal} (analogous to Figure~\ref{fig:diagram}), for online $\mathsf{PairBalance}$ for centralized \grab. 
\newpage

\begin{figure}[t!]
\centering
\includegraphics[width=.95\linewidth]{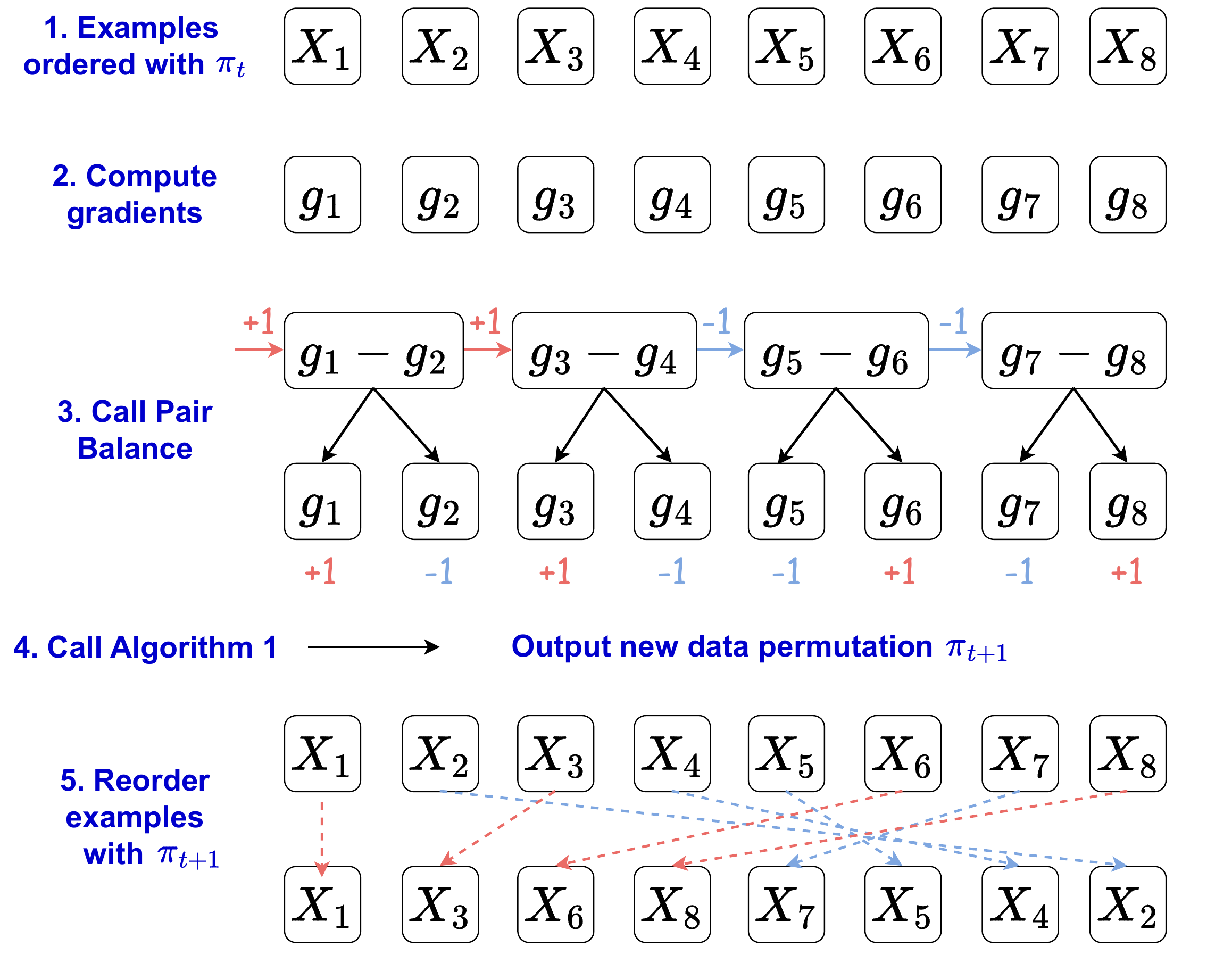}
\caption{Schematic representation of online  $\mathsf{PairBalance}$ for centralized \grab.}
\label{app:fig:centralizedpairbal}
\end{figure}

We also provide empirical results comparing \grab's $\mathsf{Balance}$ routine to the online $\mathsf{PairBalance}$ routine that we instead use in this work.  We observe that both $\mathsf{PairBalance}$ and $\mathsf{Balance}$ would have similar convergence rates under centralized settings, and both outperform \shuffle. 

This experiment justifies the uses of $\mathsf{PairBalance}$ even in centralized learning settings. $\mathsf{PairBalance}$ theoretically tolerates higher learning rates and, as we will justify in Appendix~\ref{sec:appendix-memory}, is more memory-efficient than $\mathsf{Balance}$. In short,   $\mathsf{PairBalance}$ an excellent substitute for $\mathsf{Balance}$ when running \grab. 

\begin{figure}[!h]
    \centering
    \includegraphics[width=\linewidth]{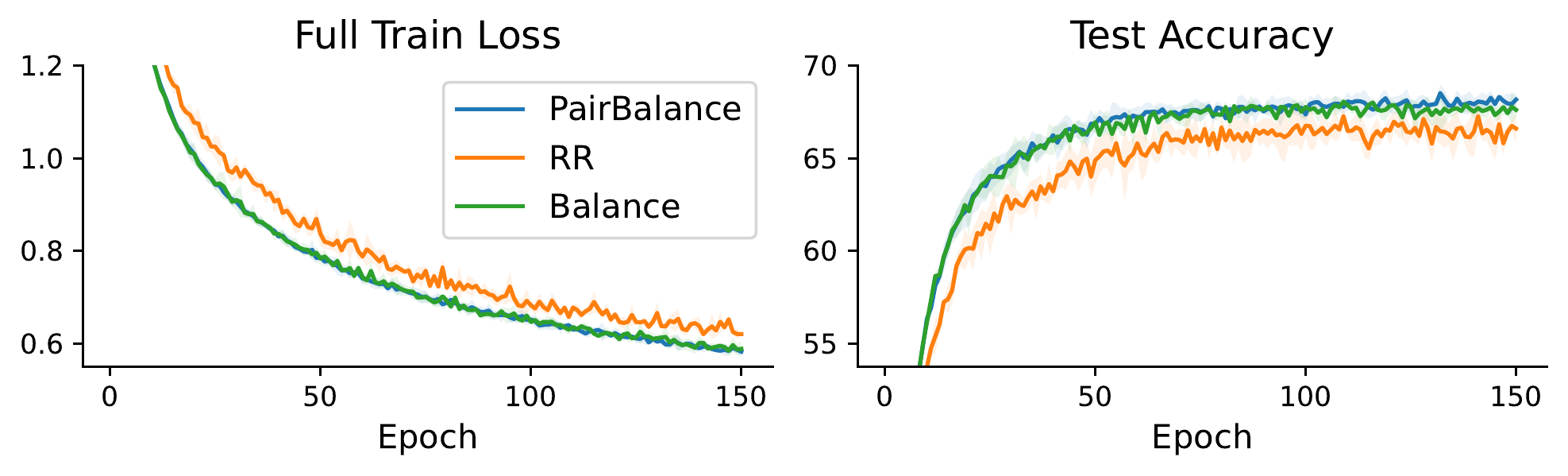}
    \caption{Convergence for centralized online $\mathsf{PairBalance}$ on LeNet on CIFAR-10. We use the identical set of hyperparameters ($\alpha$ = 1e-3, weight decay = 1e-2, momentum = 0.9, $B = 64$) as in the scaling experiments as in Figure~\ref{fig:nodes}.}
    \label{fig:centralized_pair_bal_convergence}
\end{figure}

\section{Proof Results}

We present supporting results, which we use to prove the main results presented in Section~\ref{sec:theory}. First, we show how to analyze the parallel herding bound in terms of a single step over the server-side $\mathsf{PairBalance}$ algorithm (Appendix~\ref{app:sec:proof:herding}). We then include some additional observations/notation (Appendix~\ref{app:proof:note}), which we use in the remaining intermediate results. We prove some intermediate results about how much the loss can change over the course of one epoch, assuming smoothness (Appendix~\ref{app:proof:smooth}) and bounded gradient variance and heterogeneity (Appendix~\ref{app:proof:other}). We combine these results to get one more intermediate result about the maximum the loss  can change on average over many epochs (Appendix~\ref{app:proof:last}), which we then use altogether to prove the two theorems that we present in the main paper (Appendix~\ref{app:thm:proof}). 

\subsection{Analyzing the parallel herding bound}\label{app:sec:proof:herding}

In the main paper, we cover how \dgrab{} runs on both the worker- and server-side. In this section, we dive deeper into the example-ordering part of \dgrab{}, and demonstrate in theory how server-side online $\mathsf{PairBalance}$ reduces the parallel herding bound (\ref{equ:paraherding:objective}), as formulated in Section~\ref{sec:dgrab}. We conclude this section by presenting Lemma~\ref{lem:pair-balance}, which shows server-side $\mathsf{PairBalance}$ is able to iteratively reduce the parallel herding bound.

To begin, we formalize our illustration over a group of vectors (since vector balancing, including $\mathsf{PairBalance}$, does not inherently involve an optimization context until we use it in our online setting on gradients). 
Without loss of generality, we assume that the $\examples$ examples are divided evenly among the $\workers$ workers 
and that $\workerexamples$ is even.
That is, we consider that we are 
given a set of vectors $\exij \in \R^d$ for $\windex \in [\workers]$ and $\exindex \in [\workerexamples]$ evenly located on $\workers$ workers (i.e., $\workerexamples=\frac{\examples}{\workers}$), where $\exij$ denotes the $\exindex$-th vector located on the $\windex$-th worker. Now denote $\perm_{\windex}$ as the original permutation of the vectors on worker $\windex$. Consider running Algorithm~\ref{alg:dgrab:vector:server} on the server side over these $\examples$ vectors. 

\begin{figure}[h!]
\vspace{-.5cm}
\begin{algorithm}[H]
\caption{Server-side $\mathsf{PairBalance}$ over a set of vectors (one step)}\label{alg:dgrab:vector:server}
\begin{algorithmic}[1]
    \Statex \textbf{require:} $\workers$ workers, $\workerexamples \coloneqq \frac{\examples}{\workers}$ vectors per worker
    \Statex \textbf{input:} initial permutations for all the workers $\{\perm_{\windex}\}_{\windex=1}^\workers$
    \State \textbf{initialize:} new permutations for all the workers $\{\perm'_{\windex}\}_{\windex=1}^\workers$
    \State  \textbf{initialize:} running partial sum $\vh=\bm{0}$
    \State  \textbf{initialize:} new indices front (left) pointer $\{l_i=1\}_{\windex=1}^\workers$
    \State  \textbf{initialize:} new indices back (right) pointer $\{r_i=1\}_{\windex=1}^\workers$
    \For{ example $j \coloneqq 1 \ldots \workerexamples$}
        \For{worker $\windex \coloneqq 1 \ldots \workers$}
            \If{$\exindex\bmod 2 = 0$} \hspace{.25cm} $\rhd$ If at an even index, i.e., can examine a full pair of examples
                \State $\vh,\; \wexsignprev,\; \wexsign \leftarrow \mathsf{PairBalance}(\vh,\; 
                \vz_{\exindex-1}^{\windex}, \vz_{\exindex}^{\windex})$
                \If{$\wexsignprev=+1$}
                    \State $\perm'_{\windex}(l_\windex)=\exindex-1$; \hspace{.5em} $l_\windex=l_\windex+1$ \hspace{.25cm} $\rhd$ Append first in pair to the front/left
                    \State $\perm'_{\windex}(r_\windex)=\exindex$; \hspace{.5em} $r_\windex=r_\windex-1$ \hspace{.25cm} $\rhd$ Append second in pair to the right/back
                \Else
                    \State $\perm'_{\windex}(l_\windex)=\exindex$; \hspace{.5em} $l_\windex=l_\windex+1$ \hspace{.25cm} $\rhd$ Append second in pair to the left/front
                    \State $\perm'_{\windex}(r_\windex)=\exindex-1$; \hspace{.5em} $r_\windex=r_\windex-1$ \hspace{.25cm} $\rhd$ Append first in pair to the right/back
                \EndIf
            \EndIf
        \EndFor
    \EndFor
    \State \textbf{output:} new permutations for all $\workers$ workers $\{\perm'_{\windex}\}_{\windex=1}^\workers$
\end{algorithmic}
\end{algorithm}
\vspace{-.4cm}
\caption{One-step $\mathsf{PairBalance}$ algorithm on the server side to solve the parallel herding problem (\ref{equ:paraherding:objective}). This algorithm can be seen as a prototype for Algorithms~\ref{alg:dgrab:workers} and~\ref{alg:dgrab:server}, without the optimization context.}
\vspace{-.1cm}
\end{figure}

It follows, in the following Lemma~\ref{lem:pair-balance}, that we can get the parallel herding bound with the output permutations $\{\perm'_\windex\}_{\windex=1}^\workers$ from Algorithm~\ref{alg:dgrab:vector:server}:

\begin{lemma}
\label{lem:pair-balance}
Suppose that we have a set of vectors $\exij \in \R^d$ for all $\windex, \windex' \in [\workers]$ and for all $\exindex, \exindex' \in [\workerexamples]$ that satisfies
\begin{align*}
    \norm{\sum_{\windex=1}^\workers \sum_{\exindex=1}^\workerexamples   \exij}_\infty \le c_1 \hspace{2em}\text{and}\hspace{2em}
    \norm{\vz_{\windex',\exindex'} - \frac{1}{\workers\workerexamples}\sum_{\windex=1}^\workers \sum_{\exindex=1}^\workerexamples  \exij}_\infty \le c_2 
\end{align*}

for some constants $c_1>0$ and $c_2>0$.
If we run Algorithm~\ref{alg:dgrab:vector:server} over these vectors, then, for any $\delta>0$, it holds with probability at least $1-\delta$ that 
\begin{align*}
\max_{l \in [\workerexamples]}  \norm{\sum_{\windex=1}^\workers\sum_{\exindex=1}^l \ex_{\windex,\perm'_\windex(\exindex)}}_\infty \le & \;\; \frac{1}{2}\max_{l \in [\workerexamples]}\norm{\sum_{\windex=1}^\workers \sum_{\exindex=1}^{l} \ex_{\windex,\perm_\windex(\exindex)}}_\infty  + c_1 + \tilde{A}c_2,
\end{align*}
where $\tilde{A}$ comes from Theorem~\ref{statement:alweiss}.
\end{lemma}

Lemma~\ref{lem:pair-balance} shows that $\mathsf{PairBalance}$ reduces the parallel herding objective (\ref{equ:paraherding:objective}) towards a constant (invariant to $n$) at each step. This implies that, if we repeatedly call $\mathsf{PairBalance}$ on a given permutation, it will return a permutation that guarantees the parallel herding bound to be $\tilde{O}(1)$.

\begin{proof}
We prove this lemma by defining the following auxiliary sequence of pair differences, as in Section~\ref{sec:dgrab:solution}

\[
    \vy_{\workerexamples\cdot(k-1)+\windex} = \ex_{\windex,\perm_\windex(2k-1)} - \ex_{\windex,\perm_\windex(2k)}, \;\; \forall k\in[\workerexamples/2],
\]

which we also can refer to as $\{\vy_\exindex\}_{\exindex=1}^{\workers\workerexamples/2}$.

We also leverage Theorem~\ref{statement:alweiss}, which we reprint below for clarity of presentation:

\setcounter{theorem}{0}


\begin{theorem}[\textbf{Corollary 7, \citet{dwivedi2021kernel}}]
    Consider any vectors $\{\vz_\exindex\}_{\exindex=1}^\examples$ ($\exj \in \R^d$) with $\norm{\vz_\exindex}_2 \le 1$ supplied as input to the $\mathsf{RandomizedBalance}$ subroutine in Algorithm~\ref{alg:pairbalance}. 
    Then for any $\delta > 0$, with probability at least $1 - \delta$, $\mathsf{RandomizedBalance}$ outputs a sequence of signs $\{s_\exindex\}_{\exindex=1}^\examples\in \{-1,1\}$ that satisfy $\textstyle \max_{k\in[\examples]}\norm{\sum\nolimits_{\exindex=1}^k s_\exindex\vz_\exindex}_{\infty} \le \tilde{A}$, where 
    $\tilde{A}=\sqrt{2\log(\frac{4d}{\delta})\log(\frac{4N}{\delta})}=\tilde{O}(1)$.


\end{theorem}

Note that the reordering part of Algorithm~\ref{alg:dgrab:vector:server} (line 8) gives a sequence of signs $\{\sgn_\exindex\}_{\exindex=1}^{\workers\workerexamples/2}$. Therefore, by Theorem~\ref{statement:alweiss}, the sequence $\{\vy_\exindex\}_{\exindex=1}^{\workers\workerexamples/2}$ satisfies

\begin{align}
\label{app:lemma1:a}
    \max_{P \in [\workers\workerexamples/2]} \norm{\sum_{p=1}^P \sgn_p \vy_p}_{\infty} \le 2\tilde{A}c_2,
\end{align}

since (based on what is given in Lemma~\ref{lem:pair-balance})
\begin{align*}
    \norm{\vy_{\workerexamples(k-1)+\windex}}_\infty \le \norm{\vz_{\windex,\perm_{\eindex,\windex}(2k-1)} - \frac{1}{\workers\workerexamples}\sum_{\windex=1}^\workers\sum_{\exindex=1}^\workerexamples \vz_{\windex,\exindex}}_\infty + \norm{\vz_{\windex,\perm_{\eindex,\windex}(2k)} - \frac{1}{\workers\workerexamples}\sum_{\windex=1}^\workers\sum_{\exindex=1}^\workerexamples \vz_{\windex,\exindex}}_\infty \le 2c_2.
\end{align*}

Note that, if $\sgn_{\windex,k}$ is the sign associated with $\vy_{\workerexamples(k-1)+\windex}$, then $\vz_{\windex,\perm_{\eindex,\windex}(2k-1)}$ and $\vz_{\windex,\perm_{\eindex,\windex}(2k)}$ will receive opposite signs $\sgn_{\windex,k}$ and $-\sgn_{\windex,k}$, respectively. 

We denote $\vx^+_{\windex,k}$ to be the example that receives sign $\sgn_{\windex,k} = +1$ and $\vx^-_{\windex,k}$ to be the example that receives sign $\sgn_{\windex,k} = -1$. 

That is, if $\sgn_{\windex,k} = +1$, then $\vx^+_{\windex,k} = \vz_{\windex,\perm_\windex(2k-1)}$, otherwise, if $\sgn_{\windex,k} = -1$, then $\vx^+_{\windex,k} = \vz_{\windex,\perm_\windex(2k)}$; and, $\vx^-_{\windex,k}$ is the other term of the pair $\{\vz_{\windex,\perm_\windex(2k-1)}, \vz_{\windex,\;\perm_\windex(2k)}\}$.

Now, for $K \in [\frac{\workerexamples}{2}]$, let
\begin{align*}
    \kappa_{\windex,K} &= \sum_{k=1}^K (\vz_{\windex,\perm_\windex(2k-1)} + \vz_{\windex,\perm_\windex(2k)}) \hspace{.5em} \text{and}\\
    \upsilon_{i,K} &= \sum_{k=1}^K (\sgn_{\windex,k} \vz_{\windex,\perm_\windex(2k-1)} - \sgn_{i,k} \vz_{\windex,\perm_\windex(2k)}).
\end{align*}
Then
\begin{align*}
    \sum_{k=1}^K \vx^+_{\windex,k} = \frac{1}{2} (\kappa_{\windex,K} + \upsilon_{\windex,K}) \hspace{.5cm} \text{and} \hspace{.5cm}
    \sum_{k=1}^K \vx^-_{\windex,k} = \frac{1}{2} (\kappa_{\windex,K} - \upsilon_{\windex,K}).
\end{align*}

Now, observe that

\begin{align*}
    \sum_{\windex=1}^\workers \kappa_{\windex,K} = \sum_{\exindex=1}^{2K} \sum_{\windex=1}^\workers \vz_{\windex,\perm_\windex(\exindex)} \hspace{.5cm} \text{and} \hspace{.5cm}
    \sum_{\windex=1}^\workers \upsilon_{\windex,K} = \sum_{p=1}^{\workers K} \sgn_p \vy_p.
\end{align*}

Therefore,

\begin{align*}
    \max_{K \in [\workerexamples/2]} \norm{\sum_{k=1}^K \sum_{\windex=1}^\workers \vx^+_{\windex,k}}_\infty &\le \frac{1}{2}\left(\max_{K \in [\workerexamples/2]} \norm{\sum_{\windex=1}^\workers \kappa_{K,\windex}}_\infty + \max_{K \in [\workerexamples/2]} \norm{\sum_{\windex=1}^\workers \upsilon_{K,\windex}}_\infty\right)\\
    &\le \frac{1}{2}\max_{K \in [\workerexamples/2]}\norm{\sum_{\exindex=1}^{2K} \sum_{\windex=1}^\workers \vz_{\windex,\exindex}}_\infty + \tilde{A}c_2 \hspace{1em} \text{By substituting above and (\ref{app:lemma1:a})}\\
    &\le \frac{1}{2}\max_{k \in [\workerexamples]}\norm{\sum_{\exindex=1}^{k} \sum_{\windex=1}^\workers \vz_{\windex,\exindex}}_\infty + \tilde{A}c_2. 
\end{align*}

And similarly,

\begin{align*}
    \max_{K \in [\workerexamples/2]} \norm{\sum_{k=1}^K \sum_{\windex=1}^\workers \vx^-_{\windex,k}} &\le \frac{1}{2}\left(\max_{K \in [\workerexamples/2]} \norm{\sum_{\windex=1}^\workers \kappa_{K,\windex}}_\infty + \max_{K \in [\workerexamples/2]} \norm{\sum_{\windex=1}^\workers \upsilon_{K,\windex}}_\infty\right) \\
    &\le \frac{1}{2}\max_{k \in [\workerexamples]}\norm{\sum_{\exindex=1}^{k} \sum_{\windex=1}^\workers  \vz_{\windex,\exindex}}_\infty + \tilde{A}c_2. 
\end{align*}

Applying the new permutation $\perm'_\windex(\exindex)$ on the vectors $\vz_{\windex,\perm_\windex(\exindex)}$, we get for each $\windex \in [\workers]$ the permuted sequence
\[
    \vx^+_{\windex,1},\dots, \vx^+_{\windex,\workerexamples/2}, \vx^-_{\windex,\workerexamples/2},\dots, \vx^-_{\windex,1}.
\]

Thus, we need to bound the herding objective of the sequence

\[
    \sum_{\windex=1}^\workers \vx^+_{\windex,1},\dots, \sum_{\windex=1}^\workers \vx^+_{\windex,\workerexamples/2}, \sum_{\windex=1}^\workers \vx^-_{\windex,\workerexamples/2},\dots, \sum_{\windex=1}^\workers \vx^-_{\windex,1}.
\]

If the partial sums above peak at $t_0 \le \workerexamples/2$, then we can bound the parallel herding objective as

\begin{align*}
    \norm{\sum_{k=1}^{t_0} \sum_{\windex=1}^\workers \vx^+_{\windex,k}}_\infty = \max_{K \in [\workerexamples/2]} \norm{\sum_{k=1}^K \sum_{\windex=1}^\workers \vx^+_{\windex,k}}_\infty \le \frac{1}{2}\max_{k \in [\workerexamples]}\norm{\sum_{\exindex=1}^{k} \sum_{\windex=1}^\workers \vz_{\windex,\exindex}}_\infty + \tilde{A}c_2;
\end{align*}

otherwise, we can bound the parallel herding objective as

\begin{align*}
    \norm{\sum_{\exindex=1}^\workerexamples \sum_{\windex=1}^\workers \vz_{\windex,\exindex} - \sum_{k=1}^{m-t_0}\sum_{\windex=1}^\workers \vx^-_{\windex,k}}_\infty &\le \norm{\sum_{\exindex=1}^\workerexamples \sum_{\windex=1}^\workers \vz_{\windex,\exindex}}_\infty + \norm{\sum_{k=1}^{m-t_0}\sum_{\windex=1}^\workers \vx^-_{\windex,k}}_\infty\\ 
    &\le c_1 + \frac{1}{2}\max_{t \in [\workerexamples]}\norm{\sum_{\exindex=1}^{t} \sum_{\windex=1}^\workers \vz_{\windex,\exindex}}_\infty + \tilde{A}c_2,
\end{align*}

since in Algorithm~\ref{alg:reorder} the list of vectors with negative signs is reversed before concatenated. 

The claim follows.
\end{proof}

\subsection{Notation and observations}\label{app:proof:note}

We begin with three notes that we will use throughout the intermediate results we present in this section. We will use the lemmas presented here to prove our main results: Theorems~\ref{thm:dgrab:smooth} and~\ref{thm:dgrab:PL} in Appendix~\ref{app:thm:proof}.

\begin{enumerate}[leftmargin=.5cm]
    \item \custompar{A single $t$-th update} First, recall that one $\eindex$-th step of the parameter update can be written as
    \[
    \vw_t^{j+1} = \vw_t^j - \frac{\alpha}{\workers} \sum_{i=1}^\workers \nabla f^i(\vw_t^j; \; \pi_{t,i}(j)), \quad \forall j \in [\workerexamples]
    \]
    We will use the convention $\vw_{t+1} \triangleq \vw_{t+1}^1 \triangleq \vw_t^{\workerexamples+1}$. 
    \item \custompar{The maximum amount a parameter can change over an epoch} The key quantity in our proof is $\Delta_t$, which is the maximum amount that a parameter in $\weights$ can change in epoch $t$. That is,

    \begin{align}
    \label{eq:deltat}
    \Delta_t &\triangleq \max_{k \in [\workerexamples]} \norm{\vw_t^{k+1} - \vw_t}_\infty\nonumber\\
    &= \frac{\alpha}{\workers}\max_{k \in [\workerexamples]}\norm{\sum_{j=1}^k\sum_{i=1}^\workers\nabla f^i(\vw_t^j; \pi_{t,i}(j))}_\infty.
    \end{align}

    Following this definition of $\Delta_t$, we note that the maximum amount that a parameter in $\weights$ can change over two different epochs is $2\Delta_t$.
    That is, we observe 
    
    \begin{align*}
    \norm{\vw_t^j - \vw_t^k}_\infty &\le \norm{\vw_t^j - \vw_t}_\infty + \norm{\vw^k_t - \vw_t}_\infty \le 2 \Delta_t\\
    \norm{\vw_{t+1}^j - \vw_t^k}_\infty &\le \norm{\vw_{t+1}^j - \vw_{t+1}}_\infty + \norm{\vw_{t+1} - \vw_t}_\infty + \norm{\vw_t^k - \vw_t}_\infty \le \Delta_{t+1} + 2\Delta_t, \hspace{.5cm} \forall j,k \in [\workerexamples]. 
    \end{align*}

    We make repeated use of this relation in the results that follow, which we typically will use in combination with the Lipschitz assumption to bound gradients of the same loss function but with different parameters. 

    \item \custompar{Bounding loss at epoch $t$} We will denote $F_t = f(\vw_t) - f(\vw^*)$ where $\vw^*$ is the minimizer of $f$ which we assume to be bounded from below. 
    
\end{enumerate}

\subsection{Assuming $L_{2,\infty}$-smoothness: results on the amount the loss can change over one epoch)}\label{app:proof:smooth}

We will next prove an intermediate result regarding that bounds the loss $\loss$ at epoch $t + 1$ in relation to the loss at the prior epoch $t$ (Lemma~\ref{lemma:loss}).  That is, we prove results about how much the loss with respect to the parameters can change over the course of one epoch.

\begin{lemma}\label{lemma:loss} 
If the loss $\loss$ is $L_{2,\infty}$-smooth and the learning rate $\alpha \le \frac{1}{\workerexamples L_{2,\infty}}$, then
\begin{align}
\label{lem:smoothness:eq:critical}
f(\vw_{t+1}) \le f(\vw_t) + \frac{\alpha \workerexamples L_{2,\infty}^2}{2}\Delta_t^2 - \frac{\alpha \workerexamples}{2} \norm{\nabla f(\vw_t)}_2^2.
\end{align}
\end{lemma}

\begin{proof}
We begin with the definition of $L_{2,\infty}$-smoothness, with respect to loss $\loss$:

\begin{align*}
f(\vw_{t+1}) &\le f(\vw_t) + \nabla f(\vw_t)^\top(\vw_{t+1} - \vw_t) + \frac{L_{2,\infty}}{2}\|\vw_{t+1} - \vw_t\|_2^2\\
\end{align*}

Also observe that 

\begin{align}
\label{eq:normsquareddiff}
-\nabla f(\vw_t)^\top(\vw_t - \vw_{t+1}) &= -\frac{\alpha \workerexamples}{2}2\nabla f(\vw_t)^\top\left(\frac{\vw_t - \vw_{t+1}}{\alpha \workerexamples}\right)\nonumber\\
&= \frac{\alpha \workerexamples}{2} \left( \norm{\nabla f(\vw_t) - \frac{(\vw_t - \vw_{t+1})}{\alpha \workerexamples}}_2^2 - \norm{\nabla f(\vw_t)}_2^2 - \norm{\frac{(\vw_t - \vw_{t+1})}{\alpha \workerexamples}}_2^2 \right). 
\end{align}

Combining the above --- i.e., the definition of $L_{2,\infty}$-smoothness with (\ref{eq:normsquareddiff}) --- we get

\begin{align*}
f(\vw_{t+1}) &\le f(\vw_t) + \frac{\alpha \workerexamples}{2}\norm{\nabla f(\vw_t) - \frac{(\vw_t - \vw_{t+1})}{\alpha \workerexamples}}_2^2 - \frac{\alpha \workerexamples}{2}\norm{\nabla f(\vw_t)}_2^2
&\quad + \frac{\alpha \workerexamples L_{2,\infty} - 1}{2\alpha \workerexamples}\norm{\vw_t - \vw_{t+1}}_2^2.\\
\end{align*}
The last term on the right-hand side is $\le 0$ by the assumption that the learning rate $\alpha \le \frac{1}{\workerexamples L_{2,\infty}}$. Therefore,

\begin{align}
\label{eq:rhs-2}
f(\vw_{t+1}) \le f(\vw_t) + \frac{\alpha \workerexamples}{2}\norm{\nabla f(\vw_t) - \frac{(\vw_t - \vw_{t+1})}{\alpha \workerexamples}}_2^2 - \frac{\alpha \workerexamples}{2}\norm{\nabla f(\vw_t)}_2^2.
\end{align}

We next bound the second term on the right-hand side by $\Delta_t$ (\ref{eq:deltat}):
\begin{align*}
\norm{\nabla f(\vw_t) - \frac{(\vw_{t} - \vw_{t+1})}{\alpha \workerexamples}}_2^2 &= \norm{\frac{1}{\workers \workerexamples}\sum_{j=1}^\workers\sum_{i=1}^\workerexamples \nabla f^i(\vw_t, \pi_t(j)) - \frac{1}{\workers \workerexamples}\sum_{j=1}^\workers \sum_{i=1}^\workerexamples \nabla f^i(\vw_{t}^{j}; \pi_t(j))}_2^2\\
&\le \frac{1}{\workers \workerexamples}\sum_{j=1}^\workers\sum_{i=1}^\workerexamples\norm{\nabla f^i(\vw_t, \pi_t(j)) - \nabla f^i(\vw_{t}^{j}; \pi_t(j))}_2^2\\
&\le \frac{L_{2,\infty}^2}{\workers \workerexamples}\sum_{j=1}^\workers\sum_{i=1}^\workerexamples\norm{\vw_t^j - \vw_t}_\infty^2, 
\end{align*}

where we have used $L_{2,\infty}$-smoothness (Assumption~\ref{ass:smoothness}) in the last inequality. Substituting $\Delta_t$, we get

\begin{align*}
\frac{L_{2,\infty}^2}{\workers \workerexamples}\sum_{j=1}^\workers\sum_{i=1}^\workerexamples\norm{\vw_t^j - \vw_t}_\infty^2 \quad \le  \quad L_{2,\infty}^2 \Delta_t^2. 
\end{align*}

Plugging the above into (\ref{eq:rhs-2}), we get
\begin{align*}
f(\vw_{t+1}) \le f(\vw_t) + \frac{\alpha \workerexamples L_{2,\infty}^2}{2}\Delta_t^2 - \frac{\alpha \workerexamples}{2} \norm{\nabla f(\vw_t)}_2^2, 
\end{align*}

yielding the claim.
\end{proof}

We next build slightly on Lemma~\ref{lemma:loss} to make two additional observations. First: 
\begin{lemma}
    \label{lem:jens}
    If the loss $\loss$ is $L_{2,\infty}$-smooth and the learning rate $\alpha \le \frac{1}{\workerexamples L_{2,\infty}}$, then 
     \begin{align*}
    \frac{1}{\epochs}\sum_{t=1}^{\epochs} \norm{\nabla f(\vw_t)}_2^2 
    &\le \frac{2F_1}{\alpha \workerexamples T} + \frac{L_{2,\infty}^2}{T}\sum_{t=1}^{T}\Delta_{t}^2,
    \end{align*}
where $F_1$ comes from Theorem~\ref{thm:dgrab:smooth}.
\end{lemma}

\begin{proof}
Using Lemma~\ref{lemma:loss} and Jensen's inequality, we average (\ref{lem:smoothness:eq:critical}) over $t \in [\epochs]$ and match terms, yielding 
\begin{align*}
\frac{1}{\epochs}\sum_{t=1}^{\epochs} \norm{\nabla f(\vw_t)}_2^2 &\le \frac{2(f(\vw_1) - f(\vw_{\epochs+1}))}{\alpha \workerexamples \epochs} + \frac{L_{2,\infty}^2}{\epochs}\sum_{t=1}^T\Delta_t^2.
\end{align*}
Substituting $F_1$, we get
\begin{align*}
&\le \frac{2F_1}{\alpha \workerexamples T} + \frac{L_{2,\infty}^2}{T}\sum_{t=1}^T\Delta_t^2,
\end{align*}
yielding the claim.
\end{proof}

\newpage
We next build on Lemma~\ref{lemma:loss} by further assuming the P.L. assumption holds. 

\begin{lemma}
\label{lem:jens-pl}
If the loss $\loss$ is $L_{2,\infty}$-smooth, the learning rate $\alpha \le \frac{1}{\workerexamples L_{2,\infty}}$, and the P.L. assumption (Assumption~\ref{ass:PL}) holds, then, for $\rho = 1 - \frac{\alpha \workerexamples \mu}{2}$ 
\begin{align*}
F_{T+1} \le \rho^T F_1 + \frac{\alpha \workerexamples L_{2,\infty}^2}{2}\sum_{t=1}^T\rho^{T-t}\left(\Delta_t^2 - \frac{1}{2L_{2,\infty}^2} \norm{\nabla f(\vw_t)}_2^2\right).
\end{align*}
\end{lemma}



\begin{proof}
From Lemma~\ref{lemma:loss}, we got (\ref{lem:smoothness:eq:critical}), i.e.,

\begin{align*}
f(\vw_{t+1}) \le f(\vw_t) + \frac{\alpha \workerexamples L_{2,\infty}^2}{2}\Delta_t^2 - \frac{\alpha \workerexamples}{2} \norm{\nabla f(\vw_t)}_2^2, 
\end{align*}

Applying the P.L. assumption (Assumption~\ref{ass:PL}) to (\ref{lem:smoothness:eq:critical}), we get

\begin{align*}
f(\vw_{t+1}) &\le f(\vw_t) + \frac{\alpha \workerexamples L_{2,\infty}^2}{2}\Delta_t^2 - \frac{\alpha \workerexamples}{4} \norm{\nabla f(\vw_t)}_2^2 - \frac{\alpha \workerexamples}{4} \norm{\nabla f(\vw_t)}_2^2\\
&\le f(\vw_t) + \frac{\alpha \workerexamples L_{2,\infty}^2}{2}\Delta_t^2 - \frac{\alpha \workerexamples \mu}{2}(f(\vw_t) - f(\vw^*))  - \frac{\alpha \workerexamples}{4} \norm{\nabla f(\vw_t)}_2^2.
\end{align*}
Subtracting $\loss^*$ from both sides, we get
\begin{align*}
f(\vw_{t+1}) - f^* \le \left(1 - \frac{\alpha \workerexamples \mu}{2} \right)(f(\vw_t) - f^*) + \frac{\alpha \workerexamples}{2}\left(L_{2,\infty}^2\Delta_t^2 - \frac{1}{2} \norm{\nabla f(\vw_t)}_2^2\right). 
\end{align*}
For $\rho = 1 - \frac{\alpha \workerexamples \mu}{2}$, we then apply the above inequality recursively for $t \in [T]$, yielding the claim: 
\begin{align*}
F_{T+1} \le \rho^T F_1 + \frac{\alpha \workerexamples L_{2,\infty}^2}{2}\sum_{t=1}^T\rho^{T-t}\left(\Delta_t^2 - \frac{1}{2L_{2,\infty}^2} \norm{\nabla f(\vw_t)}_2^2\right). 
\end{align*}
\end{proof}

\subsection{Assuming bounded gradient variance and heterogeneity: results applying Algorithm~\ref{alg:dgrab:vector:server}}\label{app:proof:other}

We next prove a result that builds on Lemma~\ref{lem:pair-balance} and our one-step version of the server-side $\mathsf{PairBalance}$ algorithm (Algorithm~\ref{alg:dgrab:vector:server}). 

We begin by introducing some additional notation. Namely, we will call $\pi^{-1}$ the operation that, given an example, yields the index in the permutation for that example. For instance, $\pi_{t+1, \windex}(\exindex)$ returns the example at the $\exindex$-th index for the $\windex$-th worker's $t+1$ permutation.  Let us denote that example $\tau$. Then, $\pi_{t,i}^{-1}\pi_{t+1,i}(j)$ is equivalent to applying $\pi_{t,i}^{-1}$ to $\tau$: it takes the example $\tau$ and returns $\tau$'s associated index in the $\windex$-th worker's epoch $t$'s permutation (in this case, the prior epoch's permutation). 

We will make use of this notation in the following Lemma.

\begin{lemma}
Assume bounded gradient variance (Assumption~\ref{ass:inner-deviation}), bounded gradient heterogeneity (Assumption~\ref{ass:outer-deviation}), and  $L_{2,\infty}$-smoothness (Assumption~\ref{ass:smoothness}). 
For $t \in [T]$ and $\delta > 0$, if we apply Algorithm~\ref{alg:dgrab:vector:server} to the gradients $\nabla f^i(\vw_t^j; \pi_t^i(j))$ at epoch $t$ to produce the next permutation $\pi_{t+1,i}$ for epoch $t+1$, then, with probability at least $1 - \delta$, 
\begin{align*}
\Delta_{t+1} \quad \le \quad \frac{1}{2}\Delta_t \;\; + \;\; \alpha L_{2,\infty} \left(4\workerexamples + \frac{2\tilde{A}}{\workers}\right)\Delta_{t} \;\; + \;\; \alpha \workerexamples L_{2,\infty} \Delta_{t+1} \;\; + \;\; \frac{\alpha (\varsigma + \sigma)\tilde{A}}{\workers} \;\; + \;\; \alpha \workerexamples\norm{\nabla f(\vw_{t+1})}_2, 
\end{align*}
\label{lem:grad-balance}
where $\tilde A$ comes from Theorem~\ref{statement:alweiss}.
\end{lemma}

\begin{proof}

We start with the triangle inequality:
\begin{align}
\label{lem:grad-balance:eq:triangle}
\norm{\sum_{j=1}^{k}\sum_{i=1}^\workers \nabla f^i(\vw_{t+1}^j; \pi_{t+1,i}(j))}_\infty \le \norm{\sum_{j=1}^{k}\sum_{i=1}^\workers \nabla f^i(\vw_{t}^{\pi_{t,i}^{-1}\pi_{t+1,i}(j)}; \pi_{t+1,i}(j))}_\infty + \nonumber\\
\norm{\sum_{j=1}^{k}\sum_{i=1}^\workers \left(\nabla f^i(\vw_{t+1,i}^j, \pi_{t+1,i}(j)) - \nabla f^i(\vw_{t,i}^{\pi_{t,i}^{-1}\pi_{t+1,i}(j)}; \pi_{t+1,i}(j))\right)}_\infty
\end{align}

We use Lemma~\ref{lem:pair-balance} to bound the first term on the right-hand side of (\ref{lem:grad-balance:eq:triangle}) from Lemma~\ref{lemma:loss}. 

That is, let 
\[
    \vz_{i,j} = \nabla f^i(\vw_{t}^{\pi_{t,i}^{-1}(j)}; j),
\]
so that 
\[
    \vz_{i,\pi_{t+1,i}(j)} = \nabla f^i(\vw_{t}^{\pi_{t,i}^{-1}\pi_{t+1,i}(j)}; \pi_{t+1,i}(j)).
\]

The upper bounds for $\norm{\vz_{i,j} - \frac{1}{\workers \workerexamples}\sum_{r,s} \vz_{r,s}}_\infty$ and $\norm{\sum_{i,j} \vz_{i,j}}_\infty$ are:

\begin{align*}
    \norm{\nabla f^i(\vw_t^{j};\pi_{t,i}(j)) - \frac{1}{\workers \workerexamples}\sum_{r=1}^\workers \sum_{s=1}^\workerexamples \nabla f^s(\vw_t^{r};\pi_{t,s}(r))}_\infty,
\end{align*}

which are 
\begin{align*}
\le &\norm{\nabla f^i(\vw_{t}^{j}; \pi_{t,i}(j)) - \frac{1}{\workers \workerexamples}\sum_{r=1}^\workers \sum_{s=1}^\workerexamples \nabla f^s(\vw_{t}^{j};\pi_{t,s}(r))}_\infty + \\ 
    & \quad \norm{\frac{1}{\workers \workerexamples}\sum_{r=1}^\workers \sum_{s=1}^\workerexamples \nabla f^s(\vw_{t}^{j}; \pi_{t,s}(r)) - \frac{1}{\workers \workerexamples}\sum_{r=1}^\workers \sum_{s=1}^\workerexamples \nabla f^s(\vw_{t}^{r}; \pi_{t,s}(r))}_\infty.
\end{align*}

We can rewrite the above to be

\begin{align*}
&\le \norm{\nabla f^i(\vw_{t}^{j}; \pi_{t,i}(j)) - \nabla f(\vw_{t}^{j})}_\infty + \frac{L_{2,\infty}}{\workers \workerexamples}\sum_{r=1}^m\workers \sum_{s=1}^\workerexamples \norm{\vw_t^j - \vw_t^r}_\infty\\
    &\le \varsigma + \sigma + 2 L_{2,\infty}\Delta_{t},
\end{align*}

by Assumptions~\ref{ass:inner-deviation},~\ref{ass:outer-deviation}, and~\ref{ass:smoothness}, and by the definition of $\Delta_t$ (\ref{eq:deltat}). 

Now, observe that

\begin{align*}
\norm{\sum_{i=1}^\workers \sum_{j=1}^\workerexamples \nabla f^i(\vw_t^j;\pi_{t,i}(j))}_\infty \le \norm{\sum_{i=1}^\workers \sum_{j=1}^\workerexamples\nabla f^i(\vw_{t}^{j}; \pi_{t,i}(j)) - \sum_{i=1}^\workers\sum_{j=1}^\workerexamples\nabla f^i(\vw_{t+1};\pi_{t,i}(j))}_\infty +\\ \norm{\sum_{i=1}^\workers\sum_{j=1}^\workerexamples \nabla f^i(\vw_{t+1};\pi_{t,i}(j))}_\infty.
\end{align*}

By using the above, we can rewrite the right-hand side to be

\begin{align*}
&\le \sum_{i=1}^\workers\sum_{j=1}^\workerexamples L_{2,\infty}\norm{\vw_{t}^{j} - \vw_{t+1}}_\infty + \workers \workerexamples \norm{\nabla f(\vw_{t+1})}_\infty\\
&\le 2\workers \workerexamples L_{2,\infty}\Delta_{t} + \workers \workerexamples \norm{\nabla f(\vw_{t+1})}_2.
\end{align*}

Therefore, by Lemma~\ref{lem:pair-balance}, 

\begin{align*}
\max_{k \in [\workerexamples]}\norm{\sum_{j=1}^{k}\sum_{i=1}^\workers \nabla f^i(\vw_{t}^{\pi_{t,i}^{-1}\pi_{t+1,i}(j)}, \pi_{t+1,i}(j))}_\infty \le 
&\max_{k\in[\workerexamples]}\norm{\sum_{j=1}^{k}\sum_{i=1}^\workers \nabla f^i(\vw_{t}^{j}, \pi_{t,i}(j))}_\infty\\
+&2\workers \workerexamples L_{2,\infty}\Delta_{t} + \norm{\nabla f(\vw_{t+1})}_2 + (\varsigma + \sigma + 2L_{2,\infty} \Delta_{t}) \tilde{A}.
\end{align*}

The second term of the triangle inequality (\ref{lem:grad-balance:eq:triangle}) can be bounded as 

\begin{align*}
\norm{\sum_{j=1}^{k}\sum_{i=1}^\workers \left(\nabla f^i(\vw_{t+1,i}^j; \pi_{t+1,i}(j)) - \nabla f^i(\vw_{t,i}^{\pi_{t,i}^{-1}\pi_{t+1,i}(j)}; \pi_{t+1,i}(j))\right)}_\infty,
\end{align*}

which is

\begin{align*}
&\le \sum_{j=1}^k\sum_{i=1}^\workers \norm{\vw_{t+1,i}^j - \vw_{t,i}^{\pi_{t,i}^{-1}\pi_{t+1,i}(j)}}_\infty\\
&\le \workers \workerexamples L_{2,\infty} (\Delta_{t+1} + 2\Delta_{t}).
\end{align*}

Substituting these bounds into the right-hand side of the triangle inequality (\ref{lem:grad-balance:eq:triangle}), taking the max of both sides, and grouping terms, we get

\begin{align*}
\max_{k\in[\workerexamples]}\norm{\sum_{j=1}^{k}\sum_{i=1}^\workers \nabla f^i(\vw_{t+1,i}^j, \pi_{t+1,i}(j))}_\infty \le& \frac{1}{2}\max_{k\in[\workerexamples]}\norm{\sum_{j=1}^{k}\sum_{i=1}^\workers \nabla f^i(\vw_{t}^{j}, \pi_{t,i}(j))}_\infty\\
&+ L_{2,\infty}(4\workers \workerexamples + 2\tilde{A})\Delta_{t} + \workers \workerexamples L_{2,\infty} \Delta_{t+1} + (\varsigma + \sigma)\tilde{A}\\
&+ \workers \workerexamples \norm{\nabla f(\vw_{t+1})}_2.\\
\end{align*}

Multiplying both sides by $\frac{\alpha}{\workers}$ and using the definition of $\Delta_t$ (\ref{eq:deltat}), we get the claim.
\end{proof}
\subsection{Combining the prior intermediate results: proofs over multiple steps}\label{app:proof:last}

\begin{lemma} 
If the learning rate $\alpha \le \frac{1}{16 L_{2,\infty} (2\workerexamples + \tilde{A}/\workers)}$, then
\begin{align*}
\frac{1}{T}\sum_{t=1}^{T} \Delta_t^2 \le \frac{21\alpha^2(\varsigma + \sigma)^2\tilde{A}^2}{\workers^2} + \frac{9\alpha^2 \workerexamples^2 \sigma^2}{T} + 21\alpha^2 \workerexamples^2 \frac{1}{T}\sum_{t=1}^T\norm{\nabla f(\vw_t)}_2^2.
\end{align*}
\label{lem:Delta-1}
\end{lemma}

\begin{proof}
First, we bound $\Delta_1^2$. 

We start with a series of triangle inequalities:
\begin{align*}
\frac{\alpha}{\workers}\norm{\sum_{j=1}^{k}\sum_{i=1}^\workers \nabla f^i(\vw_1^j, \pi_{1,i}(j))}_\infty \le & \frac{\alpha}{\workers}\norm{\sum_{j=1}^k \sum_{i=1}^\workers \nabla f^i(\vw_1^j, \pi_{1,i}(j)) - \sum_{j=1}^k \sum_{i=1}^\workers\nabla f^i(\vw_1, \pi_{1,i}(j))}_\infty\\
&+\frac{\alpha}{\workers}\norm{\sum_{j=1}^k\sum_{i=1}^\workers \left(\nabla f^i(\vw_1, \pi_{1,i}(j)) - \nabla f^i(\vw_1)\right)}_\infty + \alpha k \norm{\nabla f(\vw_1)}_\infty\\
\le& \frac{\alpha}{\workers}\sum_{j=1}^k\sum_{i=1}^\workers L_{2,\infty} \norm{\vw_1^j - \vw_1}_\infty + \alpha k \sigma + \alpha k \norm{\nabla f(\vw_1)}_2.
\end{align*}

We next take the max of both sides with respect to $k \in [\workerexamples]$:  

\begin{align*}
\Delta_1 &\le \alpha \workerexamples L_{2,\infty} \Delta_1 + \alpha \workerexamples \sigma + \alpha \workerexamples \norm{\nabla f(\vw_1)}_2\\ 
&\le (1/32) \Delta_1 + \alpha \workerexamples \sigma + \alpha \workerexamples \norm{\nabla f(\vw_1)}_2 \hspace{.5in} \text{(since} \hspace{.2cm} \alpha \le \frac{1}{32 \workerexamples L_{2,\infty}} \text{)} \\
&\le (32/31)\alpha \workerexamples \sigma + (32/31)\alpha \workerexamples \norm{\nabla f(\vw_1)}_2,
\end{align*}

Squaring both sides:

\begin{align}
\label{lem:Delta:eq:Delta_1}
\Delta_1^2 &\le 3\alpha^2 \workerexamples^2 \sigma^2 + 3\alpha^2 \workerexamples^2 \norm{\nabla f(\vw_1)}_2^2.
\end{align}

Now, we use Lemma~\ref{lem:grad-balance} to get the relationship between $\Delta_{t+1}$ and $\Delta_t$ for $t \in [T]$. 

Recall that 
\begin{align*}
\Delta_{t+1} &\le \frac{1}{2}\Delta_t + \alpha L_{2,\infty} \left(4 \workerexamples + \frac{2\tilde{A}}{\workers}\right)\Delta_{t} + \alpha \workerexamples L_{2,\infty} \Delta_{t+1} + \frac{\alpha (\varsigma + \sigma)\tilde{A}}{\workers} + \alpha \workerexamples \norm{\nabla f(\vw_{t+1})}_2\\
\end{align*}

Because $\alpha \le \frac{1}{16 L_{2,\infty} (2\workerexamples + \tilde{A}/\workers)}$, we can rewrite the above as

\begin{align*}
\Delta_{t+1} &\le \frac{1}{2}\Delta_t + (1/8)\Delta_{t} + (1/32) \Delta_{t+1} + \frac{\alpha (\varsigma + \sigma)\tilde{A}}{\workers} + \alpha \workerexamples \norm{\nabla f(\vw_{t+1})}_2. 
\end{align*}

Squaring both sides:

\begin{align*}
(31/32)^2\Delta_{t+1}^2 &\le \frac{1}{2}\Delta_{t}^2 + 2\left((1/8)\Delta_{t} + \frac{\alpha(\varsigma + \sigma)\tilde{A}}{\workers} +  \alpha \workerexamples \norm{\nabla f(\vw_{t+1})}_2\right)^2\\
&\le \frac{1}{2}\Delta_{t}^2 + (6/8^2)\Delta_{t}^2 + \frac{6\alpha^2(\varsigma + \sigma)^2\tilde{A}^2}{\workers^2} +  6\alpha^2 \workerexamples^2 \norm{\nabla f(\vw_{t+1})}_2^2,
\end{align*}
so that
\begin{align}
\begin{split}
\label{lem:Delta:eq:critical}
\Delta_{t+1}^2 &\le (32/31)^2(1/2 + 6/8^2)\Delta_t^2 + \frac{(32/31)^2 6\alpha^2(\varsigma + \sigma)^2\tilde{A}^2}{\workers^2} +  (32/31)^2 6\alpha^2 \workerexamples^2 \norm{\nabla f(\vw_{t+1})}_2^2\\
&\le (2/3)\Delta_t^2 + \frac{7\alpha^2(\varsigma + \sigma)^2\tilde{A}^2}{\workers^2} +  7\alpha^2 \workerexamples^2 \norm{\nabla f(\vw_{t+1})}_2^2.
\end{split}
\end{align}

We next sum (\ref{lem:Delta:eq:critical}) over $t \in [T-1]$ and add (\ref{lem:Delta:eq:Delta_1}):
\begin{align*}
\Delta_1^2 + \sum_{t=2}^T \Delta_t^2 &\le (2/3)\sum_{t=2}^T \Delta_{t-1}^2 + \frac{(T-1)7\alpha^2(\varsigma + \sigma)^2\tilde{A}^2}{\workers^2} + 3\alpha^2 \workerexamples^2 \sigma^2 + 7\alpha^2 \workerexamples^2 \sum_{t=1}^T\norm{\nabla f(\vw_t)}_2^2\\
\frac{1}{T}\sum_{t=1}^T \Delta_t^2 &\le (2/3)\frac{1}{T}\sum_{t=1}^T \Delta_t^2 + \frac{7\alpha^2(\varsigma + \sigma)^2\tilde{A}^2}{\workers^2} + \frac{3\alpha^2 \workerexamples^2 \sigma^2}{T} + 7\alpha^2 \workerexamples^2 \frac{1}{T}\sum_{t=1}^T\norm{\nabla f(\vw_t)}_2^2\\
&\le \frac{21\alpha^2(\varsigma + \sigma)^2\tilde{A}^2}{\workers^2} + \frac{9\alpha^2 \workerexamples^2 \sigma^2}{T} + 21\alpha^2 \workerexamples^2 \frac{1}{T}\sum_{t=1}^T\norm{\nabla f(\vw_t)}_2^2,
\end{align*}
yielding the claim. 
\end{proof}

We next build on Lemma~\ref{lem:Delta-1}. 

\begin{lemma} 
If $\alpha \le \frac{2}{9 \workerexamples \mu}$, then, for $\rho = 1 - \frac{\alpha \workerexamples \mu}{2}$, 
\begin{align*}
\sum_{t=1}^T \rho^{T-t} \Delta_t^2
&\le 12\rho^{T-1}\alpha^2 \workerexamples^2 \sigma^2 + \frac{28\rho\alpha^2 (\varsigma + \sigma)^2\tilde{A}^2}{(1-\rho)\workers^2} + \frac{1}{2L_{2,\infty}^2}\sum_{t=1}^T\rho^{T-t}\norm{\nabla f(\vw_t)}_2^2.
\end{align*}
\label{lem:Delta-2}
\end{lemma}

\begin{proof}
Recall (\ref{lem:Delta:eq:Delta_1}) from Lemma~\ref{lem:Delta-1}:

\begin{align*}
\Delta_1^2 &\le 3\alpha^2 \workerexamples^2 \sigma^2 + 3\alpha^2 \workerexamples^2 \norm{\nabla f(\vw_1)}_2^2.
\end{align*}

We multiply each term $\Delta_t$ with $\rho^{T-t}$ for $t \in [T]$ and get 

\begin{align}
\label{eq:delta1'}
\rho^{T-1}\Delta_1^2 &\le \rho^{T-1}3\alpha^2 \workerexamples^2 \sigma^2 + \rho^{T-1}3\alpha^2 \workerexamples^2 \norm{\nabla f(\vw_1)}_2^2. 
\end{align}

Similarly, recall (\ref{lem:Delta:eq:critical}) from Lemma~\ref{lem:Delta-1}, 

\begin{align*}
\Delta_{t+1}^2 &\le (2/3)\Delta_t^2 + \frac{7\alpha^2(\varsigma + \sigma)^2\tilde{A}^2}{\workers^2} +  7\alpha^2 \workerexamples^2 \norm{\nabla f(\vw_{t+1})}_2^2,
\end{align*}

for which we also multiply  each term $\Delta_t$ with $\rho^{T-t}$ for $t \in [T]$, and get

\begin{align}
\begin{split}
\rho^{T-t}\Delta_{t}^2 &\le (2/3)\rho^{T-t}\Delta_{t-1}^2 + \rho^{T-t}\frac{7\alpha^2 (\varsigma + \sigma)^2\tilde{A}^2}{\workers^2} + \rho^{T-t}7\alpha^2 \workerexamples^2 \norm{\nabla f(\vw_t)}_2^2\\
&\le (3/4)\rho^{T-(t-1)}\Delta_{t-1}^2  + \rho^{T-t}\frac{7\alpha^2 (\varsigma + \sigma)^2\tilde{A}^2}{\workers^2} + \rho^{T-t}7\alpha^2 \workerexamples^2 \norm{\nabla f(\vw_t)}_2^2,\quad \forall t \in \{2,\dots,T\}, 
\end{split}
\label{eq:deltacritical'}
\end{align}

where we have used $\alpha \le \frac{2}{9\workerexamples\mu}$ so that $\rho = 1 - \frac{\alpha \workerexamples \mu}{2} \ge (2/3)(4/3)$. 

Next, we sum the bounds in (\ref{eq:delta1'}) and (\ref{eq:deltacritical'}) for $\rho^{T-t} \Delta_t$ for all $t \in [T]$, and we get

\begin{align*}
\rho^{T-1}\Delta_1^2  + \sum_{t=2}^T \rho^{T-t} \Delta_t^2 &\le \frac{3}{4}\sum_{t=2}^T\rho^{T-(t-1)}\Delta_{t-1}^2 +  \rho^{T-1}3\alpha^2 \workerexamples^2 \sigma^2 + \sum_{t=1}^T\rho^{T-t}\frac{7\alpha^2 (\varsigma + \sigma)^2\tilde{A}^2}{\workers^2} + \\
& \quad 7\alpha^2 \workerexamples^2 \sum_{t=1}^T\rho^{T-t}\norm{\nabla f(\vw_t)}_2^2.
\end{align*}

We can rewrite the right-hand side as
\begin{align*}
&\le \frac{3}{4}\sum_{t=1}^T\rho^{T-t}\Delta_{t}^2 +  \rho^{T-1}3\alpha^2 \workerexamples^2 \sigma^2 + \frac{7\rho\alpha^2 (\varsigma + \sigma)^2\tilde{A}^2}{(1-\rho)\workers^2} + 7\alpha^2 \workerexamples^2 \sum_{t=1}^T\rho^{T-t}\norm{\nabla f(\vw_t)}_2^2\\
&\le 12\rho^{T-1}\alpha^2 \workerexamples^2 \sigma^2 + \frac{28\rho\alpha^2 (\varsigma + \sigma)^2\tilde{A}^2}{(1-\rho)\workers^2} + 28\alpha^2 \workerexamples^2 \sum_{t=1}^T\rho^{T-t}\norm{\nabla f(\vw_t)}_2^2.\\
\end{align*}

Lastly, we use $\alpha \le \frac{1}{\sqrt{56} \workerexamples L_{2,\infty}}$ to get:
\begin{align*}
\sum_{t=1}^T \rho^{T-t} \Delta_t^2
&\le 12\rho^{T-1}\alpha^2 \workerexamples^2 \sigma^2 + \frac{28\rho\alpha^2 (\varsigma + \sigma)^2\tilde{A}^2}{(1-\rho)\workers^2} + \frac{1}{2L_{2,\infty}^2}\sum_{t=1}^T\rho^{T-t}\norm{\nabla f(\vw_t)}^2.\\
\end{align*}
\end{proof}

\subsection{Proof of Theorems~\ref{thm:dgrab:smooth} and~\ref{thm:dgrab:PL}}\label{app:thm:proof}

Using the Lemmas above, we next prove our main results, presented in Section~\ref{sec:theory}.

\begin{proof}[Proof of Theorem \ref{thm:dgrab:smooth}]
The given learning rate $\alpha$ satisfies the constraints of Lemma~\ref{lem:jens} and Lemma~\ref{lem:Delta-1}. 

Therefore, 
\begin{align*}
\frac{1}{T}\sum_{t=1}^{T} \norm{\nabla f(\vw_t)}_2^2 &\le \frac{2F_1}{\alpha \workerexamples T} + L_{2,\infty}^2\left(\frac{21(\alpha(\varsigma + \sigma)\tilde{A})^2}{\workers^2} + \frac{9(\alpha \workerexamples \sigma)^2}{T} + 21\alpha^2 \workerexamples^2 \frac{1}{T}\sum_{t=1}^T\norm{\nabla f(\vw_t)}_2^2\right)\\
&\le \frac{4F_1}{\alpha \workerexamples T} + \frac{42L_{2,\infty}^2(\alpha(\varsigma + \sigma)\tilde{A})^2}{\workers^2} + \frac{18L_{2,\infty}^2(\alpha \workerexamples \sigma)^2}{T},\\
\end{align*}
due to $\alpha \le \frac{1}{\sqrt{42} \workerexamples L_{2,\infty}}$. 

We next derive the convergence rate. Let $\Gamma = \frac{42(L_{2,\infty}(\varsigma + \sigma)\tilde{A})^2}{\workers^2} + \frac{18L_{2,\infty}^2 \workerexamples^2 \sigma^2}{T}$. Then,
\begin{align*}
\frac{1}{T}\sum_{t=1}^{T} \norm{\nabla f(\vw_t)}_2^2 \le \frac{4F_1}{\alpha \workerexamples T} + \Gamma \alpha^2.
\end{align*}
We then set $\alpha \le \left(\frac{4F_1}{\workerexamples \Gamma T}\right)^{1/3}$. So we will have $\alpha = \min\left\{\frac{1}{16L_{2,\infty} (2\workerexamples + \tilde{A}/\workers)}, \left(\frac{4F_1}{\workerexamples\Gamma T}\right)^{1/3}\right\}$ or
\begin{align*}
\frac{1}{\alpha} = \max\left\{16 L_{2,\infty} (2\workerexamples + \tilde{A}/\workers), \left(\frac{4F_1}{\workerexamples\Gamma T}\right)^{-1/3}\right\}.
\end{align*}
Substitute $\alpha$:
\begin{align*}
\frac{1}{T}\sum_{t=1}^{T} \norm{\nabla f(\vw_t)}_2^2 &\le \frac{4F_1}{\workerexamples T} \left\{16 L_{2,\infty} (2\workerexamples + \tilde{A}/\workers) + \left(\frac{4F_1}{\workerexamples\Gamma T}\right)^{-1/3}\right\} + \Gamma \left(\frac{4F_1}{\workerexamples\Gamma T}\right)^{2/3}\\
&\le \left(\frac{4F_1}{\workerexamples T}\right)^{2/3}\Gamma^{1/3} + \frac{64F_1 L_{2,\infty} (2 + \tilde{A}/(\workers \workerexamples))}{T}\\
&\le \left(\frac{4F_1}{\workerexamples T}\right)^{2/3}\left(\frac{(\sqrt{42}L_{2,\infty}(\varsigma + \sigma)\tilde{A})^{2/3}}{\workers^{2/3}} + \frac{(\sqrt{18}L_{2,\infty} \workerexamples \sigma)^{2/3}}{T^{1/3}}\right)\\
&\quad + \frac{64 F_1 L_{2,\infty} (2 + \tilde{A}/(\workers \workerexamples))}{T}\\
&\le \frac{(4\sqrt{42} F_1 L_{2,\infty}(\varsigma + \sigma)\tilde{A})^{2/3}}{(\workers \workerexamples T)^{2/3}} + \frac{(72 F_1 L_{2,\infty}\sigma)^{2/3}}{T}\\
&\quad + \frac{64F_1 L_{2,\infty} (2 + \tilde{A}/(\workers \workerexamples))}{T},\\
\end{align*}
Since $(4\sqrt{42})^{2/3} < 9$, the above is

\begin{align*}
&\le \frac{9(F_1 L_{2,\infty}(\varsigma + \sigma)\tilde{A})^{2/3}}{(\workers \workerexamples T)^{2/3}} + \frac{(72 F_1 L_{2,\infty}\sigma)^{2/3} + 64F_1 L_{2,\infty} (2 + \tilde{A}/(\workers \workerexamples))}{T}, 
\end{align*}

in which the leading term (slowest in terms of $T$) is $\tilde{O}((\workers \workerexamples T)^{-2/3})$, proving the claim.
\end{proof}

\begin{proof}[Proof of Theorem \ref{thm:dgrab:PL}]
With the P.L. assumption (Assumption~\ref{ass:PL}), we use Lemma~\ref{lem:jens-pl} and Lemma~\ref{lem:Delta-2}. We show that their constraints are satisfied later) to get 
\begin{align*}
F_{T+1} &\le \rho^T F_1 + \frac{\alpha \workerexamples L_{2,\infty}^2}{2}\sum_{t=1}^T\rho^{T-t}\left(\Delta_t^2 - \frac{1}{2L_{2,\infty}^2} \norm{\nabla f(\vw_t)}_2^2\right)\\
&\le \rho^T F_1 + \frac{\alpha \workerexamples L_{2,\infty}^2}{2}\left(12\rho^{T-1}\alpha^2 \workerexamples^2 \sigma^2 + \frac{28\rho\alpha^2 (\varsigma + \sigma)^2\tilde{A}^2}{(1-\rho)\workers^2}\right)\\
&\le \rho^{T} F_1 + \rho^{T-1}6\alpha^3 \workerexamples^3 L_{2,\infty}^2\sigma^2 + \frac{28\rho\alpha^3 \workerexamples L_{2\infty}^2 (\varsigma + \sigma)^2\tilde{A}^2}{\alpha \workerexamples \mu \workers^2}\\
&\le \rho^{T} F_1 + \rho^{T} 7 \alpha^3 \workerexamples^3 L_{2,\infty}^2\sigma^2 + \frac{28\rho\alpha^3 \workerexamples L_{2\infty}^2 (\varsigma + \sigma)^2\tilde{A}^2}{\alpha \workerexamples \mu \workers^2}\\
&\le \rho^{T} (F_1 + \sigma^2/L_{2,\infty}) + \frac{28\alpha^2 L_{2,\infty}^2 (\varsigma + \sigma)^2\tilde{A}^2}{\mu \workers^2}\\
&\le (F_1 + \sigma^2/L_{2,\infty})\exp(- T\alpha \workerexamples \mu/2) + \frac{28\alpha^2 L_{2,\infty}^2 (\varsigma + \sigma)^2\tilde{A}^2}{\mu \workers^2},
\end{align*}
where we have further constrained $\alpha \le \frac{2}{9 \workerexamples \mu}$ so that $\rho \le 9/8$ in the forth inequality and  $\alpha \le \frac{1}{7^{1/3} \workerexamples L_{2,\infty}}$ in the fifth inequality. By setting the derivative w.r.t $\alpha$ of the RHS to 0, the minimizer $\alpha$ under the constraint that $0 < \alpha \le \min\left\{\frac{2}{9 \workerexamples \mu}, \frac{1}{16 L_{2,\infty}(2 \workerexamples + \tilde{A}/\workers)}\right\}$ (required by the lemmas) is:
\begin{align*}
\alpha = \frac{2}{T \workerexamples \mu} W_0(T^2 \workers^2 \workerexamples^2 C_3),
\end{align*}
as long as 
\begin{align*}
T &\ge 1 + \frac{2}{\workerexamples \mu}\max\{(9/2)\workerexamples \mu, 16 L_{2,\infty}(2\workerexamples+\tilde{A}/\workers)W_0(T^2\workers^2\workerexamples^2C_3)\}\\
&= 10 + \frac{1}{\mu}32 L_{2,\infty}(2+\tilde{A}/(\workers \workerexamples))W_0(T^2\workers^2\workerexamples^2C_3),
\end{align*}
where $C_3 = \frac{(F_1+\sigma^2/L_{2,\infty})\mu^2}{224L_{2,\infty}^2(\varsigma + \sigma)^2\tilde{A}^2}$. 

What we did here was to set $T$ just large enough so that the minimizer $\alpha$ is the same with or without the constraint. 

Denoting $\tilde{W} = W_0(T^2m^2n^2C_3) = \tilde{O}(1)$, we get 
\begin{align*}
F_{T+1} &\le  \frac{(F_1 + \sigma^2/L_{2,\infty})\tilde{W}}{T^2\workers^2\workerexamples^2C_3} + \frac{112L_{2,\infty}^2(\varsigma + \sigma)^2\tilde{A}^2\tilde{W}^2}{T^2\workers^2\workerexamples^2\mu^3}\\
&\le \frac{1}{T^2\workers^2\workerexamples^2}\left(\frac{(F_1 + \sigma^2/L_{2,\infty})\tilde{W}}{\tilde{C}_3} + \frac{112L_{2,\infty}^2(\varsigma + \sigma)^2\tilde{A}^2\tilde{W}^2}{\mu^3}\right),\\
\end{align*}
which shows rate the convergence rate in the P.L. case is $\tilde{O}((\workers \workerexamples T)^{-2})$.
\end{proof}
\definecolor{codegreen}{rgb}{0,0.6,0}
\definecolor{codegray}{rgb}{0.5,0.5,0.5}
\definecolor{codepurple}{rgb}{0.58,0,0.82}
\definecolor{backcolour}{rgb}{0.95,0.95,0.92}

\lstdefinestyle{mystyle}{
    backgroundcolor=\color{backcolour},   
    commentstyle=\color{codegreen},
    keywordstyle=\color{magenta},
    numberstyle=\tiny\color{codegray},
    stringstyle=\color{codepurple},
    basicstyle=\ttfamily\footnotesize,
    breakatwhitespace=false,         
    breaklines=true,                 
    captionpos=b,                    
    keepspaces=true,                 
    numbers=left,                    
    numbersep=5pt,                  
    showspaces=false,                
    showstringspaces=false,
    showtabs=false,                  
    tabsize=2
}

\lstset{style=mystyle}

\section{Experiment Details}

Here we provide more extensive details on our empirical results. This includes background information on our experimental setup in the main paper (Appendix~\ref{sec:appendix-model-dataset}), an additional simulation experiment on pre-training and fine-tuning Tiny GPT-2 (Appendix~\ref{appendix-gpt2}), and an additional simulation experiment that investigates \dgrab{} with different learning rates (Appendix~\ref{appendix-lr}). Our source codes can be found \href{https://github.com/GarlGuo/CD-GraB}{here}.

\subsection{Additional details on setup for main paper experiments}\label{sec:appendix-model-dataset}

\subsubsection{Distributed experiments}

We provide additional details on the experiments shown in Figure~\ref{fig:exp}.

\custompar{Hardware and software} We use a single machine with 128 GiB memory, 1 CPU, and 4 Nvidia GeForce 2080ti GPUs for the HMDA mortgage application, M4, and WikiText-2 tasks. We first discard the remainder $\examples \mod B$, and then randomly partition $\workerexamples$ to each worker. Our experiments are all implemented with the PyTorch library. We release our code suite at [REDACTED].

\custompar{Datasets and models}

\begin{itemize}[topsep=0pt, leftmargin=.5cm]

    \item \textbf{Logistic regression on mortgage application (NY 2017 subset)}: The US Home Mortgage Disclose Act (HMDA) makes available US national data regarding mortgage applications, which has recently been packaged up for easy ML research use~\citep{cooper2023variance}. We use the binary classification version of the task, which classifies features as either ``grant loan'' or ``deny loan,'' for the New York (NY) 2017 subset of the dataset, which includes 244107 examples with 18 features. We model this problem using logistic regression, for which we first perform a random 80/20 train/test split on the raw dataset, and then we discard $\examples \mod B$ ($B$ is the aggregated minibatch size) examples to ensure that each worker receives exactly $\workerexamples$ examples. We use 1 worker per GPU, and in total we have $\workers=4$ workers, and use NCCL~\citep{nccl} as the distributed communication backend; $\workers=4, \; \workerexamples=48816, \; d=18, \; B = 16$. We report test accuracy as our evaluation metric. 
    

    \item \textbf{LSTM on WikiText-2}: We follow the settings in \citet{lu2022grab} and train a 2-layer LSTM with an embedding size of 32 and dropout set to 0. We use backpropagation through time, for which we set the sequence length to 35. We also adopt the word-vector-classifier-weight-sharing strategy inspired by \citet{inan2016tying}.  WikiText-2 \citep{stephen2017pointer} has 600 articles in the train set, with more than 2M tokens and 30K vocabulary; the validation and test sets each have 60 articles. We adapt our training script from \href{https://github.com/pytorch/examples/tree/main/word_language_model}{PyTorch's official Word Language Modeling Github repository}. We use 4 workers in total, with each GPU hosting 1 worker, and use NCCL as the distributed communication backend; $\workers = 4, \; \workerexamples= 3728, \; d = 1081760, \; B = 16$.  We report test perplexity as the evaluation metric, and we follow the HuggingFace's approach of computing perplexity as the exponentiated average negative log-likelihood of a sequence\footnote{\url{https://huggingface.co/docs/transformers/perplexity}}.
    
    \item \textbf{Autoregressive MLP on M4 Weekly Dataset}: We build a 3-layer autoregressive MLP with a hidden dimension of 64. We set input sequence length to be 20 and the output sequence length to be 6. M4 is a time series dataset composed of 100,000 time series for yearly, quarterly, monthly, weekly, daily and hourly data~\citep{MAKRIDAKIS202054}, which is drawn from a random sample of ForeDeCk database \citep{spiliotis2020forecasting}. We use the weekly data in our experiment. We use 32 workers, where each of the 4 GPUs hosts 8 process workers. We use GLOO as the distributed communication backend. $\workers = 32, \; \workerexamples = 3355, \; d = 5569, \; B = 32$.  We report test symmetric mean absolute percentage error (SMAPE) as the evaluation metric. We follow the formula of SMAPE in \citep{MAKRIDAKIS202054} as follows: 
\begin{equation*}
    \text{SMAPE} \triangleq \dfrac{2}{h} \sum_{t = n + 1}^{n + h} \dfrac{|Y_t - \hat{Y}_t|}{|Y_t| + |\hat{Y}_t|} * 100\%
\end{equation*}
where $Y_t$ is the reference time series value at timestep $t$, $\hat{Y}_t$ is the forecast time series value at timestep $t$, and $h$ is the forecasting horizon and $n$ is the number of datapoints. 

\end{itemize}

\custompar{Hyperparameter optimization} For all tasks, we tune the learning rate $\alpha$ for \dshuffle{} first, and then use the selected learning rate for \dgrab{}. Therefore, an performance improvement here implies we would have in-place substitution benefits via switching from \dshuffle{} to \dgrab{} with identical learning rate and experiment setups. We use SGD with momentum as the optimizer for all tasks. The hyperparameters for each task are as follows:

\begin{itemize}[topsep=0pt, leftmargin=.5cm]
    \item \textbf{Logistic regression on mortgage application (NY 2017 subset)}: $\alpha = $ 5e-3 $\in \{$1e-2, 5e-3, 1e-3$\}$, momentum: 0.9, weight decay: 0, $B$: 16.

    \item \textbf{LSTM on WikiText-2}: $\alpha = 5 \in \{$5, 10$\}$ and decays by 0.1 per 10 epochs, momentum: 0.9, weight decay: 0, $B$: 16. 

    \item \textbf{Autoregressive MLP on Weekly M4 Dataset} $\alpha = $ 1e-3 $\in \{$1e-2, 1e-3, 1e-4$\}$, momentum: 0.9, weight decay: 0, $B$: 32.
\end{itemize}

\subsubsection{Memory Overhead of \dgrab{} in LSTM on WikiText-2 Task} \label{sec:appendix-memory}

We profile the CUDA memory usage for the LSTM on WikiText-2 Task with \dgrab{} and \dshuffle{} to understand the memory overhead of both data permutation algorithms. This memory analysis is both task and implementation dependent, but still serves to illustrate the overarching point that \dgrab's memory overhead is not so significant. The additional overhead comes from two sources for \dgrab: communication and example sorting (Figure~\ref{fig:cuda-memory-lstm}). 

In more detail: In our LSTM experiment, each local worker will share its gradients with all other workers at every optimization step. To reduce the communication burden of \dgrab, we make each local worker function as an order server.\footnote{An ideal location for a dedicated order server is on a network node that has large input bandwidth and memory buffer to host all gradients while not blocking the normal optimization stages. Since we do not have enough computational resources to host a dedicated order server, we make each worker an order server.} The memory consumption of forward, backward, and optimizer states between \dgrab{} and \dshuffle{} should be (at least approximately) identical. 
The model size of LSTM is roughly 4 MiB. We use 4 workers, and as each worker (functioning as an order server) needs to \textit{all-gather} gradients, the memory overhead for \textit{all-gather} communication is roughly tensor\_size $\times$ \# workers = 4 MiB $\times$ 4 = 16 MiB for \dgrab (we observe 16.51 MiB in practice, Communication in Figure~\ref{fig:cuda-memory-lstm}), while \dshuffle{} only needs to \textit{all-reduce} the gradients (yielding no memory overhead; the communication buffer for \textit{all-reduce} is reusing the same gradient tensor). The $\mathsf{PairBalance}$ algorithm (Algorithm~\ref{alg:pairbalance}) internally needs a model-sized accumulator as the running sum $\vr$, and both computing inner product between $\vr$ and $\vg_1 - \vg_2$ and updating $\vr$ with $\vr + s \vc$ takes virtually no space with a memory-efficient implementation. Therefore, the memory consumption for $\mathsf{PairBalance}$ is still roughly 4 MiB (Data Sorter in Figure~\ref{fig:cuda-memory-lstm}). 
.

\begin{figure}[!ht]
\vspace{-.2cm}
  \centering
    \includegraphics[width=\columnwidth]{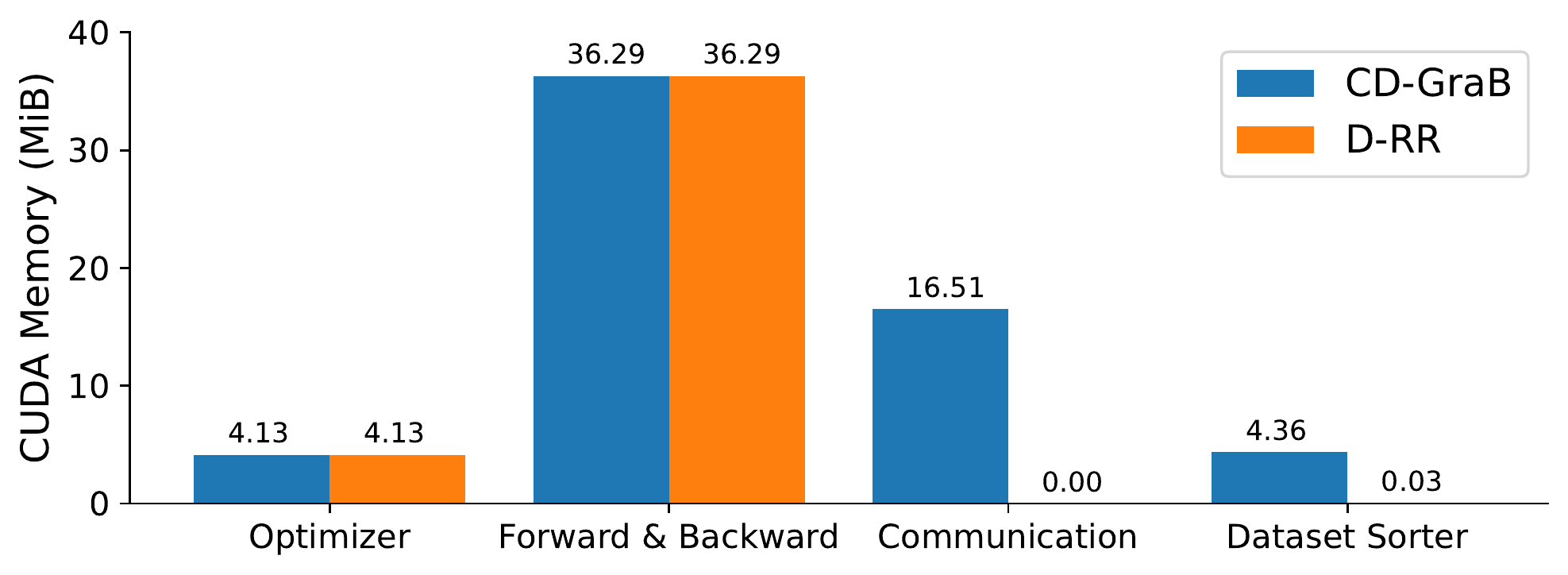}
    \caption{CUDA Memory Overhead of \dgrab{} and \dshuffle{} in LSTM on WikiText-2 Task. 
    \looseness=-1} 
  \label{fig:cuda-memory-lstm}
\end{figure}

The main memory overhead of \dgrab{} will be dominated by the communication buffer size on the order server side: the order server have to gather the gradient (differences) from all workers, and sequentially apply the $\mathsf{PairBalance}$ algorithm. This memory bottleneck would similarly be found in \grab{} as \grab{} also needs per-example gradients to perform $\mathsf{Balance}$ sequentially. 

A future algorithmic improvement to the general gradient balancing framework would be finding a balancing algorithm that does not need per-example gradients to achieve comparable convergence guarantees. However, we still notice that $\mathsf{PairBalance}$ is more memory-efficient than $\mathsf{Balance}$ as $\mathsf{Balance}$ needs to store 3 model-sized tensors: 1 for the balancing accumulator, 1 for running-average gradients for last epoch, and 1 for the running-average for current epoch. In constrast, $\mathsf{PairBalance}$ only needs 1 model-sized tensor as the balancing accumulator. 

\subsubsection{Simulated ablation study using LeNet on CIFAR-10 \label{appendix:workers}}

In the experiment shown on Figure~\ref{fig:nodes}, we select the same learning rate, momentum, and weight decay as the LeNet experiment in \citet{lu2022grab}. We use 3 different random seeds to control 3 different initialization and the randomness in random reshuffling. The aggregated minibatch size $B$ is 64 for all runs. We implement this ablation study by using 1 GPU with up to $m=64$ workers (processes). As above, we discard $\examples \mod B$ examples and partition the remaining examples evenly on each worker. 

$\alpha = $ 1e-3 $\in \{$1e-2, 5e-3, 1e-3, 5e-4, 1e-4$\}$, momentum: 0.9, weight decay: 1e-2, $B$: 64. 

We do not implement this via distributed environment due to the fact that we do not have access to 64 GPUs, but expect the simulation results to be a good reflection of the results we would obtain in a multi-GPU setting. 

\custompar{Parallel herding bound} We further investigate the empirical parallel herding bounds (\ref{equ:paraherding:objective}) for the LeNet experiment for the different ordering methods. We plot the results in Figure~\ref{fig:parallel-herding-bound-lenet}. We observe that as the number of workers increases, the empirical parallel herding bounds of both \textbf{ID-\grab{} (Bal)} and \textbf{ID-\grab{} (PairBal)} also increase, and eventually exhibit little difference with \dshuffle{}. \dgrab, in contrast, exhibits a consistently lower bound.

\begin{figure}[!ht]
\vspace{-.2cm}
  \centering
    \includegraphics[width=\columnwidth]{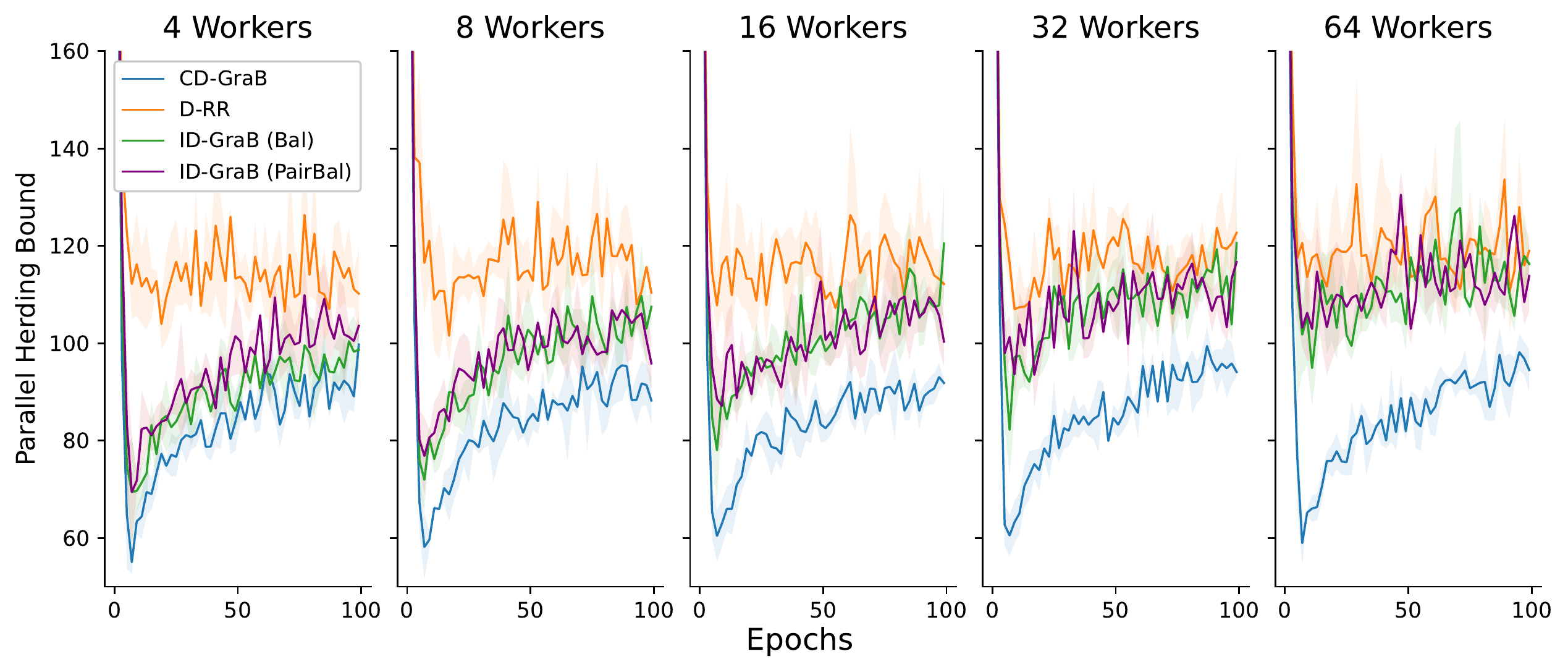}
    \vspace{-.5cm}
    \caption{Empirical parallel herding bounds of gradients for each algorithm in LeNet experiment. We plot the mean as the curve and standard deviation across 3 random seeds. 
    \looseness=-1} 
  \label{fig:parallel-herding-bound-lenet}
\end{figure}

For comparison, we also run a simulation experiment on synthetic data to investigate the behavior of the parallel herding bound. We include these below, in Figure~\ref{fig:parallel-herding-bound-sim}. 

We randomly initialize 1 million random vectors $\exij$ from a uniform distribution between 0 and 1 with 16 dimensions as $\exij \sim \mathsf{Unif}(0, 1)^{16}$, and then we zero-center this set of 1 million vectors and normalize them to all have $L_2$ norm as 1. We then evenly partition this set of 1 million random vectors to $\{5, 10, 20, 50, 100 \}$ workers and run each example ordering algorithm. 

\begin{figure}[!t]
\vspace{-.2cm}
  \centering
    \includegraphics[width=\columnwidth]{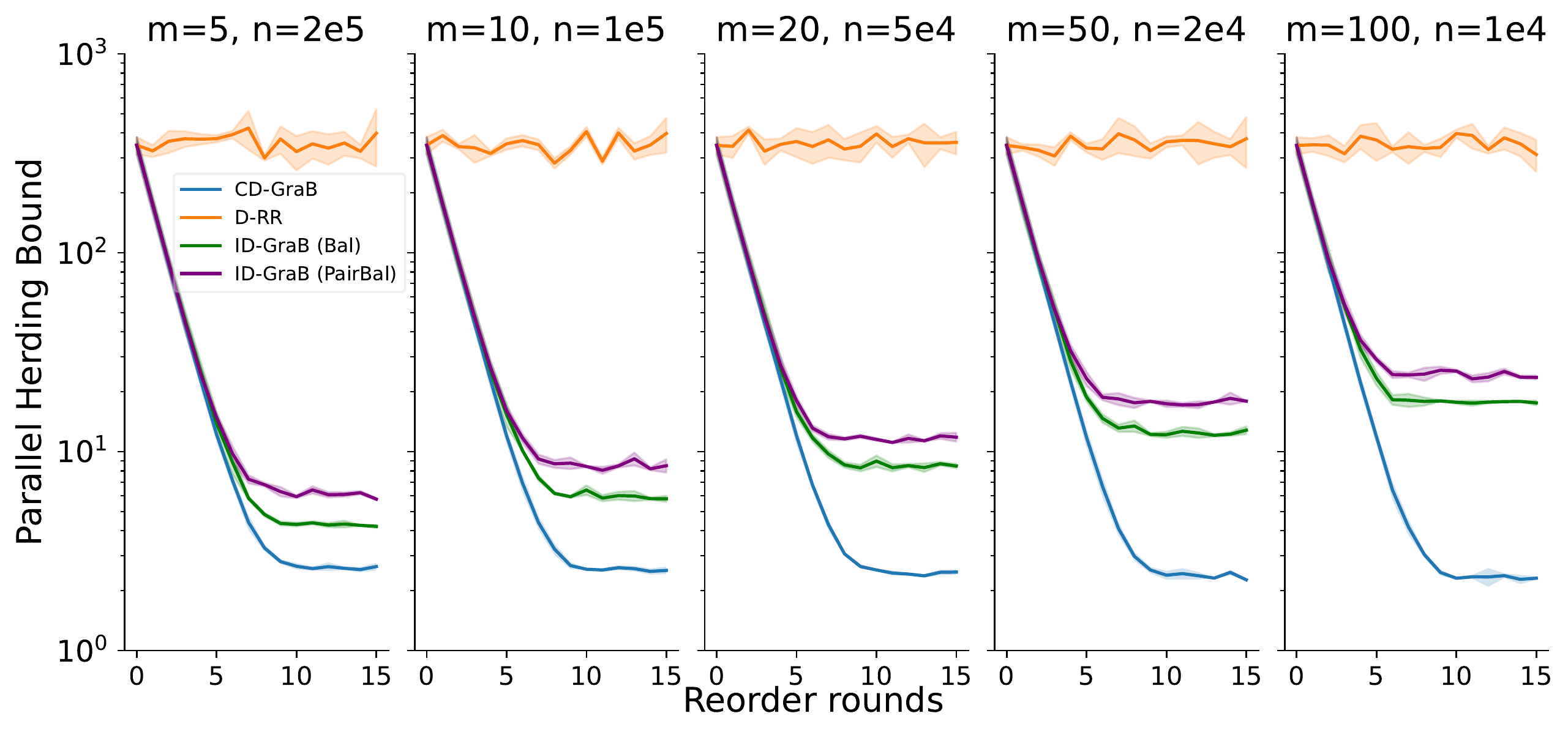}
    \caption{Parallel herding bounds for different example ordering algorithms on $\examples$=1 million random vectors. We use 3 random seeds, plot the mean and standard deviation across each random seed as the shaded area.
    \looseness=-1} 
  \label{fig:parallel-herding-bound-sim}
  \vspace{-.2cm}
\end{figure}

In Figure~\ref{fig:parallel-herding-bound-sim}, we run \dgrab, \dshuffle, \textbf{ID-\grab{} (Bal)}, \textbf{ID-\grab{} (PairBal)} on these random vectors, and compute the parallel herding bounds (\ref{equ:paraherding:objective}). From left to right in Figure~\ref{fig:parallel-herding-bound-sim}, we observe that as the number of workers $\workers$ increases, the parallel herding bound of \textbf{ID-\grab{} (Bal)}, \textbf{ID-\grab{} (PairBal)} becomes larger. This shows the importance of coordination when we have a large number of workers.

These results for random vectors cohere with our above results for LeNet on CIFAR-10.






\subsection{An additional simulation experiment: pre-training and fine-tuning Tiny GPT-2}\label{appendix-gpt2}

We perform an end-to-end simulation experiment involving pre-training and fine-tuning Tiny GPT-2 on WikiText-103, which we document below. 

\subsubsection{Pre-training}

We adapt the training script from the \href{https://github.com/huggingface/transformers/blob/main/examples/pytorch/language-modeling/run_clm_no_trainer.py}{HuggingFace's PyTorch casual language modeling code} to train the GPT-2 architecture~\citep{radford2019language}. We set the maximum sequence length to 128 and token and positional embedding dimension to 128; use 2 hidden layers in the transformer encoder and 2 attention heads; and disable dropout. This model configuration corresponds to the following Python code snippet:

\begin{lstlisting}[language=Python]
from transformers import GPT2Config, GPT2LMHeadModel, GPT2Tokenizer

tokenizer = GPT2Tokenizer.from_pretrained('gpt2')
config = GPT2Config.from_pretrained('gpt2')
config.n_embd = 128
config.n_ctx = 128
config.n_layer = 2
config.n_head = 2
config.n_positions = 128
config.summary_first_dropout = 0
config.attn_pdrop = 0
config.resid_pdrop = 0
model = GPT2LMHeadModel(config)
\end{lstlisting}

We train our Tiny GPT-2 model from scratch on WikiText-103~\citep{stephen2017pointer}. WikiText-103 is a standard language modeling benchmark that has 28,475 articles in the train set, and 60 for both the validation and test sets, with more than 100M tokens and 267K vocabulary inside the train set. We use the original GPT-2 tokenizer, and use maximum sequence length 128. We note that this is much smaller than the default maximum sequence length for GPT-2, which is 1024, which was too large to use given our computational budget. Nevertheless, 128 is still a reasonable sequence length for the initial phrase of pre-training; BERT uses a sequence length of 128 for the first 90\% of pre-training steps to speedup the experiment~\cite{devlin2019bert}. We tune the learning rate for \dshuffle{} with the grid $\{$ 5e-3, 1e-3, 5e-4, 1e-4 $\}$ (the final learning rate is 5e-4), and use AdamW optimizer~\cite{loshchilov2017decoupled}. We use 3 random seeds. Before the training, we simulate 64 workers, and similarly divide the training dataset evenly across them by discarding $\examples \mod B$ examples. Our hyperparameter optimization space is listed below:

\custompar{Pretraining Hyperparameters} $\alpha = $5e-4 $\in \{$5e-3, 1e-3, 5e-4, 1e-4$\}$, weight decay: 1e-4, $B$: 64.

\begin{figure}[!ht]
\vspace{-.2cm}
  \centering
    \includegraphics[width=\columnwidth]{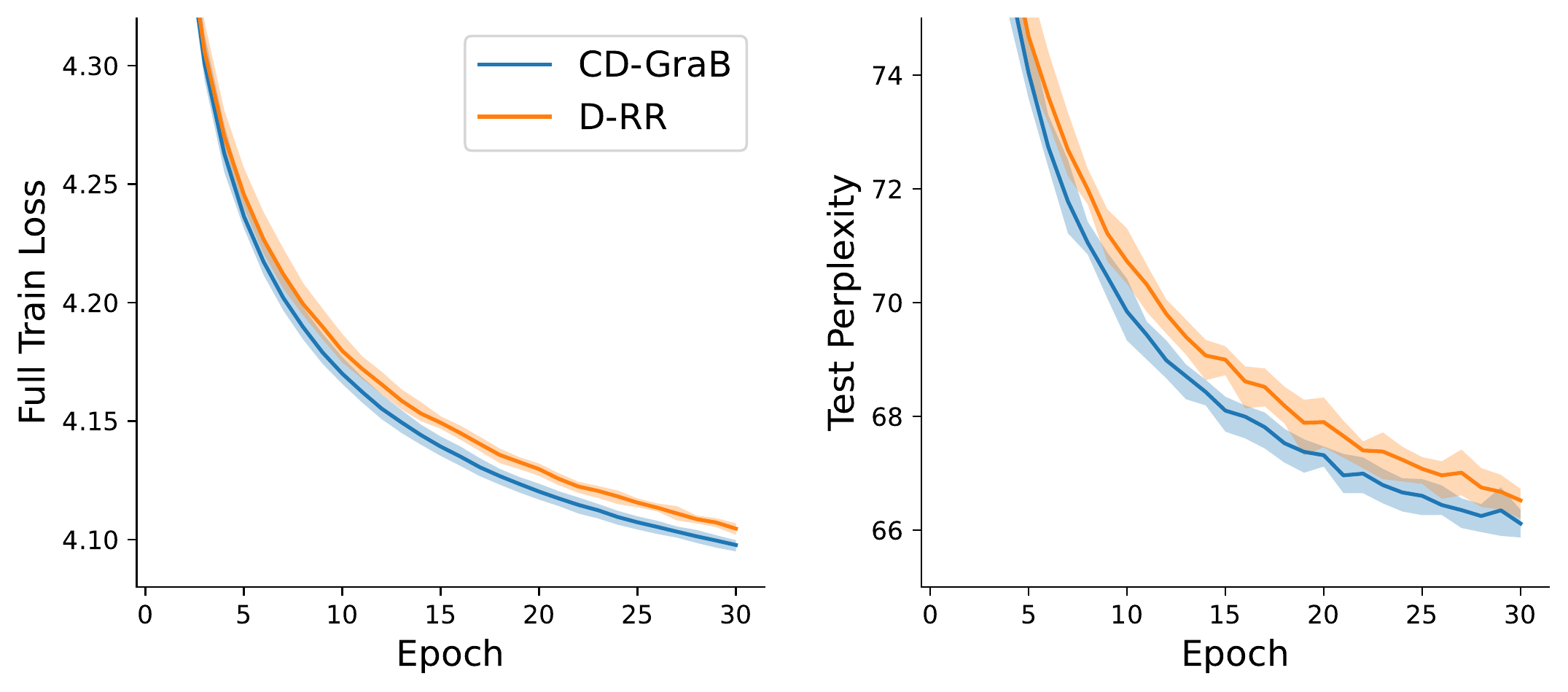}
    \vspace{-.5cm}
    \caption{Pre-training Tiny GPT-2 on WikiText-103 from scratch: Convergence for \dgrab{} and \dshuffle{} with $\workers = 64$ workers. The aggregated minibatch size per update is 64. We use 3 random seeds, and plot the mean vand standard deviation. 
    \looseness=-1} 
  \label{fig:gpt2}
\end{figure}

We document convergence for pre-training in Figure~\ref{fig:gpt2}, and use test perplexity as our evaluation metric.

\subsubsection{Fine-tuning}

We then fine-tune the pre-trained Tiny GPT-2 model on downstream tasks. For each task, we load the pre-trained foundation model weights obtained at the end of 30 epochs of each example ordering algorithm after pretraining, and use the same example ordering algorithm to perform supervised fine-tuning. We focus on the largest 4 GLUE tasks~\cite{wang-etal-2018-glue}: MNLI, QQP, QNLI, and SST2. We tune the learning rate for \dshuffle{} with the AdamW optimizer, and for each run we report the best validation accuracy. We then take an average results of each run and summarize them in Table~\ref{tab:gpt2-glue}. Our training script is adapted from the \href{https://github.com/huggingface/transformers/blob/main/examples/pytorch/text-classification/run_glue_no_trainer.py}{HuggingFace's PyTorch GLUE fine-tuning example codes}.

\textbf{Fine-Tuning Hyperparameters}
\begin{itemize}[leftmargin=.5cm]
    \item \textbf{MNLI} $\alpha = $5e-4$ \in \{$5e-3, 1e-3, 5e-4, 1e-4$\}$, Weight decay: 1e-4, $B$: 32, epochs: 10, linear learning rate scheduler

     \item \textbf{QQP} $\alpha = $5e-4$ \in \{$5e-3, 1e-3, 5e-4, 1e-4$\}$, Weight decay: 1e-4, $B$: 32, epochs: 10, linear learning rate scheduler

     \item \textbf{QNLI} $\alpha = $5e-4$ \in \{$5e-3, 1e-3, 5e-4, 1e-4$\}$, Weight decay: 1e-4, $B$: 32, epochs: 10, linear learning rate scheduler

     \item \textbf{SST2} $\alpha = $5e-4$ \in \{$5e-3, 1e-3, 5e-4, 1e-4$\}$, Weight decay: 1e-4, $B$: 32, epochs: 10, linear learning rate scheduler
\end{itemize}

\begin{table*}[!ht] 
 \centering 
 \setlength{\tabcolsep}{2pt}
 \small

 \begin{tabular}{c|ccccc}
    \toprule
    
    & \textbf{MNLI (Matched)} & \textbf{MNLI (Mismatched)}  & \textbf{QQP} & \textbf{QNLI} & \textbf{SST2}  
    \\
    \midrule
    
    \textbf{\dgrab}
    & 65.91 $\pm$ 0.46 \%
    & 64.36 $\pm$ 2.03 \%
    & 82.25 $\pm$ 0.21 \%
    & 62.11 $\pm$ 0.70 \%
    & 82.65 $\pm$ 0.39 \%
    \\

    \textbf{\dshuffle} 
    & 65.42 $\pm$ 0.36 \%
    & 63.93 $\pm$ 1.63 \%
    & 81.74 $\pm$ 0.33 \%
    & 61.87 $\pm$ 0.67 \%    
    & 82.68 $\pm$ 0.57 \%
    \\

    \bottomrule
    \end{tabular}

    \caption{GLUE fine-tuning datasets: Validation accuracy of \dgrab{} in comparison to  \dshuffle{}, reporting mean and standard deviation of best results for each run. There are 3 runs for each example ordering algorithm.}
    \label{tab:gpt2-glue}
    
\end{table*}

We include these fine-tuning results in part to support our claim in Section~\ref{sec:conclusion} that \dgrab{} exhibits its benefits more clearly when there are more training epochs. Our pre-training results suggest that \dgrab{} would confer benefits to pre-training large models over multiple epochs; however, \dgrab{} will not necessarily be useful for short runs of fine-tuning (as indicated in Table~\ref{tab:gpt2-glue}, for which the results for both ordering algorithms are effectively identical).

\newpage
\subsection{Ablation simulation study: The impact of learning rate $\alpha$}\label{appendix-lr}

In the experiment shown on Figure~\ref{fig:lenet-lr}, we select the same momentum and weight decay as the LeNet experiment for 3 random seeds as in Appendix~\ref{appendix:workers}. The aggregated minibatch size is still 64 for all runs, and we use 64 workers.

\begin{figure}[!h]
\vspace{-.2cm}
  \centering
    \includegraphics[width=\columnwidth]{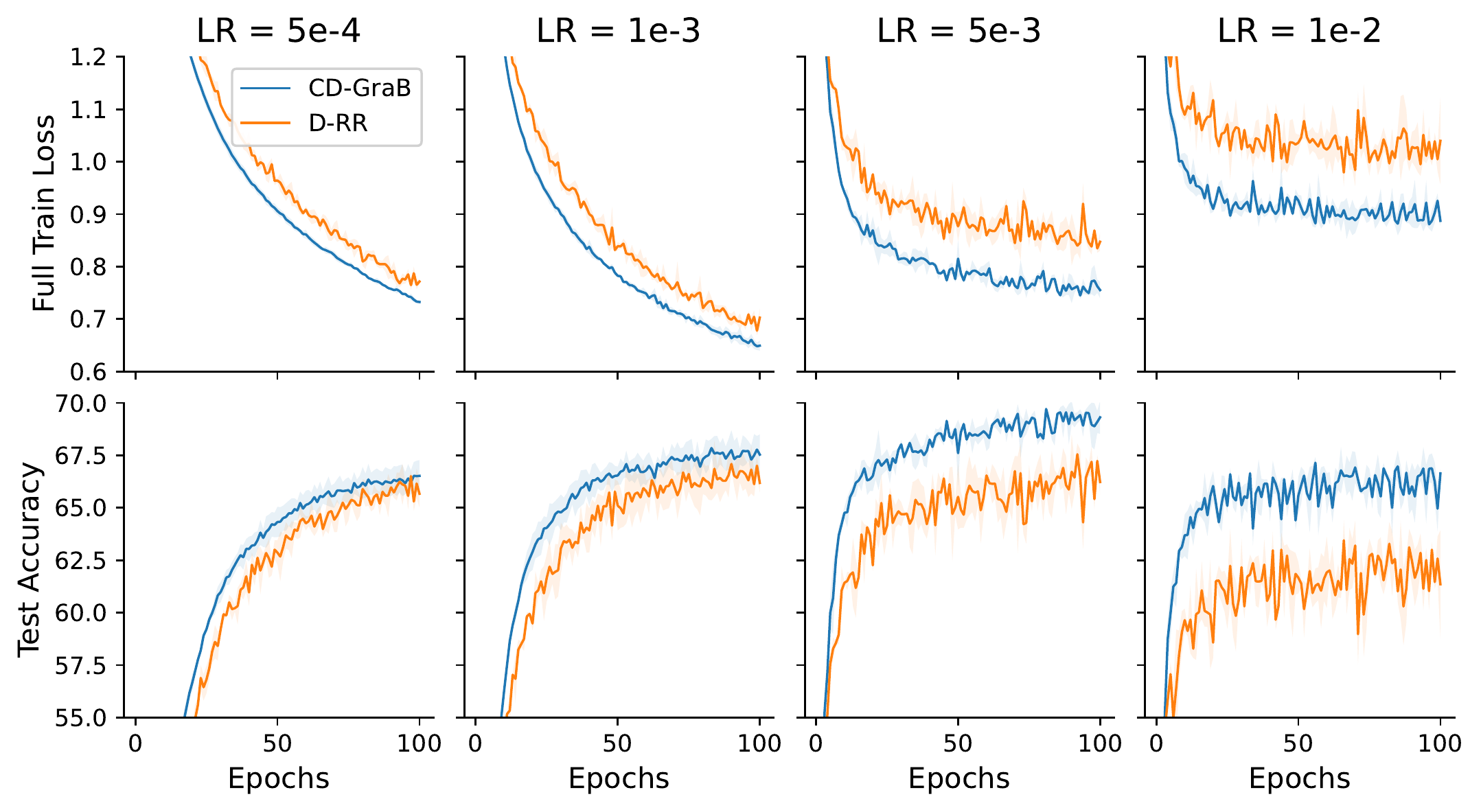}
    \vspace{-.5cm}
    \caption{Convergence for \dgrab, \dshuffle{}  training LeNet on CIFAR-10, with $\workers = 64$ workers. The aggregated minibatch size per update is 64. We use 3 random seeds, and plot the mean values across random seeds as the curve, the standard deviation as the shaded area.\looseness=-1} 
  \label{fig:lenet-lr}
\end{figure}

We find that when we increase the learning rate from 1e-3 to 1e-2, \dgrab{} still maintains relatively better performance than \dshuffle. The best learning rate for \dshuffle{} is 1e-3, in terms of achieving the best test accuracy. We did not tune the learning rate for \dgrab, and we expect that it is possible to use a higher learning rate and still maintain better empirical performance than \dshuffle{} and even faster convergence. We defer such empirical investigations to future work. Altogether, these preliminary empirical results confirm that it is possible to use higher learning rate for \dgrab{}, given that online $\mathsf{PairBalance}$ does not need to use a stale mean (Section~\ref{sec:dgrab:solution}), which would make larger learning rates perform poorly. 


\end{document}